%% file: thesis.tex
\documentclass[12pt,letterpaper]{report}          
\usepackage[intlimits]{amsmath}
\usepackage{amsfonts,amssymb,amsthm}
\DeclareSymbolFontAlphabet{\mathbb}{AMSb}
\usepackage{natbib}
\usepackage{apalike}
\usepackage{float}
\usepackage[bf]{caption}       
\setcaptionmargin{0.5in}
\usepackage{fancyhdr}
\usepackage{fancybox}
\usepackage{ifthen}
\usepackage{bu_ece_thesis}
\usepackage{url}
\usepackage{hyperref}
\usepackage{lscape,afterpage}
\usepackage{xspace}
\usepackage{xcolor}
\usepackage{epstopdf} 
\usepackage{subfig}
\usepackage{longtable}
\usepackage{nicefrac}
\usepackage[T1]{fontenc}
\usepackage{tabularx}
\usepackage{multirow}
\usepackage{makecell}
\usepackage{appendix}

\usepackage{algorithm}
\usepackage{algorithmic}
\include{symbol}

\begin{document}

\include{0_Prelim/prelim}        
\cleardoublepage

\include{1_Intro/intro}
\cleardoublepage

\include{2_Adaptive/adap}
\cleardoublepage

\include{3_CosExp/cos_exp}
\cleardoublepage

\include{4_Momentum/momentum}
\cleardoublepage

\include{5_Conclusion/conclusion}
\cleardoublepage

\newpage
\singlespace
\begin{sloppypar}
\bibliographystyle{plainnat}

\bibliography{thesis}
\end{sloppypar}

\cleardoublepage


\end{document}

%% file: symbol.tex
\renewcommand{\Pr}{\field{P}}

\newcommand{\bta}{\boldsymbol{\eta}}
\newcommand{\bps}{\boldsymbol{\epsilon}}
\newcommand{\bb}{\boldsymbol{b}}
\newcommand{\ba}{\boldsymbol{a}}
\newcommand{\bm}{\boldsymbol{m}}
\newcommand{\bg}{\boldsymbol{g}}
\newcommand{\bx}{\boldsymbol{x}}

\newcommand{\bu}{\boldsymbol{u}}
\newcommand{\bh}{\boldsymbol{h}}
\newcommand{\by}{\boldsymbol{y}}

\newcommand{\bz}{\boldsymbol{z}}

\newcommand{\bw}{\boldsymbol{w}}

\newcommand{\bgamma}{\boldsymbol{\gamma}}

\newcommand{\argmin}{\mathop{\text{argmin}}}

\newcommand{\field}[1]{\mathbb{#1}}

\newcommand{\R}{\field{R}}

\newcommand{\E}{\field{E}}

\newtheorem{lemma}{Lemma}
\newtheorem{theorem}{Theorem}
\newtheorem{cor}{Corollary}

\newtheorem{example}{Example}

\newtheorem{claim}{Claim}

\newtheorem{assumption}{Assumption}

\newtheorem{assumptionA}{Assumption}
\newtheorem{assumptionC}{Assumption}


%% file: 0_Prelim/prelim.tex

\title{Formal guarantees for heuristic optimization algorithms used in machine learning}

\author{Xiaoyu Li}

\degree=2

\prevdegrees{B.S., University of Science and Technology of China, 2016}

\department{Division of System Engineering}

\defenseyear{2022}
\degreeyear{2022}
\reader{First Reader}{Francesco Orabona, Ph.D.}{Associate Professor of Electrical and Computer Engineering\\Associate Professor of Systems Engineering\\
Associate Professor of Computer Science}
\reader{Second Reader}{Ashok Cutkosky, Ph.D.}{Assistant Professor of Electrical and Computer Engineering\\Assistant Professor of Systems Engineering\\
	Assistant Professor of Computer Science}
\reader{Third Reader}{Alexander Olshevsky, Ph.D.}{Associate Professor of Electrical and Computer Engineering\\Associate Professor of Systems Engineering}
\reader{Fourth Reader}{Ioannis Ch. Paschalidis, Ph.D.}{Distinguished Professor of Engineering\\Professor of Electrical and Computer Engineering\\Professor of Systems Engineering\\
	Professor of Biomedical Engineering\\ Professor of Computing \& Data Sciences}

\numadvisors=1
\majorprof{Francesco Orabona, Ph.D.}{{Associate Professor of Electrical and Computer Engineering\\Associate Professor of Systems Engineering\\
		Associate Professor of Computer Science}}




\maketitle
\cleardoublepage

\copyrightpage
\cleardoublepage

\approvalpagewithcomment
\cleardoublepage


\newpage
\section*{\centerline{Acknowledgments}}
\input{0_Prelim/ack}
\cleardoublepage


\begin{abstractpage}
\input{0_Prelim/abs}
\end{abstractpage}
\cleardoublepage


\tableofcontents
\cleardoublepage

\newpage
\listoftables
\cleardoublepage


\chapter*{List of Abbreviations}


\begin{center}
 \begin{tabular}{lll}
    \hspace*{2em} & \hspace*{1in} & \hspace*{4in} \\
    DA & \dotfill & Dual Averaging \\
    FTRL & \dotfill & Follow The Regularized Leader\\
    GD   & \dotfill & Gradient Descent \\
    OMD & \dotfill & Online Mirror Descent\\
    $\mathbb{R}^{2}$  & \dotfill & the Real plane \\
    SGD  & \dotfill & Stochastic Gradient Descent \\
    SGDM & \dotfill & Stochastic Gradient Descent with Momentum\\
    SHB & \dotfill & Stochastic Heavy Ball \\

    
 \end{tabular}
\end{center}
\cleardoublepage


\newpage
\endofprelim

%% file: 0_Prelim/ack.tex
I want to extend my most sincere thanks to my advisor, Professor Francesco Orabona, for his five years of guidance and support during my PhD study.  He led me to step into the field of optimization, spent a lot of hours with me on the discussion from high level ideas to technical proofs, and provided precious advice on writing and presentation. I have been extremely lucky to have an advisor who cared so much about my work and offered his help so promptly whenever I need. His profession, patience, and kindness to research work, colleagues, and students have been deeply engraved in my mind. He was, is and will always be a role model to me. 

I am also very grateful to Professor Ashok Cutkosky, Professor Alexander Olshevsky, and Professor Ioannis Paschalidis for their advice and feedback on improving my dissertation. They are outstanding researchers in their field. I am thankful to have them on my dissertation committee. 

I would like to thank all the present and past members of OPTIMAL Lab. I would like to thank my coauthors Zhenxun Zhuang and Dr. Mingrui Liu for the fruitful discussions and collaborations. 

Many thanks go to all my colleagues and friends at BU and past. I would like to thank Jing Zhang, Xiang Li, Qianqian Ma, Hong Wang, Keyi Chen and many others for their encouragement and support during the tough times of this journey. Special thanks to Irene for the voice projection tutorial. 

I would also like to acknowledge the support from the Division of Systems Engineering. Especially, I want to thank Elizabeth Flagg and Christine Ritzkowski for helping me with paperwork, scheduling and many other things during my study at BU. 

I want to express my gratitude to my husband Libo Wu, for his love, care and encouragement. He has been incredibly supportive to me throughout this entire process. I also thank my fur family member Jojo for always making me smile.

Finally, I would like to express my deepest gratitude to my parents for their unconditional love and support. They always encourage me and believe in me throughout all the ups and downs in my life.

This dissertation is dedicated to my family.

%% file: 0_Prelim/abs.tex
Recently, Stochastic Gradient Descent (SGD) and its variants have become the dominant methods in the large-scale optimization of machine learning problems. A variety of strategies have been proposed for tuning the step sizes, ranging from adaptive step sizes (e.g., AdaGrad) to heuristic methods to change the step size in each iteration. Also, momentum has been widely employed in machine learning tasks to accelerate the training process. Yet, there is a gap in our theoretical understanding of them. In this work, we start to close this gap by providing formal guarantees to a few heuristic optimization methods and proposing improved algorithms if the theoretical results are suboptimal. 

First, we analyze a generalized version of the AdaGrad (Delayed AdaGrad) step sizes in both convex and non-convex settings, showing that these step sizes allow the algorithms to automatically adapt to the level of noise of the stochastic gradients. We show sufficient conditions for Delayed AdaGrad to achieve almost sure convergence of the gradients to zero, which is the first guarantee for Delayed AdaGrad in the non-convex setting. Moreover, we present a high probability analysis for Delayed AdaGrad and its momentum variant in the non-convex setting. 

Second, we present an analysis of SGD with exponential and cosine step sizes, which are simple-to-use, empirically successful but lack of theoretical support. We provide the very first convergence guarantees for them in the smooth and non-convex setting, with and without the Polyak-Łojasiewicz (PL) condition. We show that these two strategies also have the good property of adaptivity to noise under PL condition. 

Third, we study the last iterate of momentum methods. We prove the first lower bound in the general convex setting for the last iterate of SGD with constant momentum. Based on the fact that the lower bound is suboptimal, we investigate a class of (both adaptive and non-adaptive) Follow-The-Regularized-Leader-based momentum algorithms (FTRL-based SGDM) with increasing momentum and shrinking updates. We show that their last iterate has optimal convergence for unconstrained convex stochastic optimization problems without projections onto bounded domains nor knowledge of the number of iterations.

%% file: 1_Intro/intro.tex
\chapter{Introduction}
\label{chapter:Introduction}
\thispagestyle{myheadings}

Modern machine learning has led to remarkable empirical success in a few areas, including computer vision, natural language processing, generative modeling and reinforcement learning.
Optimization is one of the core parts of machine learning: most machine learning algorithms can be reduced to the minimization of an objective function and constructed using the given data. With the rapid growth of data amount and model complexity, Stochastic Gradient Descent (SGD) ~\citep{RobbinsM51} has become the tool of choice to train machine learning models due to its simplicity and efficiency. In particular, in the Deep Learning community, it is widely used to minimize the training error of deep neural networks. 

In machine learning optimization, SGD often comes with heuristic tricks, such as momentum, and a variety of strategies in tuning the stepsizes. 
Indeed, the performance of SGD heavily depends on the choice of stepsizes, which sparkles a lot of strategies for stepsize tuning, ranging from coordinate-wise ones (a.k.a. “adaptive” stepsize) ~\citep[e.g.,][]{DuchiHS11, McMahanS10, TielemanH12, Zeiler12, KingmaB15} to heuristic approaches to change the stepsize in each iteration \citep{LoshchilovH17, HeZZZXL19}. 
Besides, momentum is often employed to accelerate the optimization process and proved to be important in many machine learning applications.  
Although these heuristics are empirically successful,  we are far from a complete theoretical understanding of them. As a consequence, a successful training process comes at a cost of a considerable trial-and-error tuning procedure.

Towards the theoretical understanding of these methods, a big challenge is the non-convex nature of many machine learning objective functions, such as in neural networks. Indeed, classic convex optimization theories and analysis techniques for SGD can not be applied to these heuristics of training neural networks. Moreover, even in the convex setting, it is often the case that an idealized version of the algorithm is used in the theory rather than the actual one people use in practice. 
For example, most existing analyses of adaptive gradient methods are conducted in the context of online learning, assuming the optimization to be constrained in a convex bounded set. 
Also, in the classic analysis of momentum, 
projections onto bounded domains at each step, averaging of the iterates~\citep[e.g.,][]{AlacaogluMMC20}, 
and knowledge of the total number of iterations~\citep{GhadimiL12} are often assumed. 

Motivated by the above facts, we aim to bridge the gap between theory and practice by providing theoretical guarantees of these advanced SGD-based methods as well as proposing improved algorithms when the theoretical results are sub-optimal. In particular, for the stepsize strategies of SGD, we focus on a family of adaptive stepsizes and two heuristic stepsizes: exponential stepsize and cosine stepsize. Then, we study the convergence of the last iterate of SGD with momentum and its improved variants.

In the remainder of this chapter, we start by introducing the problem set-up and the limitations of the existing analysis, and then give a summary of our results that will be discussed in this dissertation.

\section{Preliminary and Problem Set-Up}
\label{subchapter:setup}

Consider the unconstrained optimization problem 
\begin{equation}
\label{eq:problem}
\min_{ \bx \in \R^d }  \  f(\bx), 
\end{equation}
where $f(\bx):\R^d\rightarrow \R$ is a function bounded from below and we denote its infimum by $f^{\star}$. 
In this work, we do not require $f$ to have a finite-sum structure.

Let's introduce some definitions to characterize family of functions.

A function $f$ is called convex if for any $\alpha \in [0,1]$ and any $\bx, \by \in \R^d$, 
\[
f(\alpha \bx + (1- \alpha) \by ) \leq \alpha f(\bx) + (1- \alpha)  f(\by). 
\]

$\| \cdot \|$ stands for $\ell_2$ norm.

A real-value function $f$ is called $L$-Lipschitz if there exists a positive number $L > 0$, such that for any $\bx, \by \in \R^d$
\begin{equation}
	\label{eq:lipcz}
	\| f(\bx) - f(\by) \| \leq L \| \bx - \by \|~. 
\end{equation}

A differentiable function $f$ is called $M$-smooth if its gradients is $M$-Lipschitz, i.e., 
\begin{equation}
	\label{eq:smooth_def}
	\| \nabla f(\bx) - \nabla f(\by) \| \leq M \| \bx - \by \|, \quad \forall \bx, \by \in \R^d. 
\end{equation}

Note that \eqref{eq:smooth_def} implies~\citep[Lemma 1.2.3]{Nesterov04}
\begin{equation}
\label{eq:smooth2}
\left|f(\by)-f(\bx)-\langle \nabla f(\bx), \by-\bx\rangle\right|
\leq \frac{M}{2}\|\by-\bx\|^2, \quad \forall \bx, \by \in \R^d.
\end{equation}

We assume to have access to a first-order black-box optimization oracle that returns a stochastic (sub)gradient in any point $\bx \in \R^d$. In particular, we assume that we receive a vector $\bg(\bx, \xi)$ such that $\E_{\xi} [\bg(\bx, \xi)] = \nabla f(\bx)$ for any $\bx \in \R^d$. In words, $\bg(\bx, \xi)$ is an unbiased estimate of $\nabla f(\bx)$. For example, in machine learning, $\xi$ can be the random index of a training sample we use to calculate the gradient of the training loss. SGD starts from an arbitrary point $\bx_1 \in \R^d$ and iteratively updates the solution as 
\[
\bx_{t+1} = \bx_t - \eta_t \bg(\bx_t, \xi_t), 
\]
where $\eta_t > 0$ is the stepsize or learning rate.  In words,  the iterate $\bx_t$ moves along the opposite direction of the vector $\bg (\bx_t, \xi_t)$ by $\eta_t$ at the $t$-th step.  To make the notion concise, we denote by $\bg_t \triangleq \bg(\bx_t, \xi_t)$. 

In the convex case, our goal is to solve~\eqref{eq:problem}, that is, to find a global minimum of $f$. Yet if $f$ is nonconvex, solving~\eqref{eq:problem} is generally NP-hard~\citep{NemirovskyY83}, so we might turn to a less ambitious goal. We assume the function to be smooth. By its definition in~\eqref{eq:smooth_def}, when we approach to a local minimum the gradients go to zero. So minimizing the gradient norm will be our objective for SGD in the nonconvex case if without extra assumptions. 

To warm up, we introduce some classic results. 

We consider to minimize a (nonconvex) $M$-smooth function $f$ using SGD with unbiased stochastic gradient. We will also assume the variance of noise on the stochastic gradients is bounded, i.e., $\E \left[ \| \bg(\bx_t, \xi_t, ) - \nabla f(\bx_t) \|^2 \right] \leq \sigma^2.$

We start by making use of the property of smooth function~\eqref{eq:smooth2} on the iterates of SGD:
\begin{align*}
\label{eq:smooth_one_step}
f(\bx_{t+1}) 
&  \leq  f(\bx_t) - \langle \nabla f(\bx_t), \eta_t \bg_t \rangle + \frac{M}{2} \| \eta_t \bg_t \|^2 \\
\end{align*}

Then, taking expectation with respect to the underlying variable $\xi$, we will have
\begin{align*}
\E f(\bx_{t+1}) 
&  \leq  \E f(\bx_t) - \eta_t\E \langle \nabla f(\bx_t), \bg_t \rangle + \frac{M \eta_t^2}{2} \E \| \bg_t \|^2 \\
&  = \E f(\bx_t) - \eta_t\E \| \nabla f(\bx_t ) \|^2 + \frac{M \eta_t^2}{2} \left( \E \| \bg_t - \nabla f(\bx_t) \|^2 + \E \| \nabla f(\bx_t) \|^2 \right) \\
& \leq \E f(\bx_t) - \left(\eta_t - \frac{M \eta_t^2}{2}\right)\E \| \nabla f(\bx_t ) \|^2 + \frac{M \eta_t^2 \sigma^2}{2}~.
\end{align*}

Summing over $t$ from 1 to $T$ and reordering the terms, we have 
\begin{align*}
\sum_{t=1}^T \left(\eta_t - \frac{M \eta_t^2}{2}\right)\E \| \nabla f(\bx_t ) \|^2  
& \leq  f(\bx_1 ) - \E f(\bx_{T+1}) + \frac{M \sigma^2}{2}\sum_{t=1}^T \eta_t^2 \\
& \leq  f(\bx_1) - f^{\star} + \frac{M \sigma^2}{2}\sum_{t=1}^T \eta_t^2~. 
\end{align*}

For any constant $\eta_t = \eta \leq \frac{1}{M}$ such that $\eta - \frac{M\eta^2}{2} \geq \frac{\eta}{2}$, we obtain
\begin{align*}
\frac{1}{T}\sum_{t=1}^T \E \| \nabla f(\bx_t ) \|^2  \leq \frac{2f(\bx_1) - f^{\star}}{\eta T} + M\eta\sigma^2~.
\end{align*}

The first term in the right hand side does not depends on the noise level $\sigma$ and the second one does. Choosing $\eta$ is a trade-off between these two terms. In particular, considering $\eta = \min \left(\frac{1}{L}, \frac{c}{\sigma \sqrt{T}} \right)$, where $c > 0$ is a parameter of the stepsize, we have 
\begin{align*}
\frac{1}{T}\sum_{t=1}^T \E \| \nabla f(\bx_t ) \|^2  
\leq \frac{M(f(\bx_1) - f^{\star})}{T} + \left( c+ \frac{f(\bx_1) - f^{\star}}{c}\right) \frac{\sigma}{\sqrt{T}}~. 
\end{align*}

In words, it tells that the average gradient norm converges to zero with a rate of $O(\frac{1}{T} + \frac{\sigma}{\sqrt{T}})$. In addition, due to the fact that the average can be lower bounded by the minimum of a sequence, we know that there exists at least a point $\bx_t$ in $\bx_1, \cdots, \bx_T$, of which the gradient norm is as small as the rate.

Now, let's look at the convergence rate. The first term $\frac{1}{T}$ is fast in $T$ and the second terms $\frac{\sigma}{\sqrt{T}}$ is slower.  That means the algorithm makes fast progress at the beginning of the optimization and then slowly converges as long as the number of iterations becomes big enough compared to the variance of the noise. In case the noise on the gradients is zero, SGD becomes simply gradient descent and it will converge at a $O(1/T)$ rate.

Though SGD with such stepsizes guarantees a convergence rate, we have to assume we know everything. Yet some factor like the noise level is rarely given in real world applications. One possible alternative is a decreasing stepsize $\eta_t = \frac{c}{\sqrt{t}}$ where $c > 0$. Compared to a constant stepsize, these stepsizes help to fight the disturbance of the noise when the iterate is close to the optimal point. However, they slow down the progress in the early stage of optimization process, where the oscillation brought by the noise is relatively small compared to how far the iterates are from the solution. Consequently, those stepsizes do not really shine in practice and practitioners turn to some heuristic stepsizes. In the next section, we will zoom in and focus on SGD with such stepsizes. We will discuss the weakness of the current results, as well as what this dissertation contributes in this area. 

\subsection{Convergence of SGD with heuristic stepsizes.}
Classic convergence analysis of the SGD algorithm for non-convex smooth function relies on conditions on the positive stepsizes $\eta_t$~\citep{RobbinsM51}. In particular, a sufficient condition for $\lim_{t \to \infty} \E \left[ \| \nabla f(\bx_t) \|^2 \right] = 0$ is that $(\eta_t)_{t=1}^{\infty}$ is a deterministic sequence of non-negative numbers that satisfies
\begin{equation}
	\label{eq:conditions_stepsize}
	\sum_{t=1}^\infty \eta_t = \infty \quad\text{ and }\quad \sum_{t=1}^\infty \eta_t^2 < \infty~.
\end{equation}
The first condition basically says the iterates should be able to travel anywhere, and the second condition suggests that the stepsize should be small in the late stage to keep the noise under control. Though these conditions cover a broad family of stepsizes, they provide limited information on how to set stepsizes when SGD runs for finite number of iterations.

The state-of-the-art SGD variants use adaptive stepsizes. Among them, 
AdaGrad was proposed in~\citet{DuchiHS11} and has become the basis of all other adaptive optimization algorithms used in machine learning, \citep[e.g.,][]{Zeiler12,TielemanH12,KingmaB15,ReddiKK18}. 

AdaGrad can be reduced to SGD with a vector stepsize, that is, instead of using a scalar as stepsize, it adopts different stepsize in each coordinate. In particular, the iterates update itself with the following form:  $\bx_{t+1} = \bx_t - \bta_t \bg_t$, where $\bta_t = (\eta_{t,1}, \cdots, \eta_{t,d})$ and 
\begin{equation}
\label{eq:adagrad}
\eta_{t,i} = \frac{c}{\sqrt{\epsilon + \sum_{j=1}^t} g_{j,i}}, \quad \epsilon , c > 0,  
\end{equation}
where the products of vectors are element-wise.

In words, AdaGrad updates $\bta_t$ on the fly with the information of all previous stochastic gradients observed on the go. 

Towards the theory of these methods, adaptive stepsizes generally do not fit in the conditions~\eqref{eq:conditions_stepsize}. For example, for AdaGard, when the stochastic gradients are upper bounded by a constant, i.e., $ |g_{i,j}| \leq G, G>0$, 
\begin{align*} 
\sum_{t=1}^{\infty} \eta_{t,i}^2 =  \sum_{t=1}^{\infty}  \frac{c^2}{\epsilon + \sum_{j=1}^t g_{j,i}^2} \geq  \sum_{t=1}^{\infty}  \frac{c^2}{\epsilon + tG^2} = \infty, \quad i = 1, \cdots, d~. 
\end{align*}

In addition, the adaptive stepsizes are believed to require less tweaking to achieve good performance in machine learning applications and we have partial explanations in the convex setting, i.e sparsity of the gradients~\citep{DuchiHS11}. However in the nonconvex setting, little theory is known to explain the better performance. Indeed, for a large number of SGD variants employed by practitioners, condition~\eqref{eq:conditions_stepsize} are not satisfied.  In fact, these algorithms are often designed and analyzed for the convex domain under restrictive conditions, e.g., bounded domains, or they do not provide convergence guarantees at all, \citep[e.g.,][]{Zeiler12}, or even worse they are known to fail to converge on simple one-dimensional convex stochastic optimization problems~\citep{ReddiKK18}. Even considering an \emph{infinite} number of iterations, the behavior of these algorithms is often unknown. 

In Chapter~\ref{chapter:adap}, we focus on a generalized version of AdaGrad (Delayed AdaGrad), with and without momentum, and present theoretical analysis in both convex and nonconvex settings, going in the direction of closing the gap between theory and practice. Continuing in this way, we then focus on SGD with two heuristic stepsizes: exponential and cosine stepsizes and prove for the first time a convergence guarantee in Chapter~\ref{chapter:exp_cos}. 

On the other hand, SGD with appropriate stepsizes is already optimal in all many possible situations, which makes it unclear what kind of advantage we might show. An interesting viewpoint is to go beyond worst-case analyses and show that these learning rates provide SGD with some form of \emph{adaptivity} to the characteristics of the function. More specifically, an algorithm is considered adaptive (or \emph{universal}) if it has the best theoretical performance w.r.t. to a quantity X without the need to know it~\citep{Nesterov15b}. So, for example, it is possible to design optimization algorithms adaptive to scale~\citep{OrabonaP15}, smoothness~\citep{LevyYC18}, noise~\citep{LevyYC18,LiO19,liO20}, and strong convexity~\citep{CutkoskyO18}.
On the other hand, as noted in \citet{Orabona19}, it is remarkable that even if most of the proposed step size strategies for SGD are called ``adaptive'', for most of them their analyses do not show any provable advantage over plain SGD nor any form of adaptation to the intrinsic characteristics of the non-convex function. Following this direction, we show the property of adaptive-to-noise for Delayed AdaGrad (with momentum) in Chapter~\ref{chapter:adap} and for SGD with cosine and exponential stepsizes in Chapter~\ref{chapter:exp_cos}, providing possible explanations for the empirical success of these kinds of algorithms in practical machine learning applications.

\subsection{Convergence of SGD with Momentum}
SGD with Momentum includes several variants in the literature, such as the stochastic version of the Heavy Ball momentum (SHB) \citep{Polyak64} and Nesterov's momentum (also called Nesterov Accelerate Gradient method)~\citep{Nesterov83}, as well as exponential moving average of the (stochastic) gradients used to replace the gradients in the updates~\citep{KingmaB15, ReddiHSPS16, AlacaogluMMC20, LiuGY20}.
In this dissertation, we denote by Stochastic Gradient Descent with Momentum (SGDM) the following updates 
\begin{equation}
\bx_{t+1} = \bx_t - \eta_t \bm_t, \quad  \bm_t = \beta_t \bm_{t-1} + \nu_t \bg_t,
\end{equation}
where $\nu_t>0$ and $0 \leq \beta_t\leq 1$.
In particular,  when $\nu_t = 1$, it recovers the updates of the stochastic version of the Heavy Ball momentum (SHB) \citep{Polyak64}. Instead, when $\nu_t = 1-\beta_t$, it recovers the variant with exponential moving average of the stochastic gradients.

Momentum seem to accelerate the training process in machine learning optimization.
However, due to the presence of noise, our theoretical understanding of the advantage of SGD with momentum over SGD is not clear. Indeed, recent studies~\citep{LiuB19, KidambiNJK18} reveal that SGD with either Polyak momentum or Nesterov momentum does not guarantee an accelerated rate of convergence of noise nor any real advantage over plain SGD on linear regression problems. In fact, a variant of SGD with momentum improves only the non-dominant terms in the convergence rate on some specific stochastic problems~\citep{DieuleveutFB17,JainKKNS18}. 
Moreover, often an idealized version of SGD with momentum is used in the theoretical analysis rather than the actual one people use in practice.
For example, projections onto bounded domains at each step, averaging of the iterates~\citep[e.g.,][]{AlacaogluMMC20}, and knowledge of the total number of iterations~\citep{GhadimiL12} are often assumed. Overall, recent analyses seem unable to pinpoint any advantage of using a momentum term in SGD in the stochastic optimization of general convex functions.

To show a discriminant difference between SGD and SGD with Momentum, we focus on the convergence of the last iterate of SGD with momentum for unconstrained optimization of convex functions in Chapter~\ref{chapter:momentum}. We first show that momentum does not help to remove the $\ln T$ in the lower bound of the last iterate of SGD. 
Then, motivated by this result, we analyze yet another variant of SGD with Momentum, which yields the optimal convergence rate. 

\section{Contributions}
We summarize the contributions of this thesis as follows. 
\begin{itemize}
\item In Chapter~\ref{chapter:adap}, we present an analysis of a generalized version of the AdaGrad (Delayed AdaGrad) step sizes. We prove for the first time in the nonconvex setting almost sure convergence to zero of the gradients of SGD with both coordinate-wise and global versions of these stepsizes. 
We prove that both in the convex and nonconvex setting, the global Delayed AdaGrad stepsizes adapts to the noise level with a convergence rate, which interpolates between the convergence rate of Gradient Descent and the one of SGD, depending on the noise level. 
We further present a high probability analysis of SGD with momentum and adaptive learning rates. We show the first high probability convergence rates, which are adaptive to the level of noise, for the gradients of Delayed AdaGrad in the nonconvex setting. 
\item In Chapter~\ref{chapter:exp_cos}, we provide the very first convergence results for SGD with two popular stepsizes: exponential and cosine step sizes,  which are simple to use, empirically successful, but lack a theoretical justification.
We show that, in the case when the function satisfies the PL condition~\citep{Polyak63,Lojasiewicz63,KarimiNS16}, both exponential step size and cosine step size strategies \textit{automatically adapt to the level of noise of the stochastic gradients}. 
Without the PL condition, we show that SGD with either exponential step sizes or cosine step sizes has an (almost) optimal convergence rate for smooth non-convex functions.

\item 
In Chapter~\ref{chapter:momentum}, we present an analysis of the convergence of the last iterate of SGD with momentum. We show for the first time that the last iterate of SGDM can have a suboptimal convergence rate for \emph{any constant momentum setting}.
Based on this fact, we investigate a class of Follow-The-Regularized-Leader-based momentum algorithms (FTRL-based SGDM). We show the optimal convergence of their last iterate for unconstrained convex stochastic problems without projections onto bounded domain nor prior knowledge of the number of iterations. 
\end{itemize}

\section{Notation}
We use bold letters to denote vectors and matrices, e.g, $\bx \in \R^d$, and use ordinary letters to denote scalars. 
The coordinate $j$ of a vector $\bx$ is denoted by $x_j$ and as $(\nabla f(\bx))_j$ for the gradient $\nabla f(\bx)$. 
To keep the notation concise, all standard operations $\bx \by$, $1/\bx$, $\bx^2$, $\bx^{\nicefrac{1}{2}}$ and $\max(\bx,\by)$ on the vectors $\bx, \by$ are supposed to be element-wise. 
We denote by $\E[\cdot]$ the expectation with respect to the underlying probability space and by $\E_t[\cdot]$ the conditional expectation with respect to the past. $\Pr$ denotes the probability of an event. 
As mentioned in Section~\ref{subchapter:setup}, we denote by $\bg_t \triangleq \bg(\bx_t, \xi_t)$. Also , we denote the $j$-th element of the vector $\bg_i$ as$g_{i,j}$. 
Any norm without particular notation is the $\ell_2$ norm. 

%% file: 2_Adaptive/adap.tex
\chapter{Adaptive Stepsize}
\label{chapter:adap}
\thispagestyle{myheadings}

\input{2_Adaptive/intro}

\input{2_Adaptive/rel}
\input{2_Adaptive/step}

\input{2_Adaptive/almost_sure}
\input{2_Adaptive/convex}

\input{2_Adaptive/adapt}

\input{2_Adaptive/highp}
\input{2_Adaptive/discussion}

%% file: 2_Adaptive/intro.tex
\section{Introduction}
In this chapter, we focus on a generalized version of the adaptive stepsizes popularized by AdaGrad~\citep{DuchiHS11} and present an analysis of the convergence of SGD~(with momentum) with these stepsizes. 

We analyze two types of step size: a global step size
\begin{equation}
	\label{eq:eta}
	\eta_t=\frac{\alpha}{\left(\beta+ \sum_{i=1}^{t-1} \|\bg_i\|^2\right)^{\nicefrac{1}{2}+\epsilon}},
\end{equation}
and a coordinate-wise one $\bta_t = (\eta_{t,1}, \dots, \eta_{t,d})$, 
\begin{equation}
	\label{eq:eta2}
	\eta_{t,j}=\frac{\alpha}{\left(\beta+ \sum_{i=1}^{t-1} g_{i,j}^2\right)^{\nicefrac{1}{2}+\epsilon}}, \quad j=1, \dots,d
\end{equation}
where $\alpha>0$ and $\beta, \epsilon\geq0$. 

With $\epsilon=0$, \eqref{eq:eta} have been used in online convex optimization to achieve adaptive regret guarantees, \citep[e.g.,][]{RakhlinS13,OrabonaP15}.

The additional parameter $\epsilon$ allows us to increase the decrease rate of the stepsize and it will be critical to obtain our almost sure convergence results.

In this chapter, we address the following two basic questions and answer positively to both of them.
\begin{itemize}
\item Are there conditions under which the generalized
AdaGrad stepsize converge almost surely with an infinite number of iterations in the non-convex setting?
\item Are there conditions under which the rate is better
than the one of the plain SGD with decreasing stepsizes?
\end{itemize}

In particular, in Section~\ref{sec:almost_sure}, we prove an asymptotic convergence to zero of the gradients of SGD with these stepsizes in the nonconvex case. In Section~\ref{sec:adapt_rates},  we prove a convergence rate for SGD with a global version of these stepsizes, showing that they adapt to the noise level, in both convex and nonconvex cases. Last, in Section~\ref{sec:adap_highp}, we analyze its momentum variant and extend the nonconvex results to a high-probability analysis.

%% file: 2_Adaptive/rel.tex
\section{Related Work}
\paragraph{Adaptive stepsizes in the convex world}
Adaptive stepsizes were first proposed in the online learning literature~\citep{AuerCG02} and adopted into the stochastic optimization one later~\citep{DuchiHS11}. In particular, \citet{DuchiHS11} prove that AdaGrad can converge faster if the gradients are sparse and the function is convex. Yet, most of these studies assumed the optimization to be constrained in a convex bounded set, which is often false in many applications of optimization for machine learning. \citet{YousefianNS12} analyze different adaptive stepsizes, but only for strongly convex optimization. 
\citet{WuWB18} have analyzed a choice of adaptive stepsizes similar to the global stepsizes we consider, but their result in the convex setting requires the norm of the gradients strictly greater than zero. 
\citet{LevyYC18} propose an acceleration method with adaptive stepsizes which are also similar to our global ones, proving the $\tilde{O}(1/T^2 )$ convergence in the deterministic smooth case and $\tilde{O} (1/\sqrt{T})$ in both general deterministic case and stochastic smooth case, but requiring a bounded-domain assumption. 

\paragraph{Convergence of SGD in the nonconvex setting}
The convergence of a random iterate of SGD for non-convex smooth functions has been proved by \citet{GhadimiL13}, and it was already implied by the results in \citet{Bottou91}. With additional regularity assumptions, these results imply almost sure convergence of the gradient to zero~\citep{Bottou91,BottouCN16}. 
 \citet{Bottou98} assume that beyond a certain horizon the update always moves the iterate closer to the origin on average, which implies the confinement in a bounded domain and, in turn, the almost sure convergence.
On the other hand, the weakest assumptions for the almost sure convergence of SGD for non-convex smooth functions have been established in \citet{BertsekasT00}: the variance of the noise on the gradient in $\bx_t$ can grow as $1+\|\nabla f(\bx_t)\|^2$, $f$ is lower bounded, and the stepsizes satisfy $\sum_{t=1}^{\infty} \eta_t = \infty$ and $\sum_{t=1}^{\infty} \eta_t^2 < \infty$.
However, both approaches do not cover adaptive stepsizes.

\paragraph{Adaptive stepsize in nonconvex setting}
The first work we know on adaptive stepsizes for non-convex stochastic optimization is \citet{KresojaLS17}.
 \citet{KresojaLS17} study the convergence of a choice of adaptive stepsizes that require access to the function values, under strict conditions on the direction of the gradients. \citet{WuWB18} also consider adaptive stepsizes, but they only consider deterministic gradients in the non-convex setting.  Independently, \citet{WardWB19} improved their guarantees proving results similar to our Theorems~\ref{thm:convex} and \ref{thm:sgd_adaptive}. In contrast to \citet{WardWB19}, we do not require the assumption of bounded expected squared norm of the stochastic gradients but the choices of our parameter depend on the Lipschitzness $L$.
Some other related works were proposed after our submission. 
\citet{ZhouTYCG18} analyze an adaptive gradient method in the non-convex setting, but their bounds give advantages only in very sparse case.

A weak condition for almost sure convergence to the global optimum of non-convex functions was proposed in \citet{Bottou98} and recently independently reproposed in~\citet{ZhouMBBG17}. However, this condition implies the very strong assumption that the gradients never point in the opposite direction of the global optimum. In this chapter, in our most restrictive case in Section~\ref{sec:almost_sure}, we will only assume the function to be smooth and Lipschitz.

\paragraph{High probability bounds}
The results on high probability bounds are relatively rare compared to those in expectation, which are easier to obtain. \citet{KakadeT09} used Freeman's inequality to prove high probability bounds for an algorithm solving the SVM objective function. For classic SGD, \citet{HarveyLR19} and \citet{HarveyLPR19} used a generalized Freedman’s inequality to prove bounds in non-smooth and strongly convex case, while \citet{JainNN21} proved the optimal bound for the last iterate of SGD with high probability.  As far as we know, there are currently no high probability bounds for adaptive methods in the nonconvex setting. 

%% file: 2_Adaptive/step.tex
\section{Keeping the Update Direction Unbiased}
\label{sec:stepsize}
A key difference between the generalized AdaGrad stepsizes in \eqref{eq:eta} and \eqref{eq:eta2} with the AdaGrad stepsizes in \citet{DuchiHS11} is the fact that $\bg(\bx_t,\xi_t)$ is not used in $\eta_t$.
It is easy to see that doing otherwise introduces a spurious bias in the update direction and we show an example as follows.  
\begin{example}
\label{ex:ninety}
There exist a differentiable convex and smooth function, an additive noise on the gradients satisfying $\| \bg(\bx, \xi) - \nabla f(\bx) \| \leq S, S>0$, and a sequence of gradients such that for a given $t$ we have $\E_{\xi_t}[\langle\eta_{t+1}\bg(\bx_t,\xi_t),\nabla f(\bx_t)\rangle]<0$. 
\end{example}

We now present the details of Example~\ref{ex:ninety}. Consider the function $f(x) = \frac{1}{2} x^2$. The gradient in $t$-th iteration is $\nabla f(x_t) = x_t$. 
Let the stochastic gradient be defined as $\bg_t = \nabla f(x_t) + \xi_t$, where $P(\xi_t = \sigma_t) =\frac{7}{15}$, $ P(\xi_t = -\frac{3}{2} \sigma_t) = \frac{1}{5}$
and $ P(\xi_t = -\frac{1}{2} \sigma_t) = \frac{1}{3}$. 

Let $A \triangleq \sum_{i=1}^{t-1} g_i^2+\beta$. Then
\begin{align*}
& \langle \E_t \eta_{t+1} \bg_t , \nabla f(x_t) \rangle \\
& = \alpha \left[ \frac{7}{15}\frac{(x_t +\sigma_t) x_t}{[A+(x_t+\sigma_t)^2]^{\frac{1}{2}+\epsilon}} + \frac{1}{5}\frac{(x_t -\frac{3}{2}\sigma_t) x_t}{[A+(x_t-\frac{3}{2}\sigma_t)^2]^{\frac{1}{2}+\epsilon}}+\frac{1}{3}\frac{(x_t -\frac{1}{2}\sigma_t) x_t}{[A+(x_t-\frac{1}{2}\sigma_t)^2]^{\frac{1}{2}+\epsilon}} \right]~.
\end{align*}
This expression can be negative, for example, setting $x_t=1$, $\sigma_t = 10$, $A=10$, $\epsilon=0$ or $\epsilon = 0.1$. 
In words, including the current noisy gradient in $\eta_t$ (that is, using $\eta_{t+1}$) can make the algorithm deviate in expectation more than $90$ degrees from the correct direction.  So, in the following, we will analyze this minor variant of the AdaGrad stepsizes. We call this variant Delayed AdaGrad stepsize.

%% file: 2_Adaptive/almost_sure.tex
\section{Almost Sure Convergence for Nonconvex functions}
\label{sec:almost_sure}

In this section, we show that SGD with Delayed AdaGrad stepsizes in \eqref{eq:eta} and \eqref{eq:eta2} allows to decrease the gradients to zero almost surely, that is, with probability 1. This is considered a required basic property for any optimization algorithm.

As shown in Chapter~\ref{chapter:Introduction}, the stepsizes in \eqref{eq:eta} and \eqref{eq:eta2} \emph{do not satisfy} $\sum_{i=1}^{\infty} \eta_t^2 < \infty$,  not even in expectation. Hence, the results here cannot be obtained from the classic results in stochastic approximation~\citep[e.g.,][]{BertsekasT00}.

Here, we will have to assume our strongest assumptions. In particular, we will need the function to be Lipschitz and the noise to have bounded support. This is mainly needed in order to be sure that the sum of the stepsizes diverges.

We now state our almost sure convergence results.
\begin{sloppypar}
\begin{theorem}
\label{thm:convergence_sgd}
Assume $f$ is $M$-smooth, $L$-Lipschitz and there exists $S>0$ such that $\| \bg(\bx,\xi) - \nabla f(\bx) \| \leq S, \forall \bx$. 
The stepsizes are chosen as in \eqref{eq:eta}, where $\alpha,\beta>0$ and $\epsilon \in (0,\frac{1}{2}]$.
Then, the gradients of SGD converge to zero almost surely.
Moreover, $\lim\inf_{t\rightarrow \infty} \|\nabla f(\bx_t)\|^2 t^{\nicefrac12-\epsilon}=0$ almost surely.
\end{theorem}
\end{sloppypar}

We also state a similar result for the coordinate-wise stepsizes in \eqref{eq:eta2}. 

\begin{theorem}
\label{thm:convergence_adagrad}
Under the assumptions in Theorem~\ref{thm:convergence_sgd}, 
the stepsizes are given by a diagonal matrix $\bta_t$ whose diagonal values are defined in \eqref{eq:eta2}, where $\alpha,\beta>0$ and $\epsilon \in (0,\frac{1}{2}]$.
Then, the gradients of SGD converges to zero almost surely. Moreover, $\lim\inf_{t\rightarrow \infty} \|\nabla f(\bx_t)\|^2 t^{\nicefrac12-\epsilon}=0$ almost surely.
\end{theorem}

As far as we know, the above theorems are the first results on the almost sure convergence of the gradients using generalized AdaGrad stepsizes and assuming $\epsilon>0$. In particular, Theorem~\ref{thm:convergence_adagrad} is the first theoretical support for the common heuristic of selecting the last iterate, rather than the minimum over the iterations.

For the proofs of the above theorems, we will need some technical lemmas.
\begin{lemma}{\citep[Proposition~2]{AlberIS98}\citep[Lemma A.5]{Mairal13}}
\label{lemma:remove_liminf}
Let $(a_t)_{t \geq 1}$, $ (b_t)_{t \geq 1}$ be two non-negative real sequences. Assume that $\sum_{t=1}^{\infty} a_t b_t$ converges and $\sum_{t=1}^{\infty} a_t$ diverges, and there exists $K \geq 0$ such that $|b_{t+1}-b_t| \leq K a_t$. Then $b_t$ converges to 0.
\end{lemma}
\begin{proof}[Proof of Lemma~\ref{lemma:remove_liminf}]
Since the series $\sum_{t=1}^{\infty}a_t$ diverges, given that $\sum_{t=1}^{\infty}a_t b_t$ converges, we necessarily have $\liminf_{t \rightarrow \infty}b_t = 0$. So there exists a subsequence $\{ b_{i(t)} \}$ of $\{ b_t \}$ such that $ \lim_{t \to \infty} b_{i(t)} =0 .$

Let us proceed by contradiction and assume that there exists some $\alpha> 0 $ and some other subsequence $ \{  b_{m(t)}\}$ of $\{ b_t \}$ such that $ b_{m(t)} \geq \alpha$ for all $t$. In this case, we can construct a third subsequence $\{ b_{j(t)}\}$ of $\{ b_t \}$ where the sub-indices $j(t)$ are chosen in the following way: 
\begin{equation}
\label{eq: indice_1}
j(0) = \min \{ l \geq 0:  b_l \geq \alpha \}
\end{equation}

and, given $j(2t)$,

\begin{equation}
\label{eq: indice_2}
j(2t+1) = \min \{ l \geq j(2t) : b_l \leq \frac{1}{2} \alpha \},
\end{equation}
\begin{equation}
\label{eq: indice_3}
j(2t+2) = \min \{ l \geq j(2t+1): b_l \leq \frac{1}{2} \alpha \}~.
\end{equation}

Note that the existence of $\{ b_{i(t)} \}$ and $\{ b_{m(t)} \}$ guarantees that $j(t)$ is well defined. Also by $\eqref{eq: indice_2}$ and $\eqref{eq: indice_3}$
\[
b_l \leq \frac{\alpha}{2}, \quad \text{for } j(2t) \leq l \leq j(2t+1)-1~.
\]
Then, denoting $\phi_t  = \sum_{l=2t}^{j(2t+1)-1} a_l $, we have 
\[
\infty > \sum_{t=1}^{\infty} a_t b_t \geq \sum_{t=1}^{\infty} \sum_{l=2t}^{j(2t+1)-1} a_l b_l \leq \frac{\alpha}{2} \sum_{t=1}^{\infty} \phi_t~.
\]
Therefore, we have $\lim_{t \to \infty} \phi_t = 0$.

On the other hand, by $\eqref{eq: indice_2}$ and $\eqref{eq: indice_3}$, we have $b_{j(2t)} \geq \alpha$, $b_{j(2t+1)} \leq \frac{1}{\alpha}$, so that 
\[
\frac{\alpha}{2} \leq b_{j(2t)}-b_{j(2t+1)} = \sum_{l=j(2t)}^{j(2t+1)-1} (b_l - b_{l+1}) \leq \sum_{l=j(2t)}^{j(2t+1)-1} K a_l = K \phi_t~.
\]

\begin{sloppypar}
So $\phi_t \geq \frac{\alpha}{2K}$, which is in contradiction with $\lim_{t \to \infty} \phi_t = 0.$
Therefore, $b_t$ goes to zero. 
\end{sloppypar}

\end{proof}

\begin{lemma}
\label{lemma:sum_bounded}
Let $a_0>0$, $a_i\geq 0, \ i=1,\dots,T$ and $\beta>1$.
Then
$
\sum_{t=1}^T \frac{a_t}{(a_0+\sum_{i=1}^{t} a_i)^\beta} 
\leq \frac{1}{(\beta-1)a_0^{\beta-1}}
$.
\end{lemma}

\begin{lemma}
\label{lemma:sum_integral_bounds}
Let $a_i\geq0, \dots, T$ and $f:[0,+\infty)\rightarrow [0, +\infty)$ nonincreasing function.
Then
\begin{align*}
\sum_{t=1}^T a_t f\left(a_0+\sum_{i=1}^{t} a_i\right) 
&\leq \int_{a_0}^{\sum_{t=0}^T a_t} f(x) dx~.
\end{align*}
\end{lemma}
\begin{proof}
Denote by $s_t=\sum_{i=0}^{t} a_i$.
\begin{align*}
a_i f(s_i) 
=  \int_{s_{i-1}}^{s_i} f(s_i) d x 
\leq \int_{s_{i-1}}^{s_i} f(x) dx~.
\end{align*}
Summing over $i=1, \dots, T$, we have the stated bound.
\end{proof}

\begin{proof}[Proof of Lemma~\ref{lemma:sum_bounded}]
The proof is immediate from Lemma~\ref{lemma:sum_integral_bounds}.
\end{proof}

We now state a Lemma that allows us to study the progress made in $T$ steps. 
\begin{lemma}
\label{lemma:basic_lemma}
Assume $f$ is $M$-smooth and $\E_\xi [\bg(\bx,\xi)]=\nabla f(\bx)$ for any $\bx \in \R^d$. Then, the iterates of SGD with stepsizes $\bta_t \in \R^{d \times d}$ satisfy the following inequality

\begin{align*}
\E\left[\sum_{t=1}^T \langle \nabla f(\bx_t), \bta_t \nabla f(\bx_t)\rangle\right]
\leq f(\bx_1)- f^* 
+ \frac{M}{2} \E\left[\sum_{t=1}^T \|\bta_t \bg(\bx_t,\xi_t)\|^2\right].
\end{align*}

\end{lemma}
\begin{proof}[Proof of Lemma~\ref{lemma:basic_lemma}]
From \eqref{eq:smooth2}, we have
\begin{align*}
f(\bx_{t+1}) 
&\leq f(\bx_t) + \langle \nabla f(\bx_t), \bx_{t+1}-\bx_t\rangle + \frac{M}{2}\|\bx_{t+1}-\bx_t\|^2 \\
&= f(\bx_t) + \langle \nabla f(\bx_t), \bta_t(\nabla f(\bx_t) - \bg_t)\rangle \\
& \quad - \langle \nabla f(\bx_t), \bta_t \nabla f(\bx_t)\rangle + \frac{M}{2}\|\bta_t\bg_t\|^2.
\end{align*}
Taking the conditional expectation with respect to $\xi_1, \dots, \xi_{t-1}$, we have that
\[
E_t[\langle \nabla f(\bx_t), \bta_t(\nabla f(\bx_t) - \bg_t) \rangle]
=  \langle \nabla f(\bx_t), \bta_t \nabla f(\bx_t) - \bta_t \E_t[\bg_t] \rangle
= 0.
\]
Hence, from the law of total expectation, we have
\[
\E\left[\langle \nabla f(\bx_t), \bta_t \nabla f(\bx_t)\rangle\right]
\leq \E\left[f(\bx_t)- f(\bx_{t+1}) + \frac{M}{2}\|\bta_t \bg_t\|^2\right].
\]
Summing over $t=1$ to $T$ and lower bounding $f(\bx_{T+1})$ with $f^\star$, we have the stated bound.
\end{proof}

With Lemma~\ref{lemma:remove_liminf} - Lemma~\ref{lemma:basic_lemma}, we can prove Theorem~\ref{thm:convergence_sgd}. 
\begin{proof}[Proof of Theorem~\ref{thm:convergence_sgd}]
From the result in Lemma~\ref{lemma:basic_lemma}, taking the limit for $T\rightarrow \infty$ and exchanging the expectation and the limits because the terms are non-negative, we have

\[
\E\left[\sum_{t=1}^\infty \eta_t \| \nabla f(\bx_t)\|^2\right] 
\leq f(\bx_1)- f^\star+ \frac{M}{2}\E\left[\sum_{t=1}^\infty \|\eta_t \bg(\bx_t,\xi_t)\|_{2}^2\right].
\]

Observe that

\begin{align}
\sum_{t=1}^{\infty}  \|\eta_t \bg(\bx_t,\xi_t)\|^2 
& =\sum_{t=1}^{\infty}  \eta_{t+1}^2 \| \bg(\bx_t,\xi_t)\|^2 + \sum_{t=1}^{\infty}  (\eta_t^2-\eta_{t+1}^2) \| \bg(\bx_t,\xi_t)\|^2 \nonumber\\
& \leq \frac{\alpha^2}{2\epsilon \beta^{2\epsilon}} + \max_{t\geq 1} \|\bg(\bx_t,\xi_t)\|^2 \sum_{t=1}^{\infty} (\eta_t^2-\eta_{t+1}^2)  \nonumber\\
& \leq \frac{\alpha^2}{2\epsilon \beta^{2\epsilon}} + \max_{t\geq 1} \|\bg(\bx_t,\xi_t)\|^2 \eta_1^2  \nonumber\\
& \leq \frac{\alpha^2}{2\epsilon \beta^{2\epsilon}} + 2\eta_1^2\max_{t\geq 1} \|\nabla f(\bx_t)\|^2 + \|\nabla f(\bx_t)-\bg(\bx_t,\xi_t)\|^2  \nonumber\\
& \leq \frac{\alpha^2}{2\epsilon \beta^{2\epsilon}} + 2 \frac{\alpha^2}{\beta^{1+2\epsilon}} (L^2+S^2) < \infty, \label{eq:convergence_sgd_eq1}
\end{align}

where in the first inequality we have used Lemma~\ref{lemma:sum_bounded}, and in the third one the elementary inequality $\|\bx+\by\|^2 \leq 2\|\bx\|^2 +2 \|\by\|^2$.

Hence, we have $\E\left[\sum_{t=1}^{\infty} \eta_t \| \nabla f(\bx_t)\|^2\right] < \infty$.
Now, note that $\E[X]<\infty$, where $X$ is a non-negative random variable, implies that $X<\infty$ with probability 1. In fact, otherwise $\Pr[X=\infty]>0$ implies $\E[X] \geq \int_{X=\infty} x d\Pr(X) = \infty$, contradicting our assumption. 
Hence, with probability 1, we have $\sum_{t=1}^{\infty} \eta_t \| \nabla f(\bx_t)\|^2 < \infty$.

Now, observe that the Lipschitzness of $f$ and the bounded support of the noise on the gradients gives

\begin{align*}
\sum_{t=1}^{\infty} \eta_{t}
= \sum_{t=1}^ {\infty} \frac{\alpha}{(\beta+\sum_{i=1}^{t-1} \|g(\bx_i,\xi_i)\|^2)^{\nicefrac12+\epsilon}} 
\geq  \sum_{t=1}^{\infty} \frac{\alpha}{(\beta+2(t-1)(L^2+S^2))^{\nicefrac12+\epsilon}}
= \infty~.
\end{align*}

Using the fact the $f$ is $L$-Lipschitz and $M$-smooth, we have

\begin{align*}
& \left| \|\nabla f(\bx_{t+1})\|^2- \|\nabla f(\bx_t)\|^2\right| \\
& = ( \|\nabla f(\bx_{t+1})\|+ \|\nabla f(\bx_t)\|) \cdot \left| \|\nabla f(\bx_{t+1})\|- \|\nabla f(\bx_t)\| \right| \\
&  \quad \leq 2L M \|\bx_{t+1}-\bx_t\|
= 2LM \|\eta_t \bg(\bx_t,\xi_t)\|  \\
& \quad \leq 2LM (L+S) \eta_t~.
\end{align*}

Hence, we can use Lemma~\ref{lemma:remove_liminf} to obtain $\lim_{t \to \infty} \|\nabla f(\bx_t)\|^2 = 0$.

For the second statement, observe that, with probability 1,
\begin{align*}
\sum_{t=1}^\infty &\|\nabla f(\bx_t)\|^2 t^{\nicefrac12-\epsilon} \frac{\alpha}{t(2L^2+2S^2+\beta)^{\nicefrac12+\epsilon}} 
\leq \sum_{t=1}^\infty \eta_t \|\nabla f(\bx_t)\|^2 <\infty,
\end{align*}

\begin{sloppypar}
where in the first inequality we used the Lipschitzness of $f$ and the bounded support of the noise on the gradients.
Hence, noting that $\sum_{t=1}^\infty \frac{1}{t} =\infty$, we have that $\lim\inf_{t\rightarrow \infty} \|\nabla f(\bx_t)\|^2 t^{\nicefrac12-\epsilon}=0$.
\end{sloppypar}
\end{proof}

\begin{proof}[Proof of Theorem~\ref{thm:convergence_adagrad}]
We proceed similarly to the proof of Theorem~\ref{thm:convergence_sgd}, to get 
\[
\E\left[\sum_{t=1}^\infty \langle \nabla f(\bx_t), \bta_t \nabla f(\bx_t) \rangle\right] 
\leq f(\bx_1)- f(\bx^\star)+ \frac{M}{2}\E\left[\sum_{t=1}^\infty \|\bta_t \bg_t\|_{2}^2\right]~.
\]
Observe that
\[
\sum_{t=1}^{\infty}  \|\bta_t \bg_t\|^2 
= \sum_{t=1}^{\infty} \sum_{i=1}^d \eta_{t,i}^2 \bg_{t,i}^2 
= \sum_{i=1}^d \sum_{t=1}^{\infty} \eta_{t,i}^2 \bg_{t,i}^2 
 < \infty,
\]
where the last inequality comes from the same reasoning in \eqref{eq:convergence_sgd_eq1}.
Hence, we have
\[
\E\left[\sum_{t=1}^{\infty} \langle \nabla f(\bx_t), \bta_t \nabla f(\bx_t)\rangle\right] < \infty~.
\]
Hence, with probability 1, we have
\[
\sum_{t=1}^{\infty} \langle \nabla f(\bx_t), \bta_t \nabla f(\bx_t)\rangle 
= \sum_{t=1}^{\infty} \sum_{j=1}^d \eta_{t,j} \nabla f(\bx_t)_{j}^2 
= \sum_{j=1}^d \sum_{t=1}^{\infty} \eta_{t,j} \nabla f(\bx_t)_{j}^2
< \infty~.
\]
and, for any $j = 1,\dots,d$, 
\[
\sum_{t=1}^{\infty} \eta_{t,j} (\nabla f(\bx_t))_{j}^2 < \infty~.
\]
Now, observe that the Lipschitzness of $f$ and the bounded support of the noise on the gradients gives
\begin{align*}
\sum_{t=1}^{\infty} \eta_{t,j}
= \sum_{t=1}^ {\infty} \frac{\alpha}{(\beta+\sum_{i=1}^{t-1} (g(\bx_i,\xi_i)_j)^2)^{\nicefrac12+\epsilon}}
\geq  \sum_{t=1}^ {\infty} \frac{\alpha}{(\beta+2(t-1)(L^2+S^2))^{\nicefrac12+\epsilon}}
= \infty~.
\end{align*}
Using the fact the $f$ is $L$-Lipschitz and $M$-smooth, we also have
\begin{align*}
&\left| ((\nabla f(\bx_{t+1}))_j)^2 - ((\nabla f(\bx_t))_j)^2\right| \\
& = ( (\nabla f(\bx_{t+1}))_j+ (\nabla f(\bx_t))_j) \cdot \left| (\nabla f(\bx_{t+1}))_j- (\nabla f(\bx_t))_j \right| \\
 &\quad\leq 2L M \|\bx_{t+1}-\bx_t\|
 = 2LM \|\bta_t \bg_t\| 
 \leq 2LM(L+S) \eta_t~.
\end{align*}
Hence, we case use Lemma~\ref{lemma:remove_liminf} to obtain 
\[
\lim_{t \to \infty} ((\nabla f(\bx_t))_j)^2 = 0~.
\]

For the second statement, observe that, with probability 1,
\begin{align*}
\sum_{t=1}^\infty ((\nabla f(\bx_t))_j)^2 t^{\nicefrac12-\epsilon} \frac{\alpha}{t(2L^2+2S^2+\beta)^{\nicefrac12+\epsilon}} 
\leq \sum_{t=1}^\infty \eta_{t,j} (\nabla f(\bx_t))_j)^2 <\infty~.
\end{align*}
Hence, noting that $\sum_{t=1}^\infty \frac{1}{t} =\infty$, we have that $\lim\inf_{t\rightarrow \infty} ((\nabla f(\bx_t))_j)^2 t^{\nicefrac12-\epsilon}=0$.
\end{proof}

%% file: 2_Adaptive/convex.tex
\section{Adaptive Convergence Rates}
\label{sec:adapt_rates}

We will now show that the global Delayed AdaGrad stepsizes give rise to adaptive convergence rates. 
In particular, we will show that for a large range of the parameters $\alpha, \beta, \epsilon$ and independently from the noise variance $\sigma$, the algorithms will have a faster convergence when $\sigma$ is small and worst-case optimal convergence when $\sigma$ is large. Recall that to achieve the same behavior with SGD we should use a different stepsize for each level of noise. 
In the following, we will consider both the convex and non-convex cases.

\subsection{Adaptive Convergence for Convex Functions}
\label{sec:convex}
As a warm-up, in this section, we show that the global stepsizes \eqref{eq:eta} give adaptive rates of convergence that interpolate between the rate of GD and SGD, for a wide range of the parameters $\alpha, \beta$, and $\epsilon$ and without knowledge of the variance of the noise.
Note that, differently from the other proofs on SGD with adaptive rates~\citep[e.g.,][]{DuchiHS11}, we do not assume to use projections onto bounded domains. This makes our novel proof more technically challenging, but at the same time, it mirrors the setting of many applications of SGD in machine learning optimization problems.

We make the following assumption on the stochastic gradients $\bg(\bx,\xi)$. 
\begin{assumptionA}
\label{as:subgaussion}
$\E_{\xi} \left[ \exp \left(\frac{\|\bg(\bx, \xi) - \nabla f(\bx) \|^2}{\sigma^2}\right)\right] \leq \exp(1), \forall \bx. $
\end{assumptionA}

The above assumption has been already used by \citet{NemirovskiJLS09} to prove high-probability convergence guarantees. This condition allows to control the expectation of the maximum of the noise terms $\| \nabla f(\bx) - \bg(\bx,\xi) \|^2$. Using Jensen's inequality, this condition implies a bounded variance of the noise.

\begin{theorem}
\label{thm:convex}
Assume $f$ is convex, $M$-smooth and the stochastic gradients satisfy Assumption~\ref{as:subgaussion}. Let the stepsizes set as in \eqref{eq:eta}, where $\alpha,\beta>0$, $0\leq \epsilon<\frac{1}{2}$, and $4\alpha M<\beta^{\nicefrac{1}{2}+\epsilon}$.
Then, the iterates of SGD satisfy the following bound
\begin{align*}
& \E \left[ \left(  f(\bar{\bx}_T)- f(\bx^\star) \right)^{\nicefrac{1}{2}-\epsilon}  \right] 
\leq  \frac{1}{T^{\nicefrac{1}{2}-\epsilon}} 
\max \left( 2^\frac{1}{\nicefrac{1}{2}-\epsilon} M^{\nicefrac{1}{2}+\epsilon} \gamma, 
 \left( \beta+T\sigma^2  \right)^{\nicefrac{1}{4}-\epsilon^2} \gamma^{\nicefrac{1}{2}-\epsilon}  \right),
\end{align*}
where $\bar{\bx_T}=\frac1T \sum_{t=1}^T \bx_t$ and 
$
\gamma = 
\begin{cases} 
O\left(\frac{1+\alpha^2\ln T}{\alpha(1-\frac{4\alpha M}{\sqrt{\beta}})}\right), & \text{for } \epsilon=0\\ 
O\left(\frac{1+\alpha^2(\frac{1}{\epsilon}+\sigma^2 \ln T)}{\alpha(1-\frac{4\alpha M}{\beta^{\nicefrac{1}{2}+\epsilon}})}\right), & \text{for } \epsilon>0.
\end{cases}
$
\end{theorem}

\paragraph{Remark.}
Using Markov's inequality, from the above bound it is immediate to get that, with probability at least $1-\delta$, we have
\begin{align*}
f(\bar{\bx_T})- f(\bx^\star) 
 \leq \frac{1}{\delta^{\frac{1}{\nicefrac{1}{2}-\epsilon}} T} \max \left(M^{\frac{\nicefrac{1}{2}+\epsilon}{\nicefrac{1}{2}-\epsilon}} \gamma^{\frac{1}{\nicefrac{1}{2}-\epsilon}}, (\beta + T \sigma^2)^{\nicefrac{1}{2}+\epsilon} \gamma \right).
\end{align*}

Up to polylog terms, if $\sigma=0$ this recovers the GD rate, $O(\tfrac{1}{T})$, and otherwise we get the worst-case optimal rate of SGD, $O(\tfrac{1}{\sqrt{T}})$. The same behavior was proved in \citet{DekelGBSX12} with the knowledge of $\sigma$ and stepsize depending on it. Instead, here we do not need to know the noise level or assume a bounded domain.
In the case the constants of the slow term are small compared with the ones of the first term, we can expect a first quick convergent phase, followed by a slow one, as it is often observed in empirical experiments.

For the proof, we need the following technical lemmas.

\begin{lemma}
\label{lemma:smooth}
If $f$ is $M$-smooth, then $\|\nabla f(\bx)\|^2 \leq 2 M (f(\bx)- \min_{\by} f(\by)), \ \forall \bx$.
\end{lemma}

\begin{proof}[Proof of Lemma~\ref{lemma:smooth}]
From \eqref{eq:smooth2}, for any $\bx,\by \in\R^d$, we have
\[
f(\bx+\by)\leq f(\bx)+\langle \nabla f(\bx), \by\rangle+ \frac{M}{2}\|\by\|^2~.
\]
Take $\by=-\frac{1}{M} \nabla f(\bx)$, to have
\[
f(\bx+\by)\leq f(\bx)+\left(\frac{1}{2M}-\frac{1}{M}\right)\| \nabla f(\bx)\|^2~.
\]
Hence,
\[
\|\nabla f(\bx) \|^2
\leq 2M(f(\bx)-f(\bx+\by))
\leq 2M(f(\bx)-\min_{\bu} f(\bu))~.\qedhere
\]
\end{proof}

\begin{lemma}
	\label{lemma:solvex}
	If $x \geq 0$ and $x \leq C(A+Bx)^{\frac{1}{2}+\epsilon}$, then 
	\[
	x < \max ( [ C (2B)^{\frac{1}{2}+\epsilon} ]^{\frac{1}{1/2-\epsilon}}, C(2A)^{\frac{1}{2}+\epsilon} )~.
	\]
\end{lemma}
\begin{proof}[Proof of Lemma~\ref{lemma:solvex}]
	If $A \leq Bx$, then $x \leq C(2Bx)^{\frac{1}{2}+\epsilon}$, so $x \leq \left[ C (2B)^{\frac{1}{2}+\epsilon} \right]^{\frac{1}{1/2-\epsilon}}$. 
	And if $A > Bx$, then $x < C(2A)^{\frac{1}{2}+\epsilon}$. 
	Taking the maximum of the two cases, we have the stated bound.
\end{proof}
\begin{lemma}
	\label{lemma:logsolvex}
	If $x \geq 0$, $A, C, D \geq 0$, $B>0$, and $x^2 \leq (A+Bx)(C+D\ln (A+Bx))$, then $x < 32 B^3 D^2 + 2 B C + 8 B^2 D \sqrt{C} + A/B$.
\end{lemma}
\begin{proof}[Proof of Lemma~\ref{lemma:logsolvex}]
	Assume that $B x > A$. We have that
	\[
	x^2 
	\leq (A + B x) (C + D \ln (A+B x))
	< 2 B x (C + D \ln (2B x))
	< 2 B x (C + 2D\sqrt{2Bx}), 
	\]
	that is
	\[
	x 
	< 2 B C + 4BD\sqrt{2Bx}~.
	\]
	We can solve this inequality, to obtain
	\[
	x < 32 B^3 D^2 + 2 B C + 8B^2 D \sqrt{C}~.
	\]
	On the other hand, if $B x\leq A$, we have $x\leq \frac{A}{B}$.
	Taking the sum of these two case, we have the stated bound.
\end{proof}
\begin{lemma}
	\label{lemma: exponential}
	If $x,y \geq 0$ and $0 \leq p \leq 1$, then $(x+y)^p \leq x^p+y^p$.
\end{lemma}
\begin{proof}[Proof of Lemma~\ref{lemma: exponential}]
	Let $f(x) = (x+y)^p-x^p-y^p$. We can see that $f'(x) = p(x+y)^{p-1}-px^{p-1} \leq 0$ when $x,y \geq 0$. So $f(x) \leq f(0) = 0$. The inequality holds. 
\end{proof}

\begin{lemma}
	\label{lemma: bound log}
	If $x>0$, $\alpha >0 $, then $\ln(x) \leq \alpha(x^{\frac{1}{\alpha}}-1)$. 
\end{lemma}
\begin{proof}[Proof of Lemma~\ref{lemma: bound log}]
	Let $f(x) = \ln(x)-\alpha x^{\frac{1}{\alpha}}+\alpha$. $f'(x)= \frac{1}{x}-x^{\frac{1}{\alpha}-1}$ is positive when $0<x<1$, $f'(1)=0$ and $f'(x) <0$ when $x>1$. So 
	$f(x) \leq f(1)=0.$ The inequality holds. 
\end{proof}
\begin{lemma}
	\label{lemma:bounded_sum_squares}
	Suppose that $f$ is $M$-smooth and the stochastic gradients satisfy Assumption~\ref{as:subgaussion}.  The stepsizes are chosen as \eqref{eq:eta}, where $\alpha,\beta,\epsilon\geq0$.
	Then, 
	\begin{equation}
		\begin{aligned}
			\E\left[\sum_{t=1}^T \eta_t^2 \|\bg(\bx_t,\xi_t)\|^2\right] 
			& \leq K + \frac{4\alpha^2}{\beta^{1+2\epsilon}} (1+\ln T) \sigma^2 \\ 
			& \quad + \frac{4\alpha}{\beta^{\nicefrac{1}{2}+\epsilon}} \E\left[\sum_{t=1}^T \eta_t \|\nabla f(\bx_t)\|^2\right],
		\end{aligned}
	\end{equation}
	where in the case of  $\epsilon = 0$, $ K = 2 \alpha^2  \ln \left( \sqrt{\beta + 2T\sigma^2 }+\sqrt{2} \E \left[ \sqrt{ \sum_{t=1}^T \|\nabla f(\bx_t)\|^2}  \right] \right)$, when $\epsilon >0$, $ \small{K = \frac{\alpha^2}{2\epsilon \beta^{2\epsilon}}}$.
\end{lemma}
\begin{proof}[Proof of Lemma~\ref{lemma:bounded_sum_squares}]
	
	Using the assumption on the noise, we have
	\begin{align*}
		\exp \left(\frac{\E\left[\max_{1 \leq  i\leq T} \|\nabla f(\bx_i) -\bg_i\|^2\right]}{\sigma^2}\right) 
		& \leq \E\left[\exp\left(\frac{\max_{1 \leq  i\leq T} \|\nabla f(\bx_i) -\bg_i\|^2}{\sigma^2}\right)\right] \\
		&= \E\left[\max_{1 \leq  i\leq T} \exp\left(\frac{\|\nabla f(\bx_i) -\bg_i\|^2}{\sigma^2}\right)\right]\\ 
		& \leq \sum_{i=1}^T \E\left[\exp\left(\frac{\|\nabla f(\bx_i) -\bg_i\|^2}{\sigma^2}\right)\right] \\
		&= \sum_{i=1}^T \E\left[\E_i\left[\exp\left(\frac{\|\nabla f(\bx_i) -\bg_i\|^2}{\sigma^2}\right)\right]\right]\\
		& \leq T e,
	\end{align*}
	that implies
	\begin{equation}
		\label{eq:proof_lemma:bounded_sum_squares_eq1}
		\E\left[\max_{1 \leq  i\leq T} \|\nabla f(\bx_i) -\bg_i\|^2\right] \leq \sigma^2 (1+\ln T)~.
	\end{equation}
	Hence, when $\epsilon >0$, we have
	\begin{align*}
		\E\left[\sum_{t=1}^T \eta_t^2 \|\bg_t\|^2\right]
		&= \E\left[\sum_{t=1}^T \eta_{t+1}^2 \|\bg_t\|^2 + \sum_{t=1}^T \|\bg_t\|^2 (\eta_t^2 -\eta^2_{t+1})\right]\\
		&= \E\left[\sum_{t=1}^T \eta_{t+1}^2 \|\bg_t\|^2 + \sum_{t=1}^T \|\bg_t\|^2 (\eta_t +\eta_{t+1})(\eta_t -\eta_{t+1})\right]\\
		&\leq \E\left[\sum_{t=1}^T \eta_{t+1}^2 \|\bg_t\|^2 + \sum_{t=1}^T 2 \eta_t \|\bg_t\|^2 (\eta_t -\eta_{t+1})\right]\\
		&\leq \frac{\alpha^2}{2\epsilon \beta^{2\epsilon}} + 2 \eta_1 \E\left[\max_{1 \leq t \leq T} \eta_t \|\bg_t\|^2\right] \\
		&\leq \frac{\alpha^2}{2\epsilon \beta^{2\epsilon}} + 4 \eta_1 \E\left[\max_{1 \leq t \leq T} \eta_t \left(\|\bg_t-\nabla f(\bx_t)\|^2+\|\nabla f(\bx_t)\|^2\right)\right] \\
		&\leq \frac{\alpha^2}{2\epsilon \beta^{2\epsilon}} + 4 \eta^2_1 (1+\ln T) \sigma^2 + 4 \eta_1 \E\left[\sum_{t=1}^T \eta_t \|\nabla f(\bx_t)\|^2\right] \\
		&= \frac{\alpha^2}{2\epsilon \beta^{2\epsilon}} + \frac{4\alpha^2}{\beta^{1+2\epsilon}} (1+\ln T) \sigma^2 + \frac{4\alpha}{\beta^{\frac{1}{2}+\epsilon}} \E\left[\sum_{t=1}^T \eta_t \|\nabla f(\bx_t)\|^2\right], 
	\end{align*}
	where in second inequality we used Lemma~\ref{lemma:sum_bounded} and in fourth one we used \eqref{eq:proof_lemma:bounded_sum_squares_eq1}. Note that the analysis after the second inequality also holds when $\epsilon=0$. 
	
	And when $\epsilon = 0$, we have 
	\begin{align*}
		\E\left[ \sum_{t=1}^T \eta_{t+1}^2 \|\bg_t\|^2 \right] 
		& =  \E \left[ \sum_{t=1}^T \frac{\alpha^2 \|\bg_t\|^2 }{(\beta+\sum_{i=1}^t \| \bg_i \|^2)} \right] \\
		& \leq 2 \alpha^2 \E  \left[ \ln \left( \sqrt{\beta + \sum_{t=1}^T \|\bg_t\|^2}  \right) \right] \\
		& \leq 2 \alpha^2 \E \left[ \ln \left( \sqrt{\beta + 2 \sum_{t=1}^T \|\bg_t-\nabla f(\bx_t)\|^2} +\sqrt{2 \sum_{t=1}^T \|\nabla f(\bx_t)\|^2}  \right)\right] \\
		& \leq 2 \alpha^2  \ln \left( \sqrt{\beta + 2T\sigma^2 }+\sqrt{2} \E \left[ \sqrt{ \sum_{t=1}^T \|\nabla f(\bx_t)\|^2}  \right] \right), 
	\end{align*}
	where in first inequality we used Lemma~\ref{lemma: bound log} and in the third one we used Jensen's inequality. Putting things together, we have 
	\begin{align*}
		\E\left[\sum_{t=1}^T \eta_t^2 \|\bg_t\|^2\right]
		& = \E\left[\sum_{t=1}^T \eta_{t+1}^2 \|\bg_t\|^2 + \sum_{t=1}^T \|\bg_t\|^2 (\eta_t^2 -\eta^2_{t+1})\right]\\
		& \leq 2 \alpha^2  \ln \left( \sqrt{\beta + 2T\sigma^2 }+\sqrt{2} \E \left[ \sqrt{ \sum_{t=1}^T \|\nabla f(\bx_t)\|^2}  \right] \right)\\
		& \quad + \frac{4\alpha^2}{\beta} (1+\ln T) \sigma^2 + \frac{4\alpha}{\beta^{\frac{1}{2}}} \E\left[\sum_{t=1}^T \eta_t \|\nabla f(\bx_t)\|^2\right]~.\qedhere
	\end{align*}
\end{proof}

Now we can proof Theorem~\ref{thm:convex}. 

\begin{proof}[Proof of Theorem~\ref{thm:convex}]
For simplicity, denote by $\delta_t:=f(\bx_t)-f(\bx^\star)$ and by $\Delta := \sum_{t=1}^T \delta_t$.

From the update of SGD we have that
\begin{align*}
\|\bx_{t+1}-\bx^\star\|^2 - \|\bx_{t}-\bx^\star\|^2 
&=-2 \eta_t \langle \bg(\bx_t,\xi_t), \bx_t- \bx^\star\rangle + \eta_t^2 \|\bg(\bx_t,\xi_t)\|^2.
\end{align*}
Taking the conditional expectation with respect to $\xi_1, \dots, \xi_{t-1}$, we have that
\begin{align*}
E_t[\langle \bg(\bx_t,\xi_t), \bx_t-\bx^\star \rangle]
=  \langle \nabla f(\bx_t), \bx_t-\bx^\star \rangle
\geq \delta_t,
\end{align*}
where in the inequality we used the fact that $f$ is convex.
Hence, summing over $t=1$ to $T$, we have

\begin{align*}
\E\left[\sum_{t=1}^T \eta_t \delta_t\right] 
\leq \frac{1}{2}\|\bx^\star-\bx_1\|^2 + \frac{1}{2}\E\left[\sum_{t=1}^T \eta_t^2 \|\bg(\bx_t,\xi_t)\|^2 \right].
\end{align*}

From Lemma~\ref{lemma:smooth} and Lemma~\ref{lemma:bounded_sum_squares}, when $\epsilon>0$ we have that
\begin{equation}
\label{eq:convex_eq1}
\left(1- \frac{4\alpha M}{\beta^{\nicefrac{1}{2}+\epsilon}}\right)\E\left[\sum_{t=1}^T \eta_t \delta_t\right] 
\leq \frac{1}{2}\|\bx^\star-\bx_1\|^2 + \frac{\alpha^2}{4\epsilon \beta^{2\epsilon}} 
 + \frac{2\alpha^2}{\beta^{1+2\epsilon}} (1+\ln T) \sigma^2 .
\end{equation} 
On the other hand, when $\epsilon=0$ we have 
\begin{equation}
\label{eq:convex_eq2_epsl0}
\begin{aligned}
\left(1- \frac{4\alpha M}{\beta^{\nicefrac{1}{2}}}\right)\E\left[\sum_{t=1}^T \eta_t \delta_t\right] 
& \leq  \frac{1}{2} \| \bx_1 - \bx^{\star} \|^2 + \frac{2 \alpha^2}{\beta} (1+\ln T)\sigma^2 \\
& \quad +\alpha^2 \ln \left(  \sqrt{\beta + 2T\sigma^2} + 2\sqrt{M} \E \left[ \sqrt{ \Delta} \right]\right)~.
\end{aligned}
\end{equation}
We can also lower bound the l.h.s. of \eqref{eq:convex_eq1} and \eqref{eq:convex_eq2_epsl0} with 
\begin{equation}
\E\left[ \sum_{t=1}^T \eta_t \delta_t\right] 
\geq \E\left[ \eta_T \Delta \right] 
\geq \frac{ \left( \E \left[  \Delta^{\nicefrac{1}{2}-\epsilon}\right] \right)^\frac{1}{\nicefrac{1}{2}-\epsilon} }{\left( \E\left[  (\frac{1}{\eta_T})^{\frac{\nicefrac{1}{2}-\epsilon}{\nicefrac{1}{2}+\epsilon}}\right] \right)^{\frac{\nicefrac{1}{2}+\epsilon}{\nicefrac{1}{2}-\epsilon}}},
\end{equation}
where the second inequality is due to H\"{o}lder's inequality, i.e. $\E [B^p] \geq \dfrac{\E[AB]^p}{\E[A^q]^{\nicefrac{p}{q}}}$, with $\frac{1}{p} = \frac{1}{2}-\epsilon$, $\frac{1}{q} = \frac{1}{2}+\epsilon$, $A= (\frac{1}{\eta_T})^{\frac{1}{p}}$, and $B = \left[\eta_T \Delta\right]^{\frac{1}{p}}$.
We also have

\begin{align*}
&\frac{1}{\eta_T}
= \frac{1}{\alpha}\left(\beta+\sum_{t=1}^{T-1} \|\bg(\bx_t,\xi_t)\|^2\right)^{\nicefrac{1}{2}+\epsilon}\\
& \leq \frac{1}{\alpha}\left(\beta+2\sum_{t=1}^{T-1} \left(\|\nabla f(\bx_t)-\bg(\bx_t,\xi_t)\|^2+ \|\nabla f(\bx_t)\|^2\right)\right)^{\nicefrac{1}{2}+\epsilon} \\
& \leq \frac{1}{\alpha}\left(\beta+2\sum_{t=1}^{T-1} \big(\|\nabla f(\bx_t)-\bg(\bx_t,\xi_t)\|^2 + 2M \delta_t\big) \right)^{\nicefrac{1}{2}+\epsilon},
\end{align*}

where in the first inequality we used the elementary inequality $\|\bx+\by\|^2 \leq 2\|\bx\|^2 +2 \|\by\|^2$ and Lemma~\ref{lemma:smooth} in the second one.

Define

\[
\gamma
=\frac{1}{\alpha(1-\frac{4\alpha M}{\beta^{\nicefrac{1}{2}+\epsilon}})} \big(\|\bx^\star-\bx_1\|^2  + \frac{4\alpha^2}{\beta^{1+2\epsilon}} (1+\ln T) \sigma^2 \big) + K,
\]
where $K$ will be defined in the following for the case $\epsilon=0$ and $\epsilon>0$.

When $\epsilon >0$, we have
\begin{align}
& \frac{1}{\gamma^\frac{\nicefrac{1}{2}-\epsilon}{\nicefrac{1}{2}+\epsilon}}\left(\E \left[\Delta^{\nicefrac{1}{2}-\epsilon}\right] \right)^\frac{1}{\nicefrac{1}{2}+\epsilon} 
\leq  \alpha^\frac{\nicefrac{1}{2}-\epsilon}{\nicefrac{1}{2}+\epsilon} \E \left[ \left(\frac{1}{\eta_T}\right)^{\frac{\nicefrac{1}{2}-\epsilon}{\nicefrac{1}{2}+\epsilon}}\right] \nonumber\\
& \leq \E \Bigg[ \big( \beta+2\sum_{t=1}^{T-1} (\|\nabla f(\bx_t)-\bg(\bx_t,\xi_t)\|^2  + 2M \delta_t) \big)^{\nicefrac{1}{2}-\epsilon}  \Bigg]\nonumber\\
& \leq \E  \left[ \left(\beta+2\sum_{t=1}^{T-1} \|\nabla f(\bx_t)-\bg(\bx_t,\xi_t)\|^2\right)^{\nicefrac{1}{2}-\epsilon} \right] 
+ \E \left[ \left(4M \sum_{t=1}^{T-1} \delta_t \right)^{\nicefrac{1}{2}-\epsilon} \right]  \nonumber\\
& \leq \left(\beta+2(T-1) \sigma^2 \right)^{\nicefrac{1}{2}-\epsilon} +  (4M)^{\nicefrac{1}{2}-\epsilon}\E \left[ \Delta^{\nicefrac{1}{2}-\epsilon} \right], \label{eq:convex_eq3}
\end{align}

where in the third inequality we used Lemma~\ref{lemma: exponential} and we define $K=\frac{\frac{\alpha^2}{2\epsilon \beta^{2\epsilon}}}{\alpha(1-\frac{4\alpha M}{\beta^{\nicefrac{1}{2}+\epsilon}})}$.
Proceeding in the same way, for the case $\epsilon=0$ we get
{\small
\begin{align*}
\left( \E \left[  \sqrt{\Delta}\right]\right)^2 
& \leq \left(A +B\E \left[  \sqrt{\Delta}\right] \right) 
\times \left( C + D \ln \left(A + B \E \left[  \sqrt{\Delta}\right] \right) \right),
\end{align*}
}%
where 
$A = \sqrt{\beta + 2T \sigma^2}$, 
$B = 2\sqrt{M}$, 
$D = \frac{\alpha}{1-\frac{4\alpha M}{\sqrt{\beta}}}$ and 
$C = \frac{\beta \| \bx_1 -\bx^{\star} \|^2 + 4\alpha^2 (1+\ln T)\sigma^2}{2\alpha \beta (1-\frac{4\alpha M }{\sqrt{\beta}})}$.
Using Lemma~\ref{lemma:logsolvex}, we have that 
{\small
\[
\E \left[  \sqrt{\Delta }\right] 
\leq 32 B^3 D^2 + 2 B C + 8 B^2 D \sqrt{ C} + \frac{A}{B}~.
\]
}%
We use this upper bound in the logarithmic term, so that for $\epsilon\geq0$, we have \eqref{eq:convex_eq3} again, this time with $K=D \ln(2A + 32 B^4 D^2 + 2 B^2 C + 8 B^3 D \sqrt{C}) = O(\frac{\ln T}{1-\frac{4\alpha M}{\sqrt{\beta}}})$.

Hence, we proceed using Lemma~\ref{lemma:solvex} to have for $\epsilon\geq0$
\begin{align}
\E\left[ \Delta^{\nicefrac{1}{2}-\epsilon} \right] 
\leq  \max \left(  2^{\frac{\nicefrac{1}{2}+\epsilon}{\nicefrac{1}{2}-\epsilon}}(4M)^{\nicefrac{1}{2}+\epsilon}\gamma  ,  2^{\nicefrac{1}{2}+\epsilon}\gamma^{\nicefrac{1}{2}-\epsilon} \left( \beta+2T\sigma^2  \right)^{\nicefrac{1}{4}-\epsilon^2}  \right)~.
\end{align}

Using Jensen's inequality on the l.h.s. of last inequality concludes the proof.
\end{proof}

%% file: 2_Adaptive/adapt.tex
\subsection{Adaptive Convergence for Non-Convex Functions}
\label{sec:adap_noncvx}

We now prove that the Delayed AdaGrad stepsizes in \eqref{eq:eta} allow a faster convergence of the gradients to zero when the noise over the gradients is small.

Given that SGD is not a descent method, we are not aware of any result of convergence with an explicit rate for the last iterate for non-convex functions. Hence, here we will prove a convergence guarantee for the \emph{best iterate} over $T$ iterations rather than for the \emph{last one}.
Note that choosing a random stopping time as in \citet{GhadimiL13} would be equivalent in expectation to choose the best iterate. For simplicity, we choose to state the theorem for the best iterate.

\begin{theorem}
\label{thm:sgd_adaptive}
Suppose that $f$ is $M$-smooth and the stochastic gradients satisfy Assumption~\ref{as:subgaussion}. Let the stepsizes set as \eqref{eq:eta}, where $\alpha,\beta>0$, $\epsilon \in(0,\frac{1}{2})$, and $2\alpha M<\beta^{\frac12+\epsilon}$.
Then, the iterates of SGD satisfy the following bound
\begin{align*}
\E \left[ \min_{1\leq t \leq T} \|\nabla f(\bx_t)\|^{1-2\epsilon}\right] 
\leq \frac{1}{T^{\nicefrac{1}{2}-\epsilon}} \max \left(2^{\frac{\nicefrac{1}{2}+\epsilon}{\nicefrac{1}{2}-\epsilon}}\gamma,
 2^{\nicefrac{1}{2}+\epsilon} \left( \beta+2T\sigma^2 \right)^{\nicefrac{1}{4}-\epsilon^2}\gamma^{\nicefrac{1}{2}-\epsilon} \right),
\end{align*}
where 
$
\gamma = \begin{cases}
O \left( \frac{1+\alpha^2 \ln T}{\alpha (1-\frac{2\alpha}{\sqrt{\beta}})} \right) & \text{for } \epsilon =0\\
O \left(\frac{1+ \alpha^2 (\frac{1}{\epsilon}+\sigma^2 \ln T)}{\alpha (1-\frac{2\alpha}{\beta^{\nicefrac{1}{2}+\epsilon}})} \right)& \text{for } \epsilon >0~.
\end{cases}
$

\end{theorem}

\paragraph{Remark.}
Using Markov's inequality it's easy to get that, with probability at least $1-\delta$, 

\begin{align}
& \min_{1 \leq t \leq T} \| \nabla f(\bx_t) \|^2 
\leq \frac{1}{\delta^{\frac{1}{\nicefrac{1}{2}-\epsilon}} T} \max \left( 2^{\nicefrac{1}{2}+\epsilon} \gamma^{\frac{1}{\nicefrac{1}{2}-\epsilon}}, 
2^{\frac{\nicefrac{1}{2}+\epsilon}{\nicefrac{1}{2}-\epsilon}} (\beta + 2T \sigma^2)^{\nicefrac{1}{2}+\epsilon} \gamma \right)~.
\end{align}

This theorem mirrors Theorem~\ref{thm:convex}, proving again a convergence rate that is adaptive to the noise level. Hence, the same observations on adaptation to the noise level and convergence hold here as well. 
The main difference w.r.t. Theorem~\ref{thm:convex} is that here we only prove that the gradients are converging to zero rather than the suboptimality gap, because we do not assume convexity.

Note that such bounds were already known with an oracle tuning of the stepsizes, in particular with the knowledge of the variance of the noise, see, e.g., \citet{GhadimiL13}. In fact, the required stepsize in the deterministic case must be constant, while it has to be of the order of $O(\frac{1}{\sigma\sqrt{t}})$ in the stochastic case. However, here we obtain the same behavior automatically, without having to estimate the variance of the noise, thanks to the adaptive stepsizes. This shows for the first time a clear advantage of the global generalized AdaGrad stepsizes over plain SGD.

\begin{proof}[Proof of Theorem~\ref{thm:sgd_adaptive}]
For simplicity, denote by $\Delta:=\sum_{t=1}^T \|\nabla f(\bx_t)\|^2$.

From Lemma~\ref{lemma:basic_lemma}, we have
\[
\sum_{t=1}^T \E[\eta_t \|\nabla f(\bx_t)\|^2]
\leq f(\bx_1)-f^\star + \frac{M}{2} \E\left[\sum_{t=1}^T \eta_t^2 \|\bg(\bx_t,\xi_t)\|^2\right]~.
\]

Using Lemma~\ref{lemma:bounded_sum_squares}, we can upper bound the expected sum in the r.h.s. of the last inequality. When $\epsilon >0$,  we have

\begin{align}
& \left( \frac{1}{1-\frac{2 \alpha M}{\beta^{\nicefrac12+\epsilon}}} \right) \E\left[\sum_{t=1}^T \eta_t \|\nabla f(\bx_t)\|^2\right] 
\leq f(\bx_1)-f^\star
+\frac{\alpha^2 M}{4\epsilon \beta^{2\epsilon}} + \frac{2\alpha^2 \sigma^2 M}{\beta^{1+2\epsilon}} (1+\ln T)~.
\label{eq:sgd_adaptive_eq1}
\end{align}

When $\epsilon=0$, we have
\begin{align}
& \left( \frac{1}{1-\frac{2\alpha M}{\sqrt{\beta}}} \right) \E\left[\sum_{t=1}^T \eta_t \|\nabla f(\bx_t)\|^2\right]  \nonumber \\
& \leq f(\bx_1)-f^{\star} + M\alpha^2 \ln \left( \sqrt{\beta + 2T\sigma^2 }  + \sqrt{2} \E\left[ \sqrt{ \sum_{t=1}^T \eta_t \|\nabla f(\bx_t)\|^2 } \right] \right)  \nonumber \\
& \quad + \frac{2\alpha M }{\beta} (1+\ln T)\sigma^2~.
\label{eq:sgd_adaptive_eq2}
\end{align}

With similar methods in the proof of Theorem~\ref{thm:convex}, we lower bound the l.h.s. of  both \eqref{eq:sgd_adaptive_eq1} and \eqref{eq:sgd_adaptive_eq2} with 

\begin{align*}
\E\left[ \sum_{t=1}^T \eta_t \|\nabla f(\bx_t)\|^2\right]
\geq \E\left[ \eta_T \Delta\right] 
= \E\left[ \eta_T \Delta \right] 
\geq \frac{ \left( \E \left[ \Delta^{\nicefrac{1}{2}-\epsilon} \right] \right)^\frac{1}{\nicefrac{1}{2}-\epsilon} }{\left( \E \left[ (\frac{1}{\eta_T})^{\frac{\nicefrac{1}{2}-\epsilon}{\nicefrac{1}{2}+\epsilon}} \right] \right)^{\frac{\nicefrac{1}{2}+\epsilon}{\nicefrac{1}{2}-\epsilon}}}~.
\end{align*}

We also have 

\begin{align*}
& \frac{1}{\eta_T} = \frac{1}{\alpha}\left(\beta+\sum_{t=1}^{T-1} \|\bg(\bx_t,\xi_t)\|^2\right)^{\nicefrac{1}{2}+\epsilon} \\
& \leq \frac{1}{\alpha}\left(\beta+2\sum_{t=1}^{T-1} \left(\|\nabla f(\bx_t)-\bg(\bx_t,\xi_t)\|^2+ \|\nabla f(\bx_t)\|^2\right)\right)^{\nicefrac{1}{2}+\epsilon}~.
\end{align*}

Define 

\[
\gamma = \frac{1}{\alpha (1-\frac{2\alpha M}{\beta^{\nicefrac{1}{2}+\epsilon}})} \left( f(\bx_1)-f^{\star}+ \frac{2\alpha^2 M }{\beta^{1+2\epsilon}}\sigma^2 \right) + K,
\]

where $K$ will be defined separately for the case $\epsilon = 0$ and $\epsilon >0.$

When $\epsilon >0$, we have 

\begin{align}
\left( \E \left[ \Delta^{\nicefrac{1}{2}-\epsilon} \right] \right)^\frac{1}{\nicefrac{1}{2}-\epsilon} 
& \leq  \alpha \gamma \left( \E \left[ (\frac{1}{\eta_T})^{\frac{\nicefrac{1}{2}-\epsilon}{\nicefrac{1}{2}+\epsilon}} \right] \right)^{\frac{\nicefrac{1}{2}+\epsilon}{\nicefrac{1}{2}-\epsilon}} \nonumber \\
& \leq \gamma \left( \E  \left[ \left(\beta+2\sum_{t=1}^{T-1} \|\nabla f(\bx_t)-\bg(\bx_t,\xi_t)\|^2\right)^{\nicefrac{1}{2}-\epsilon} \right] \right.\nonumber\\
& \quad \left. + 2\E \left[ \left(\sum_{t=1}^{T-1} \|\nabla f(\bx_t)\|^2 \right)^{\nicefrac{1}{2}-\epsilon} \right] \right)^{\frac{\nicefrac{1}{2}+\epsilon}{\nicefrac{1}{2}-\epsilon}} \nonumber\\
& \leq \gamma \left( \left(\beta+2T \sigma^2 \right)^{\nicefrac{1}{2}-\epsilon} + 2\E \left[ \Delta^{\nicefrac{1}{2}-\epsilon}\right]  \right)^{\frac{\nicefrac{1}{2}+\epsilon}{\nicefrac{1}{2}-\epsilon}}~. \label{eq:adapt_eq3}
\end{align}
where in this case we define $K = \frac{\frac{\alpha^ M }{4\epsilon \beta^{2\epsilon}}}{\alpha(1-\frac{2\alpha M}{\beta^{\nicefrac{1}{2}+\epsilon}})}.$
Proceeding in the same way, when $\epsilon=0$, we have 

\begin{align*}
\left( \E \left[ \sqrt{\Delta } \right] \right)^2 
\leq \left(A + B \E \left[ \sqrt{\Delta } \right]  \right) 
 \times \left(  C + D \ln \left(A +B \E \left[ \sqrt{\Delta } \right] \right) \right),
\end{align*}

where $A = \sqrt{\beta + 2T \sigma^2}$, 
$B = \sqrt{2}$,
$D = \frac{\alpha M}{1-\frac{2\alpha M}{\sqrt{\beta}}}$, 
$C = \frac{\beta (f(\bx_1)-f^{\star}) + 2\alpha (1+\ln T)\sigma^2}{\alpha \beta (1-\frac{2\alpha M }{\sqrt{\beta}})}$. 

Using Lemma~\ref{lemma:logsolvex}, we have that  

\begin{align*}
\E \left[ \sqrt{\Delta } \right] 
\leq 32 B^3 D^2 + 2 B C + 8 B^2 D \sqrt{ C} + \frac{A}{B}~.
\end{align*}

Similar with Theorem~\ref{thm:convex}, we use this upper bound in the logarithmic term so that for $\epsilon \geq 0$, we have \eqref{eq:adapt_eq3} again, this time with $K = D \ln (2A + 32 B^4 D^2 + 2 B^2 C +8 B^3 D \sqrt{C}) = O(\frac{\ln T}{1-\frac{2\alpha M }{\beta}}).$

Hence, we proceed using Lemma~\ref{lemma:solvex} to have for $\epsilon \geq 0$

\begin{align*}
\E\left[ \Delta^{\nicefrac{1}{2}-\epsilon} \right] 
 \leq \max \left(2^{\frac{\nicefrac{1}{2}+\epsilon}{\nicefrac{1}{2}-\epsilon}} \gamma,
2^{\nicefrac{1}{2}+\epsilon} \left( \beta+2T \sigma^2 \right)^{\nicefrac{1}{4}-\epsilon^2}\gamma^{\nicefrac{1}{2}-\epsilon} \right)~.
\end{align*}

Lower bounding $\E \left[ \Delta^{\nicefrac{1}{2}-\epsilon} \right]$ by $T^{\nicefrac{1}{2}-\epsilon} \E \left[\min_{1\leq t \leq T} \|\nabla f(\bx_t)\|^{1-2\epsilon}\right]$,  we have the stated bound. 

\end{proof}

%% file: 2_Adaptive/highp.tex
\section{A High Probability Analysis for SGD with Momentum}
\label{sec:adap_highp}
In the previous sections, we presented the convergence rates in expectation. Indeed, the classic analysis of convergence for SGD in the nonconvex setting uses analysis in expectation. However, expectation bounds do not rule out extremely bad outcomes. As pointed out by \citet{HarveyLPR19}, it is a misconception that for the algorithms who have expectation bounds it is enough to pick the best of several independent runs to have a high probability guarantee: It can actually be a computational inefficient procedure. Moreover, in practical applications like deep learning, it is often the case that only one run of the algorithm is used since the training process may take a long time. Hence, it is essential to get high probability bounds that guarantee the performance of the algorithm on single runs.

In this section, we overcome this problem by proving a high probability analysis of  SGD with momentum and adaptive learning rates.

\subsection{A General Analysis for Algorithms with Momentum}
\label{sec:lemma}
\begin{sloppypar}
We consider a generic stochastic optimization algorithm with Polyak's momentum~\citep{Polyak64,Qian99,SutskeverMDH13}, also known as the Heavy-ball algorithm or classic momentum, see Algorithm~\ref{alg:momentum}.
\end{sloppypar}

\begin{algorithm}
\caption{Algorithms with Momentum}
\label{alg:momentum}
\begin{algorithmic}[1]
\STATE \textbf{Input:} $\bm_0 = 0$, $ \{\bta_t\}_{t=1}^T$, $0 < \mu \leq 1$, $\bx_1 \in \R^d$
\FOR{$t = 1, \dots, T$}
\STATE Get stochastic gradient $\bg_t = g(\bx_t, \xi_t)$
\STATE $\bm_t = \mu \bm_{t-1} + \bta_t \bg_t$
\STATE $\bx_{t+1} = \bx_t - \bm_t$
\ENDFOR
\end{algorithmic}
\end{algorithm}

\paragraph{Two forms of momentum, but not equivalent.}
First, we want to point out that there two forms of Heavyball algorithms are possible.
The first one is in Algorithm~\ref{alg:momentum}, while the second one is
\begin{equation}
\label{eq:momentum2}
\begin{split}
\bm_t &= \mu \bm_{t-1} + \bg_t, \\
\bx_{t+1} &= \bx_t - \bta_t \bm_t~.
\end{split}
\end{equation}
\begin{sloppypar}
This second is used in many practical implementations, see, for example, PyTorch~\citep{PyTorch19}. It would seem that there is no reason to prefer one over the other. However, here we argue that the classic form of momentum is the right one if we want to use adaptive learning rates. To see why, let's unroll the updates in both cases.
Using the update in Algorithm~\ref{alg:momentum}, we have
\end{sloppypar}

\[
\bx_{t+1} = \bx_t - \bta_t \bg_t - \mu \bta_{t-1} \bg_{t-1} - \mu^2 \bta_{t-1} \bg_{t-2} \dots,
\]
while using the update in \eqref{eq:momentum2}, we have
\[
\bx_{t+1} = \bx_t - \bta_t \bg_t - \mu \bta_{t} \bg_{t-1} - \mu^2 \bta_{t} \bg_{t-2} \dots~.
\]
In words, in the first case the update is composed of a sum of weighted gradients, each one multiplied by a learning rate we decided in the past. On the other hand, in the update \eqref{eq:momentum2} the update is composed of a sum of weighted gradients, each one multiplied by the \emph{current} learning rate. From the analysis point of view, the second update destroys the independence between the past and the future, introducing a dependency that breaks our analysis, unless we introduce very strict conditions on the gradients. On the other hand, the update in Algorithm~\ref{alg:momentum} allows us to carry out the analysis because each learning rate was chosen only with the knowledge of the past. Note that this is a known problem in adaptive algorithms: the lack of independence between past and present is exactly the reason why Adam fails to converge on simple 1d convex problems, see for example the discussion in \citet{SavareseMBM19}.

It is interesting to note that usually people argue that these two types of updates for momentum are usually considered equivalent. This seems indeed true only if the learning rates are not adaptive.


\paragraph{Assumptions on learning rates} Note that in the pseudo-code we do not specify the learning rates $\bta_t \in \R^d$. In fact, our analysis covers the case of generic learning rates and adaptive ones too. We only need the following assumptions on the stepsizes $\bta_t$: 
\begin{assumptionC}
\label{as:eta_non_increasing}
$\bta_t$ is non-increasing, i.e., $\bta_{t+1} \leq \bta_t$, $\forall t$. 
\end{assumptionC} 
\begin{assumptionC}
\label{as:eta_independent}
$\bta_t$ is independent with $\xi_t$.
\end{assumptionC}

The first assumption is very common \citep[e.g.,][]{DuchiHS11,ReddiKK18,ChenLSH19,ZhouTYCG18}. Indeed, AdaGrad has the non-increasing step sizes by the definition. Also, \citet{ReddiKK18} has claimed that the main issue of the divergences of Adam and RMSProp lies in the positive definiteness of $1/\bta_t - 1/\bta_{t-1}$. 

\begin{sloppypar}
The need for the second assumption is technical and shared by similar analysis \citep{SavareseMBM19}.
\end{sloppypar}
\paragraph{High probability guarantee} Adaptive learning rates and in general learning rates that are decided using previous gradients become stochastic variables. This makes the high probability analysis more complex. Hence, we use a new concentration inequality for martingales in which the variance is treated as a random variable, rather than a deterministic quantity. We use this concentration in the proof of Lemma~\ref{lemma:momentum}. Our proof, in the Appendix, merges ideas from the related results in \citet[Theorem 1]{BeygelzimerLLRS11} and \citet[Lemma 2]{LanNS12}. A similar result has also been shown by \citet[Lemma~6]{JinNGKJ19}. 
The following general lemma allows to analyze these kinds of algorithms. We will then instantiate it for Delayed AdaGrad stepsize.

\begin{lemma}
\label{lemma:sub_gaussian}
Assume that $Z_1, Z_2, ..., Z_T$ is a martingale difference sequence with respect to $\xi_1, \xi_2, ..., \xi_T$ and $\E_t \left[\exp(Z_t^2/\sigma_t^2)\right] \leq \exp(1)$ for all $1\leq t \leq T$, where $\sigma_t$ is a sequence of random variables with respect to $\xi_1, \xi_2, \dots, \xi_{t-1}$. 
Then, for any fixed $\lambda > 0$ and $\delta \in (0,1)$, with probability at least $1-\delta$, we have 
\[
\sum_{t=1}^T Z_t \leq \frac{3}{4} \lambda \sum_{t=1}^T \sigma_t^2 + \frac{1}{\lambda} \ln \frac{1}{\delta}~.
\]
\end{lemma}

\begin{proof}[Proof of Lemma~\ref{lemma:sub_gaussian}]
Set $\tilde{Z_t} = Z_t / \sigma_t$. By the assumptions of $Z_t$ and $\sigma_t$, we have
\[
\E_t [\tilde{Z_t}] = \frac{1}{\sigma_t} \E_t [Z_t] = 0 
\quad
\text{ and }
\quad 
\E_t \left[ \exp \left( \tilde{Z}_t^2 \right) \right] \leq \exp (1)~.
\]
By Jensen's inequality, it follows that for any $c \in [0,1]$, 
\begin{equation}
\label{eq: sub_gaussian_coef}
\E_t \left[ \exp \left(c \tilde{Z}_t^2 \right) \right]
= \E_t \left[ \left( \exp \left(\tilde{Z}_t^2\right) \right)^c \right]
\leq \left( \E_t \left[ \exp \left(\tilde{Z}_t^2\right)\right] \right)^c \leq \exp(c)~. 
\end{equation}
Also it can be verified that $\exp (x) \leq x + \exp (9x^2/16)$ for all $x$, hence for $ |\kappa| \in [0, 4/3]$ we get 
\begin{equation}
\label{eq: sub_gaussian1}
\E_t \left[ \exp \left( \kappa \tilde{Z}_t \right) \right] 
\leq \E_t \left[ \exp \left(9 \kappa^2 \tilde{Z}_t^2/16 \right) \right] 
\leq \exp \left( 9\kappa^2/16\right)  
\leq \exp \left( 3 \kappa^2/4\right)~.
\end{equation}
where in the second inequality, we used~\eqref{eq: sub_gaussian_coef}.
Besides, $k x \leq 3k^2/8 + 2x^2/3 $ holds for any $k$ and $x$. Hence for $ |\kappa| \geq 4/3$, we get
\begin{equation}
\label{eq: sub_gaussian2}
\E_t  \left[ \exp \left( \kappa \tilde{Z}_t \right) \right] 
\leq \exp \left( 3\kappa^2/8 \right) \E_t \left[\exp \left(2\tilde{Z}_t^2/3\right) \right]
\leq \exp \left( 3\kappa^2/8 + 2/3 \right) \leq \exp \left(3\kappa^2/4  \right),
\end{equation}
where in the second inequality we used~\eqref{eq: sub_gaussian_coef}. 
Combining \eqref{eq: sub_gaussian1} and \eqref{eq: sub_gaussian2}, we get $\forall \kappa$, 
\begin{equation}
\label{eq:sub_gaussian3}
\E_t \left[ \exp \left( \kappa \tilde{Z}_t \right) \right] \leq \exp \left( 3\kappa^2/4  \right)~. 
\end{equation}
Note that the above analysis for~\eqref{eq:sub_gaussian3} still hold when $\kappa$ is a random variable with respect to $\xi_1,\xi_2,\dots, \xi_{t-1}$.
So for $Z_t$, we have $\E_t \left[ \exp \left(\lambda Z_t \right) \right] \leq \exp \left( 3\lambda^2 \sigma_t^2/4 \right), \quad \lambda>0$.

Define the random variables $Y_0 = 1$ and $Y_t = Y_{t-1} \exp \left(\lambda Z_t - 3\lambda^2 \sigma_t^2/4  \right), \quad 1\leq t\leq T$.
So, we have $E_t Y_t = Y_{t-1} \exp \left( -3\lambda^2 \sigma_t^2/4  \right) \cdot \E_t \left[ \exp \left( \lambda Z_t \right) \right] \leq Y_{t-1}$.
Now, taking full expectation over all variables $\xi_1, \xi_2, \dots ,\xi_T$, we have 
\[
\E Y_T \leq \E Y_{T-1} \leq \dots \leq \E Y_0 = 1~. 
\]
By Markov's inequality, $P \left( Y_T \geq \frac{1}{\delta} \right) \leq \delta$, and 
$
Y_T = \exp \left( \lambda \sum_{t=1}^T Z_t- \frac{3}{4} \lambda^2 \sum_{t=1}^T \sigma_t^2 \right), 
$
we have 
\begin{align*}
P \left( Y_T \geq \frac{1}{\delta} \right) 
& = P \left( \lambda \sum_{t=1}^T Z_t- \frac{3}{4} \lambda^2 \sum_{t=1}^T \sigma_t^2  \geq \ln \frac{1}{\delta} \right) \\
& = P \left( \sum_{t=1}^T Z_t \geq  \frac{3}{4} \lambda \sum_{t=1}^T \sigma_t^2 + \frac{1}{\lambda} \ln \frac{1}{\delta} \right) \\
& \leq \delta,
\end{align*}
which completes the proof. 
\end{proof}
We can now present a general lemma, that allows to analyze SGD with momentum with adaptive learning rates. We will then instantiate it for particular examples.
\begin{lemma}
\label{lemma:momentum}
Assume $f$ is $M$-smooth and the stochasitc gradient is unbiased and satisfies Assumption~\ref{as:subgaussion}. Also, suppose that the stepsizes satisfy Assumption~\ref{as:eta_non_increasing} and~\ref{as:eta_independent}.  Then, for any $\delta \in (0,1)$, with probability at least $1-\delta$, the iterates of Algorithm \ref{alg:momentum} satisfy
\begin{align*}
\sum_{t=1}^{T} \langle \bta_t, \nabla f(\bx_t) ^2 \rangle 
& \leq \frac{3 \|\bta_1\|\sigma^2(1- \mu^T)^2 }{(1- \mu)^2}\ln \frac{1}{\delta} 
+ 2(f(\bx_1) - f^{\star}) 
+ \frac{M(3-\mu)}{1- \mu}\sum_{t=1}^{T}  \| \bta_t\bg_t \|^2
~. 
\end{align*}
\end{lemma}
Lemma~\ref{lemma:momentum} accomplishes the task of upper bounding the inner product \\ $\sum_{t=1}^{T}\langle \bta_t \nabla f(\bx_t), \bm_t \rangle$. Then, it is easy to lower bound the l.h.s by $\sum_{t=1}^{T} \langle \bta_T, \nabla f(\bx_t)^2 \rangle$ using the assumption~\ref{as:eta_non_increasing} followed by the upper bound of $\sum_{t=1}^{T} \| \bta_t\bg_t \|^2$ based on the setting of $\bta_t$. 

To prove Lemma~\ref{lemma:momentum}, we first need the following technical Lemma.
\begin{lemma}
\label{lemma:doublesum}
$\forall T \geq 1$, it holds 
\[
\sum_{t=1}^{T} a_t \sum_{i=1}^{t} b_i  = \sum_{t=1}^{T} b_t \sum_{i=t}^{T} a_i
\quad
\text{and}
\quad
\sum_{t=1}^{T} a_t \sum_{i=0}^{t-1} b_i  = \sum_{t=0}^{T-1} b_t \sum_{i=t+1}^{T}a_i~.
\]
\end{lemma}

\begin{proof}
We prove these equalities by induction. 
When $T=1$, they obviously hold.
Now, for $k < T$, assume that $\sum_{t=1}^{k} a_t \sum_{i=1}^{t} b_i  = \sum_{t=1}^{k} b_t \sum_{i=t}^{k} a_i$. 
Then, we have
\begin{align*}
\sum_{t=1}^{k+1} a_t \sum_{i=1}^{t} b_i  
& = \sum_{t=1}^{k} a_t \sum_{i=1}^{t} b_i  + a_{k+1} \sum_{i=1}^{k+1} b_i 
=  \sum_{t=1}^{k} b_t \sum_{i=t}^{k} a_i + a_{k+1} \sum_{i=1}^{k} b_i + a_{k+1}b_{k+1}\\
& = \sum_{t=1}^{k} b_t \sum_{i=t}^{k+1} a_i + a_{k+1}b_{k+1}
=  \sum_{t=1}^{k+1} b_t \sum_{i=t}^{k+1} a_i~.
\end{align*}
Hence, by induction, the equality is proved.

\begin{sloppypar}
Similarly, for second equality assume that for $k < T$ we have
$\sum_{t=1}^{k} a_t \sum_{i=0}^{t-1} b_i  = \sum_{t=0}^{k-1} b_t \sum_{i=t+1}^{k}a_i$.
\end{sloppypar}
Then, we have
\begin{align*}
\sum_{t=1}^{k+1} a_t \sum_{i=0}^{t-1} b_i  
& = \sum_{t=1}^{k} a_t \sum_{i=0}^{t-1} b_i  + a_{k+1} \sum_{i=0}^{k} b_i 
=  \sum_{t=0}^{k-1} b_t \sum_{i=t+1}^{k}a_i+ a_{k+1} \sum_{i=0}^{k-1} b_i + a_{k+1}b_{k}\\
& = \sum_{t=0}^{k-1} b_t \sum_{i=t+1}^{k+1} a_i + a_{k+1}b_{k}
=  \sum_{t=0}^{k} b_t \sum_{i=t+1}^{k+1} a_i~. 
\end{align*}
By induction, we finish the proof.  
\end{proof}

Now we can proof Lemma~\ref{lemma:momentum}. 

\begin{proof}[Proof of Lemma~\ref{lemma:momentum}]
By the smoothness of $f$ and the definition of $\bx_{t+1}$, we have 
\begin{equation}
\label{eq:adamsmooth-coodinate}
\begin{aligned}
f(\bx_{t+1})-f(\bx_t) 
\leq -\langle \nabla f(\bx_t), \bm_t \rangle + \frac{M}{2} \| \bm_t \|^2~. 
\end{aligned}
\end{equation}
We now upper bound $-\langle \nabla f(\bx_t), \bm_t \rangle$. 
\begin{align*}
& -\langle \nabla f(\bx_t), \bm_t \rangle \\
& = -\mu \langle \nabla f(\bx_t), \bm_{t-1}\rangle - \langle \nabla f(\bx_t), \bta_t \bg_t \rangle \\
& = - \mu \langle \nabla f(\bx_{t-1}), \bm_{t-1}\rangle 
- \mu \langle \nabla f(\bx_t) - \nabla f(\bx_{t-1}), \bm_{t-1}\rangle - \langle \nabla f(\bx_t), \bta_t \bg_t \rangle \\
& \leq - \mu \langle \nabla f(\bx_{t-1}), \bm_{t-1}\rangle + \mu \| \nabla f(\bx_t) - \nabla f(\bx_{t-1}) \| \| \bm_{t-1}\| - \langle \nabla f(\bx_t), \bta_t\bg_t \rangle\\
& \leq - \mu \langle \nabla f(\bx_{t-1}),\bm_{t-1}\rangle + \mu M \|\bm_{t-1}\|^2 -\langle \nabla f(\bx_t), \bta_t \bg_t \rangle, 
\end{align*}
where the second inequality is due to the smoothness of $f$.
Hence, iterating the inequality we have
\begin{align*}
-\langle \nabla f(\bx_t),\bta_t \bm_t \rangle 
& \leq - \mu^2 \langle \nabla f(\bx_{t-2}), \bm_{t-2} \rangle + \mu^2 M \|\bm_{t-2}\|^2 + \mu M \|\bm_{t-1}\|^2\\
& \quad - \mu \langle \nabla f(\bx_{t-1}), \bta_{t-1} \bg_{t-1} \rangle - \langle \nabla f(\bx_t), \bta_t \bg_t \rangle \\
& \leq M\sum_{i=1}^{t-1} \mu^{t-i} \|\bm_{i}\|^2 - \sum_{i=1}^t \mu^{t-i} \langle \nabla f(\bx_i), \bta_i \bg_i\rangle~. \qedhere
\end{align*}
Thus, denoting by $\bps_t = \bg_t - \nabla f(\bx_t)$ and summing \eqref{eq:adamsmooth-coodinate} over $t$ from $1$ to $T$, we obtain 
\begin{align*}
f^\star-f(\bx_1) 
& \leq f(\bx_{T+1})-f(\bx_1) \\
& \leq  M \sum_{t=1}^{T}\sum_{i=1}^{t-1} \mu^{t-i} \|\bm_{i} \|^2
- \sum_{t=1}^{T} \sum_{i=1}^t \mu^{t-i}\langle \nabla f(\bx_i),  \bta_i \bg_i\rangle +  \sum_{t=1}^{T} \frac{M}{2} \| \bm_t \|^2\\	
& \leq  M \sum_{t=1}^{T}\sum_{i=1}^{t-1} \mu^{t-i} \|\bm_{i} \|^2
- \sum_{t=1}^{T} \langle \bta_t, \nabla f(\bx_t) ^2 \rangle 
- \sum_{t=1}^{T} \sum_{i=1}^t \mu^{t-i}\langle \nabla f(\bx_i),  \bta_i \bps_i\rangle \\
& \quad + \sum_{t=1}^{T} \frac{M}{2} \| \bm_t \|^2~.
\end{align*}
By Lemma \ref{lemma:doublesum}, we have
\begin{align*}
M \sum_{t=1}^{T}\sum_{i=1}^{t-1}  \mu^{t-i}\| \bm_{i}\|^2
\leq \frac{M}{1-\mu} \sum_{t=1}^{T} \| \bm_{t}\|^2~.
\end{align*}
Also, by Lemma~\ref{lemma:doublesum}, we have  
\begin{align*}
- \sum_{t=1}^{T} \sum_{i=1}^t \mu^{t-i}\langle \nabla f(\bx_i), \bta_i \bps_i\rangle
&= - \sum_{t=1}^{T} \mu^{-t} \langle \nabla f(\bx_t), \bta_t \bps_t\rangle \sum_{i=t}^T \mu^{i} \\
& = - \frac{1}{1-\mu}\sum_{t=1}^{T} \langle \nabla f(\bx_t), \bta_t  \bps_t\rangle (1- \mu^{T-t+1})
\triangleq S_T~.
\end{align*}
\begin{sloppypar}
We then upper bound $S_T$. 
Denote by $L_t:= - \frac{1- \mu^{T-t+1}}{1- \mu}\langle \nabla f(\bx_t), \bta_t \bps_t \rangle$, and $N_t := \frac{(1- \mu^{T-t+1})^2}{(1- \mu)^2}\| \bta_t\nabla f(\bx_t)\|^2 \sigma^2$.
\end{sloppypar}
Using the assumptions on the noise, for any $1 \leq t \leq T$, we have
\begin{align*}
\exp\left(\frac{L_t^2 }{ N_t }\right) 
& \leq \exp \left(\frac{\| \bta_t \nabla f(\bx_t) \|^2 \| \bps_t \|^2 (1- \mu^{T-t+1})^2/(1- \mu)^2 }{N_t}\right)\\
& = \exp \left(\frac{\| \bps_t\|^2 }{ \sigma^2 }\right) 
\leq \exp(1)~. 
\end{align*}
We can also see that for any $t$, $\E_t [ L_t]  = - \frac{1- \mu^{T-t+1}}{1- \mu}\sum_{i=1}^{d} \bta_{t,i} \nabla f(\bx_t)_i \E_t [\bps_{t,i}] = 0$.
Thus, from Lemma~\ref{lemma:sub_gaussian}, with probability at least $1-\delta$, any $\lambda >0$, we have 
\begin{align*}
S_T 
&= \sum_{t=1}^T L_t \leq \frac{3}{4}\lambda \sum_{t=1}^T N_t + \frac{1}{\lambda} \ln \frac{1}{\delta}\\
& \leq\frac{3\lambda(1- \mu^T)^2}{4(1- \mu)^2}\sum_{t=1}^T \| \bta_t\nabla f(\bx_t)\|^2 \sigma^2  + \frac{1}{\lambda} \ln \frac{1}{\delta} \\
&\leq \frac{3\lambda \| \bta_1\|(1- \mu^T)^2}{4(1- \mu)^2}\sum_{t=1}^T \langle \bta_t, \nabla f(\bx_t)^2 \rangle \sigma^2  + \frac{1}{\lambda} \ln \frac{1}{\delta}~.\qedhere
\end{align*}
Finally, we upper bound $\sum_{t=1}^{T} \| \bm_t \|^2$. From the convexity of $\| \cdot \|^2$, we have
\begin{align*}
\| \bm_t \|^2 = \left\| \mu \bm_{t-1} + (1- \mu) \frac{\bta_t \bg_t}{1-\mu} \right\|^2 
\leq \mu \| \bm_{t-1}\|^2 + \frac{1}{1- \mu} \| \bta_t \bg_t \|^2~. 
\end{align*}
Summing over $t$ from $1$ to $T$, we have 
\begin{align*}
\sum_{t=1}^T \| \bm_t \|^2
& \leq \sum_{t=1}^T \mu \|  \bm_{t-1}\|^2 + \frac{1}{1-\mu}\sum_{t=1}^T \| \bta_t \bg_t \|^2 \\
& =  \sum_{t=1}^{T-1} \mu \| \bm_t\|^2 + \frac{1}{1-\mu}\sum_{t=1}^T \| \bta_t \bg_t \|^2 \\
& \leq  \sum_{t=1}^T \mu \| \bm_t\|^2 + \frac{1}{1-\mu}\sum_{t=1}^T \| \bta_t \bg_t \|^2,
\end{align*}
where in the first equality we used $\bm_0=0$. Reordering the terms, we have that
\[
\sum_{t=1}^T \| \bm_t \|^2 \leq \frac{1}{(1-\mu)^2}\sum_{t=1}^T \| \bta_t \bg_t \|^2~.
\]
Combining things together, and taking $\lambda = \frac{2(1- \mu)^2}{3\| \bta_1 \| (1- \mu^T)^2 \sigma^2} $, with probability at least $1- \delta$, we have 
\begin{align*}
f^\star-f(\bx_1) 
& \leq \frac{1}{\lambda} \ln \frac{1}{\delta}
+ \sum_{t=1}^{T}\left[\left( \frac{M}{2} + \frac{M }{1- \mu }\right) \| \bta_t \bg_t \|^2 \right. \\
& \left. \quad -  \left(1 - \frac{3\lambda \| \bta_1 \| (1- \mu^T)^2 \sigma^2 }{4(1-\mu)^2} \right) \langle \bta_t ,\nabla f(\bx_t)^2 \rangle\right]\\
&= \frac{3 \|\bta_1\|(1-\mu^T)^2 \sigma^2}{2(1-\mu)^2} \ln \frac{1}{\delta} 
+ \sum_{t=1}^{T}\left[\frac{(3-\mu)M}{2(1- \mu)}\| \bta_t \bg_t \|^2
-  \frac{1}{2}\langle \bta_t ,\nabla f(\bx_t)^2 \rangle\right]~.
\end{align*}
Rearranging the terms, we get the stated bound. 
\end{proof}

\subsection{SGD with Momentum with $\frac{1}{\sqrt{t}}$ Learning Rates}
\label{ssec:sqrt}
As a warm-up, we now use Lemma~\ref{lemma:momentum} to prove a high probability convergence guarantee for the simple case of deterministic learning rates of $\eta_{t,i}=\frac{c}{\sqrt{t}}$.
\begin{theorem}
\label{thm:momentum_sqrt}
Let $T$ the number of iterations of Algorithm~\ref{alg:momentum}.
Assume $f$ is $M$-smooth and the stochastic gradient is unbiased and satisfies Assumption~\ref{as:subgaussion}. Set step size $\bta_t$ as $ \eta_{t,i} = \frac{c}{\sqrt{t}}, i= 1, \dots, d$, where $ c \leq \frac{1- \mu^T}{4M(3-2\mu)}$. Then, for any $\delta \in (0,1)$, with probability at least $1-\delta$, the iterates of Algorithm \ref{alg:momentum} satisfy
\begin{align*}
\min_{1\leq t \leq T} \| \nabla f(\bx_t) \|^2
& \leq \frac{4(f(\bx_1) - f^{\star})}{c \sqrt{T}} +  \frac{6(1- \mu^T)^2\sigma^2}{(1-\mu)^2 \sqrt{T}} 
+ \frac{4(3-\mu)cM\sigma^2 \ln \frac{2Te}{\delta} \ln T }{(1- \mu)\sqrt{T}}~.
\end{align*}
\end{theorem}

The proof of this Theorem~\ref{thm:momentum_sqrt} makes use of the following additional Lemma on the tail of sub-gaussian noise.
\begin{lemma}
\label{lemma:max_bound}
Assume ~\ref{as:subgaussion}, then for any $\delta \in (0,1)$, with probability at least $1-\delta$, we have 
\[
\max_{1 \leq t \leq T} \| \bg_t - \nabla f(\bx_t) \|^2 \leq \sigma^2 \ln \frac{Te}{\delta}~. 
\]
\end{lemma}
\begin{proof}
By Markov's inequality, for any $A>0$, 
\begin{align*}
& \Pr \left(\max_{1\leq t \leq T} \|\bg_t- \nabla f(\bx_t)\|^2 > A\right) \\
& = \Pr \left(\exp \left( \frac{\max_{1\leq t \leq T} \|\bg_t- \nabla f(\bx_t)\|^2 }{\sigma^2} \right) > \exp \left( \frac{A}{\sigma^2}\right) \right)\\
& \leq \exp \left(-\frac{A}{\sigma^2} \right)\E \left[\exp \left(\frac{\max_{1\leq t \leq T} \|\bg_t- \nabla f(\bx_t)\|^2 }{\sigma^2}\right) \right]\\
& =  \exp \left(-\frac{A}{\sigma^2} \right) \E \left[ \max_{1\leq t \leq T} \exp \left(\frac{\|\bg_t - \nabla f(\bx_t)\|^2 }{\sigma^2}\right) \right]\\
& \leq \exp \left(-\frac{A}{\sigma^2} \right) \sum_{t=1}^T \E \left[ \exp \left( \frac{\| \nabla f(\bx_t)- \bg_t \|^2 }{\sigma^2} \right) \right] 
\leq \exp \left(-\frac{A}{\sigma^2}+1 \right) T~. \qedhere
\end{align*}
\end{proof}
Now we prove Theorem~\ref{thm:momentum_sqrt}. 
\begin{proof}[Proof of Theorem~\ref{thm:momentum_sqrt}]
With the fact that $\| \ba+\bb\|^2 \leq 2\| \ba\|^2 + 2\| \bb\|^2$, we have 
\begin{align*}
\sum_{t=1}^{T} \| \bta_t \bg_t \|^2 = 
\sum_{t=1}^{T} \eta_t^2 \| \bg_t \|^2
& \leq \sum_{t=1}^{T} 2\eta_t^2 \| \nabla f(\bx_t )\|^2+  \sum_{t=1}^{T} 2 \eta_t^2 \| \bg_t - \nabla f(\bx_t) \|^2\\
& \leq \sum_{t=1}^{T} 2\eta_t^2 \| \nabla f(\bx_t )\|^2+ \max_{1\leq t \leq T} \| \bg_t - \nabla f(\bx_t) \|^2\sum_{t=1}^{T} 2 \eta_t^2~.
\end{align*}
By Lemma~\ref{lemma:max_bound}, Lemma~\ref{lemma:momentum} and the union bound, we have that with probability at least $1-\delta$, 
\begin{align*}
\frac{\eta_T}{2} \sum_{t=1}^{T} \| \nabla f(\bx_t) \|^2 
& \leq \left(1- \frac{2M(3- \mu)}{1- \mu}\eta_1\right) \sum_{t=1}^{T} \eta_t \| \nabla f(\bx_t) \|^2 \\
& \leq 2 (f(\bx_1) - f^{\star}) +  \frac{2(3-\mu)c^2M\sigma^2 \ln \frac{2Te}{\delta} \ln T }{1- \mu}\\
& \quad  + \frac{3c(1- \mu^T)^2 \sigma^2}{(1- \mu)^2} \ln \frac{1}{\delta}~. 
\end{align*}
Rearranging the terms and lower bounding  $\sum_{t=1}^{T} \| \nabla f(\bx_t) \|^2$ by \\$T  \min_{1\leq t \leq T} \| \nabla f(\bx_t) \|^2$, we have the stated bound. 
\end{proof}

\subsection{Delayed AdaGrad with Momentum}
\label{sec:adagrad}
Now, we are going to prove the convergence rate of Delayed AdaGrad with momentum. Recall that the step sizes are defined as $\bta_t = (\eta_{t,j})_{j=1,\dots,d}$
\begin{equation}
\label{eq:bta}
\eta_{t,j}=\frac{\alpha}{\sqrt{\beta+ \sum_{i=1}^{t-1} g_{i,j}^2}}, \ j=1, \dots,d,
\end{equation}
where $\alpha, \beta >0$. 
Obviously, \eqref{eq:bta} satisfies Assumption~\ref{as:eta_non_increasing} and ~\ref{as:eta_independent}. Hence, we are able to apply Lemma~\ref{lemma:momentum} to analyze this variant. Moreover, for Delayed AdaGrad, we upper bound $ \sum_{t=1}^T \| \bta_t \bg_t \| ^2$ with the following lemma, whose proof is in the Appendix.
\begin{lemma}
\label{lemma:bound_1}
Assume $f$ is $M$-smooth and the stochasitc gradient is unbiased and satisfies Assumption~\ref{as:subgaussion}. Let $\bta_t$ set as in \eqref{eq:bta}, where $\alpha, \beta > 0$. Then, for any $\delta \in (0,1)$, with probability at least $1-\delta$, we have 
\begin{align*}
\sum_{t=1}^T \| \bta_t \bg_t \| ^2  
& \leq \frac{4d\alpha^2\sigma^2}{\beta} \ln \frac{2Te}{\delta} + \frac{4\alpha }{\sqrt{\beta}} \sum_{t=1}^{T}  \langle \bta_t, \nabla f(\bx_t)^2\rangle\\
& \leq 2\alpha^2d \ln \left( \sqrt{ \beta + \frac{2T\sigma^2 \ln \tfrac{2Te}{\delta}}{d} } + \sqrt{\frac{2}{d}\sum_{t=1}^T \|\nabla f(\bx_t) \|^2} \right)~.
\end{align*}
\end{lemma}

\begin{proof}[Proof of Lemma~\ref{lemma:bound_1}]
First, we separate $\sum_{t=1}^T \|\bta_t\bg_t\|^2 $ into two terms:
\[
\sum_{t=1}^T \|\bta_t\bg_t\|^2 
= \sum_{t=1}^T \|\bta_{t+1}\bg_t\|^2 + \sum_{t=1}^T \langle \bta_t^2 - \bta_{t+1}^2, \bg_t^2\rangle~.
\]
Then, we proceed 
\begin{align}
\sum_{t=1}^T \langle \bta_t^2 - \bta_{t+1}^2, \bg_t^2\rangle
& = \sum_{i=1}^{d}\sum_{t=1}^T (\eta_{t,i}^2 - \eta_{t+1, i}^2) g_{t,i}^2 \nonumber\\
& \leq \sum_{i=1}^{d}\sum_{t=1}^T2\eta_{t,i} g_{t,i}^2 (\eta_{t,i}-\eta_{t+1,i})\nonumber\\
& \leq  2\sum_{i=1}^{d} \max_{1\leq t \leq T} \eta_{t,i} g_{t,i}^2 \sum_{t=1}^T(\eta_{t,i}-\eta_{t+1,i})\nonumber\\
& \leq 2\sum_{i=1}^{d} \eta_{1,i} \max_{1\leq t \leq T} \eta_{t,i} g_{t,i}^2 \nonumber\\
& \leq 4\sum_{i=1}^{d} \eta_{1,i} \max_{1\leq t \leq T} \eta_{t,i}\left(g_{t,i}^2 - \nabla f(\bx_t)_i^2\right) + 4\sum_{i=1}^{d} \eta_{1,i}  \sum_{t=1}^{T}\eta_{t,i}\nabla f(\bx_t)_i^2\nonumber\\
& \leq 4\sum_{i=1}^{d} \eta_{1,i}^2 \max_{1\leq t \leq T} |g_{t,i}^2 - \nabla f(\bx_t)_i^2| + 4\sum_{i=1}^{d} \eta_{1,i}  \sum_{t=1}^{T}\eta_{t,i}\nabla f(\bx_t)_i^2\nonumber\\
& \leq \frac{4d \alpha^2}{\beta} \max_{1\leq t \leq T} \| \bg_t - \nabla f(\bx_t) \|^2 + \frac{4\alpha}{\sqrt{\beta}} \sum_{t=1}^{T} \langle \bta_t, \nabla f(\bx_t)^2\rangle~. \label{eq:uni_case}
\end{align}
Using Lemma~\ref{lemma:max_bound} on \eqref{eq:uni_case}, for $\delta \in (0,1)$, with probability at least $1- \frac{\delta}{2}$, we have 
\begin{equation}
\sum_{t=1}^T \langle \bta_t^2 - \bta_{t+1}^2, \bg_t^2\rangle
\leq \frac{4d \alpha^2 \sigma^2}{\beta} \ln \frac{2Te}{\delta} + \frac{4\alpha}{\sqrt{\beta}} \sum_{t=1}^{T} \langle \bta_t, \nabla f(\bx_t)^2\rangle~.
\end{equation}
We now upper bound $\sum_{t=1}^T \|\bta_{t+1}\bg_t\|^2$:
\begin{align}
\sum_{t=1}^T \|\bta_{t+1} \bg_t \|^2 
& = \sum_{i=1}^{d} \sum_{t=1}^T \frac{\alpha^2 g_{t,i}^2}{\beta + \sum_{j=1}^{t} g_{j,i}^2} \nonumber\\
& \leq \sum_{i=1}^{d} \alpha^2 \ln \left(\beta + \sum_{t=1}^{T} g_{t,i}^2 \right) \nonumber\\
& \leq \alpha^2d \ln \left(\beta +  \frac{1}{d}\sum_{i=1}^{d}\sum_{t=1}^{T} g_{t,i}^2 \right) \nonumber\\
& =  2\alpha^2d \ln \left( \sqrt {\beta +  \frac{1}{d} \sum_{t=1}^{T} \| \bg_t \|^2} \right) \nonumber\\
& \leq 2\alpha^2d \ln \left( \sqrt{ \beta + \frac{2T}{d} \max_{1 \leq t \leq T} \|\bg_t - \nabla f(\bx_t) \|^2} + \sqrt{ \frac{2}{d}\sum_{t=1}^T \|\nabla f(\bx_t) \|^2} \right), \label{eq:eps=0}
\end{align}
where in the first inequality we used Lemma~\ref{lemma:sum_integral_bounds} and in the second inequality we used Jensen's inequality. Then using Lemma~\ref{lemma:max_bound} on \eqref{eq:eps=0}, with probability at least $1- \frac{\delta}{2}$, we have 
\[
\sum_{t=1}^T \|\bta_{t+1} \bg_t \|^2 
\leq 2\alpha^2d \ln \left( \sqrt{ \beta + \frac{2T \sigma^2}{d} \ln \frac{2Te}{\delta}} + \sqrt{ \frac{2}{d}\sum_{t=1}^T \|\nabla f(\bx_t) \|^2} \right)~.
\]
Putting things together, we have the stated bound. 
\end{proof}

We now present the convergence guarantee for Delayed AdaGrad with momentum.
\begin{theorem}[Delayed AdaGrad with Momentum]
\label{thm:momentum_adagrad}
\begin{sloppypar}
Under the same assumptions in Lemma~\ref{lemma:bound_1}.  Let $\bta_t$ set as in \eqref{eq:bta}, where $\alpha, \beta > 0$ and $4 \alpha \leq  \frac{\sqrt{\beta}(1-\mu)^2}{2M(1+\mu)}$. Then, for any $\delta
\in (0,1)$, with probability at least $1- \delta$, the iterates of Algorithm \ref{alg:momentum} satisfy
\end{sloppypar}
\begin{align*}
\min_{1\leq t \leq T} \| \nabla f(\bx_t) \|^2 
\leq \frac{1}{T} \max \left( \frac{4 C(T)^2}{ \alpha^2}, \frac{C(T)}{\alpha} \sqrt{2\beta + 4T\sigma^2 \ln \frac{3Te}{\delta}} \right),
\end{align*}
where 
$C(T) = O\left(\frac{1}{\alpha} + \frac{d \left(\alpha+ \sigma^2 \left(\alpha \ln \frac{T}{\delta} + \frac{\ln \frac{1}{\delta}}{1- \mu}\right)\right)}{1- \mu}\right)$.
\end{theorem}
\paragraph{Adaptivity to noise} Observe that when $\sigma = 0$, the convergence rate recovers the rate of Gradient Descent if $O(\frac{1}{T})$ with a constant learning rate. On the other hand, in the noisy case, it matches the rate of SGD $O(\frac{\sigma}{\sqrt{T}})$ with the optimal worst-case learning rate of $O(\frac{1}{\sigma \sqrt{t}})$. In other words, with a unique learning rate, we recover two different optimal convergence rates that require two different learning rates and the knowledge of $\sigma$.
This adaptivity of Delayed AdaGrad was already proved in \citet{LiO19}, but only in expectation and without a momentum term.

\paragraph{Dependency on $\mu$} Observe that the convergence upper bound increases over $\mu \in (0, 1)$ and the optimal upper bound is achieved when taking the momentum parameter $\mu = 0$. In words, the algorithms without momentums have the best theoretical results. This is a known caveat for this kind of analysis and a similar behavior w.r.t. $\mu$ is present, e.g., in \citet[Theorem~1]{ZouSJSL18} for algorithms with Polyak's momentum.

\begin{proof}[Proof of Theorem~\ref{thm:momentum_adagrad}]
By Lemma~\ref{lemma:momentum} and Lemma~\ref{lemma:bound_1}, for $\delta \in (0,1)$, with probability at least $1 - \frac{2}{3} \delta$, we have 
\begin{align*}
& \left(1- \frac{4\alpha M (3- \mu)}{\sqrt{\beta} (1- \mu)}\right) \sum_{t=1}^{T} \langle \bta_t, \nabla f(\bx_t)^2 \rangle \\
& \leq 2(f(\bx_1) - f^{\star}) + \frac{M(3-\mu)}{1- \mu}  \left( K	+ \frac{4d\alpha^2\sigma^2}{\beta} \ln \frac{3Te}{\delta} \right) + \frac{3\| \bta_1 \| \sigma^2 (1- \mu^T)^2}{(1- \mu)^2} \ln \frac{3}{\delta}~. 
\end{align*}
where $K$ denotes $2\alpha^2d \ln \left( \sqrt{ \beta + \frac{2T\sigma^2}{d} \ln \frac{2Te}{\delta}} + \sqrt{\frac{2}{d}} \sqrt{\sum_{t=1}^T \|\nabla f(\bx_t) \|^2} \right)$ for conciseness.

Rearranging the terms, we have 
\begin{equation}
\label{eq:rhs_ready}
\begin{aligned}
	& \sum_{t=1}^{T} \langle \bta_t, \nabla f(\bx_t)^2\rangle \\
	& \leq \frac{1}{1- \frac{4\alpha M (3- \mu)}{\sqrt{\beta} (1- \mu)}} \left[2(f(\bx_1) - f^{\star})+ \frac{M(3-\mu)}{1- \mu} \left(K
	+ \frac{4d\alpha^2\sigma^2}{\beta} \ln \frac{3Te}{\delta} \right)\right. \\
	& \left. \quad +  \frac{3\| \bta_1 \| \sigma^2 (1- \mu^T)^2}{(1- \mu)^2} \ln \frac{3}{\delta} \right]\\
	& \leq 4(f(\bx_1) - f^{\star})+ \frac{2M(3-\mu)}{1- \mu} \left(K 
	+ \frac{4d\alpha^2\sigma^2}{\beta} \ln \frac{3Te}{\delta} \right)+ \frac{3\| \bta_1 \| \sigma^2 (1- \mu^T)^2}{(1- \mu)^2} \ln \frac{3}{\delta}\\
	& \triangleq C(T), 
\end{aligned}
\end{equation}
where in the second inequality we used $4 \alpha \leq \frac{ \sqrt{\beta}(1-\mu)}{2M(3- \mu)}$. 
Also, we have
\begin{align*}
\sum_{t=1}^{T} \langle \bta_t, \nabla f(\bx_t)^2\rangle
& \geq \sum_{t=1}^{T} \langle \bta_T, \nabla f(\bx_t)^2\rangle\\
& = \sum_{i=1}^d \frac{ \alpha \sum_{t=1}^T \nabla f(\bx_t)_i^2}{ \sqrt{ \beta + \sum_{t=1}^T g_{t,i}^2 } }\\
& \geq \sum_{i=1}^d \frac{ \alpha \sum_{t=1}^T \nabla f(\bx_t)_i^2}{ \sqrt{ \beta + 2 \sum_{t=1}^T \nabla f(\bx_t)_i^2  + 2\sum_{t=1}^T (g_{t,i}-\nabla f(\bx_t)_i)^2 }} \\
& \geq \sum_{i=1}^d \frac{ \alpha \sum_{t=1}^T \nabla f(\bx_t)_i^2}{\sqrt{ \beta + 2 \sum_{i=1}^d \sum_{t=1}^T \nabla f(\bx_t)_i^2  + 2\sum_{i=1}^d \sum_{t=1}^T (g_{t,i}-\nabla f(\bx_t)_i)^2 } } \\
&  \geq \frac{ \alpha \sum_{t=1}^T \| \nabla f(\bx_t) \|^2}{ \sqrt{\beta + 2 \sum_{t=1}^T \| \nabla f(\bx_t) \|^2 + 2 T\max_{1\leq t \leq T} \| \bg_t - \nabla f(\bx_t) \|^2}}~.
\end{align*}
By Lemma~\ref{lemma:max_bound}, with probability at least $1- \delta$, we have 
\begin{equation}
\label{eq:to_solve}
\sum_{t=1}^T \| \nabla f(\bx_t) \|^2  
\leq \frac{C(T)}{\alpha} \times \sqrt{ \beta + 2\sum_{t=1}^T \| \nabla f(\bx_t) \|^2 + 2T \sigma^2 \ln \frac{3Te}{\delta} }~. 
\end{equation}
\begin{align}
& \text{RHS of \eqref{eq:to_solve}} \nonumber \\
& \leq \frac{C(T)}{\alpha} \times \left( \sqrt{ \beta + 2T\sigma^2 \ln \frac{3Te}{\delta}} + \sqrt{2\sum_{t=1}^T \| \nabla f(\bx_t) \|^2}  \right) \nonumber\\
& \leq \left[ C+ D\ln \left(A + B\sqrt{ \sum_{t=1}^{T} \| \nabla f(\bx_t) \|^2 }\right) \right] \times \left(A + B \sqrt{ \sum_{t=1}^{T} \| \nabla f(\bx_t) \|^2 }\right), \label{eq:logx}
\end{align}
where $A = \sqrt{\beta + 2T \sigma^2 \ln \frac{3Te}{\delta}}$, $B = \sqrt{2}$, 
$C = \frac{4(f(\bx_1) - f^{\star})}{\alpha}
+ \frac{8M(3- \mu) d \alpha \sigma^2}{\beta(1- \mu)} \ln \frac{3Te}{\delta} $ \\$+ \frac{3d (1- \mu^T)^2\sigma^2}{\beta (1- \mu)^2} \ln \frac{3}{\delta}$ and $D = \frac{4 \alpha dM (3-\mu)}{1- \mu}$. Using Lemma~\ref{lemma:logsolvex}, we have that 
\begin{equation}
\sqrt{\sum_{t=1}^T \| \nabla f(\bx_t) \|^2 } \leq 32B^3D^2 + 2BC + 8B^2D\sqrt{C} + \frac{A}{B}~. 
\end{equation}
We use this upper bound in the logarithmic term of \eqref{eq:logx}. Thus, we have \eqref{eq:to_solve} again, this time with 
\begin{align*}
C(T) 
& = C+ D\ln (2A + 32B^4 D^2 + 2B^2C + 8B^3D\sqrt{C} )\\
& = O\left(\frac{1}{\alpha} + \frac{d \left(\alpha+ \sigma^2 \left(\alpha \ln \frac{T}{\delta} + \frac{\ln \frac{1}{\delta}}{1- \mu}\right)\right)}{1- \mu}\right)~.
\end{align*}
Solving \eqref{eq:logx} by Lemma~\ref{lemma:solvex} and lower bounding $\sum_{t=1}^T \| \nabla f(\bx_t) \|^2 $ by \\$T \min_{1\leq t \leq T} \| \nabla f(\bx_t) \|^2$, we get the stated bound. 
\end{proof}

%% file: 2_Adaptive/discussion.tex
\section{Conclusion}
This chapter provides convergence guarantees for SGD with Delayed AdaGrad stepsizes, with and without momentum over smooth and (non)convex functions. We believe these results have twofold importance. First, we go in the direction of closing the gap between theory and practice for widely used optimization algorithms. Second, our adaptive rates provide a possible explanation for the empirical success of these kinds of algorithms in practical machine learning applications. 

%% file: 3_CosExp/cos_exp.tex
\chapter{Exponential and Cosine Stepsize}
\label{chapter:exp_cos}
\thispagestyle{myheadings}

\input{3_CosExp/intro}
\input{3_CosExp/related}
\input{3_CosExp/def}

\input{3_CosExp/motivation}
\input{3_CosExp/thm}

\input{3_CosExp/disc}

%% file: 3_CosExp/intro.tex
\section{Introduction}
In this chapter, we look at the two simple to use and empirically successful step size decay strategies, the \textit{exponential} and the \emph{cosine step size} (with and without restarts)~\citep{LoshchilovH17, HeZZZXL19}.
The exponential step size is simply an exponential decaying step size.
It is less discussed in the optimization literature and it is also unclear who proposed it first, even if it has been known to practitioners for a long time and already included in many deep learning software libraries~\citep[e.g.,][]{Tensorflow15, PyTorch19}.
The cosine step size, which anneals the step size following a cosine function, has exhibited great power in practice but it does not have any theoretical justification.

We will use the following definition for the exponential step size
\begin{equation}
	\label{eq: step_size}
	\eta_t = \eta_0 \cdot \alpha^t
\end{equation}
and for cosine step sizes
\begin{equation}
	\label{eq:cosine_step}
	\eta_t = \frac{\eta_0 }{2}\left(1+ \cos \frac{t\pi}{T}\right),  
\end{equation}
where $\eta_0 > 0$ is the initial stepsize, $\alpha \in (0,1)$ is the decay rate.

For both these step size decay strategies, we prove \emph{for the first time} a convergence guarantee. Moreover, we show that they have (unsuspected!) adaptation properties.
Finally, our proofs reveal the hidden similarity between these two step sizes.

Specifically, we show that in the case when the function satisfies the PL condition~\citep{Polyak63,Lojasiewicz63,KarimiNS16}, both exponential step size and cosine step size strategies \textit{automatically adapt to the level of noise of the stochastic gradients}. 
Without the PL condition, we show that SGD with either exponential step sizes or cosine step sizes has an \textit{(almost) optimal convergence rate} for smooth non-convex functions. 

%% file: 3_CosExp/related.tex
\section{Related Work}
\label{sec:rel}

\paragraph{Exponential step size}
To the best of our knowledge, the exponential step size has been incorporated in Tensorflow~\citep{Tensorflow15} and PyTorch~\citep{PyTorch19}, yet no convergence guarantee has ever been proved for it.
The closest strategy is the \emph{stagewise step decay}, which corresponds to the discrete version of the exponential step size we analyze.
The stagewise step decay uses a piece-wise constant step size strategy, where the step size is cut by a factor in each stage. \citep{YuanYJY19, GeKKN19,DavisDXZ19,DavisDC19}. The stagewise step decay approach was first introduced in \citep{Goffin77} and used in many \emph{convex} optimization problem~\citep[e.g.,][]{HazanK11,AybatFGO19,KulunchakovM19,GeKKN19}. Interestingly, \citet{GeKKN19} also shows promising empirical results on non-convex functions, but instead of using their proposed decay strategy, they use an exponentially decaying schedule, like the one we analyze here.
The only use of the stagewise step decay for non-convex functions we know are for sharp functions~\citep{DavisDC19} and weakly-quasi-convex functions~\citep{YuanYJY19}. However, they do not show any adaptation property and they still do not consider the exponential step size but its discrete version. As far as we know, we prove the first theoretical guarantee for the exponential step size.

\paragraph{Cosine step decay}
Cosine step decay was originally presented in~\citet{LoshchilovH17} with two tunable parameters. Later, \citet{HeZZZXL19} proposed a simplified version of it with one parameter. However, there is no theory for this strategy though it is popularly used in the practical world \citep{LiuSY18, ZhangHZZXL19, LawenBPFZ19, ZhangWZZ19, GinsburyCHKLNZJ19, CubukZMVL19,ZhaoJK20,YouLHX20,ChenKNH20,GrillSATRBDABGGPKMV20}. 
As far as we know, we prove the first theoretical guarantee for the cosine step decay and the first ones to hypothesize and prove the adaptation properties of the cosine decay step size.


\paragraph{SGD with the PL condition}
The PL condition was proposed by \citet{Polyak63} and \citet{Lojasiewicz63}. It is the weakest assumption we know to prove linear rates on non-convex functions.
For SGD, \citet{KarimiNS16} proved the rate of $O\left(1 / \mu^2 T\right)$ for polynomial step sizes assuming Lipschitz and smooth functions, where $\mu$ is the PL constant. Note that the Lipschitz assumption hides the dependency of convergence and step sizes from the noise. It turns out that the Lipschitz assumption is not necessary to achieve the same rate. Considering functions with finite-sum structure, \citet{ReddiHSPS16}, \citet{LeiJCJ17} and \citet{LiBZR20} proved improved rates for variance reduction methods.
The convergence rate that we show for the exponential step size is new in the literature on the minimization of PL functions. 
Concurrently with our work, 
\citet{KhaledR20} obtained the same convergence result with the PL condition for SGD with a stepsize that is constant in the first half and then decreases polynomially.

%% file: 3_CosExp/def.tex
\section{Assumptions}
As in the previous chapter, we consider to minimize a smooth function $f$. 
Sometimes, we will also assume the function $f$ satisfies 
\begin{assumption}
\label{assump:pl}
$f$ satisfies the \emph{$\mu$-PL} condition, that is, for some $\mu > 0$, $\frac{1}{2} \| \nabla f(\bx) \|^2 \geq \mu \left( f(\bx) - f^{\star} \right), \ \forall \bx$.
\end{assumption}
In words, the gradient grows at least as the square root of the sub-optimality. 

We will make the following assumption on the variance of the noise.
\begin{assumptionA}
\label{assump:linear_variance}
For $t = 1, 2, \dots, T$, we assume $\E_t [\| \bg_t - \nabla f(\bx_t) \|^2 ] \leq a \| \nabla f(\bx_t) \|^2 + b $, where $a, b \geq 0$. 
\end{assumptionA}
This assumption on the noise is strictly weaker than the common assumption of assuming a bounded variance, i.e., {\footnotesize{$\E_t [\| \bg_t - \nabla f(\bx_t) \|^2 ] \leq \sigma^2$}}. Indeed, it recovers the bounded variance case with $a=0$ while also allowing for the variance to grow unboundedly far from the optimum when $a>0$. This is indeed the case when the optimal solution has low training error and the stochastic gradients are generated by mini-batches. This relaxed assumption on the noise was first used by \citet{BertsekasT96} in the analysis of the asymptotic convergence of SGD. 

It is worth stressing that non-convex functions are not characterized by a particular property, but rather from the lack of a specific property: convexity. In this sense, trying to carry out any meaningful analyses on the entire class of non-convex functions is hopeless. So, the assumptions we use balance the trade-off of \emph{approximately} model many interesting machine learning problems while allowing to restrict the class of non-convex functions on particular subsets where we can underline interesting behaviours.

More in detail, the smoothness assumption is considered ``weak'' and ubiquitous in analyses of optimization algorithms in the non-convex setting. In many neural networks, it is only approximately true because ReLUs activation functions are non-smooth. However, if the number of training points is large enough, it is a good approximation of the loss landscape.

On the other hand, the PL condition (Assumption~\ref{assump:pl}) is often considered a ``strong'' condition. However, it was formally proved to hold locally in deep neural networks in \citet{Allen-ZhuLS19}. Furthermore, \citet{KleinbergLY18} empirically observed that the loss surface of neural networks has good one-point convexity properties, and thus locally satisfies the PL condition. Of course, in our theorems we only need it to hold along the optimization path and not over the entire space, as also pointed out in \citet{KarimiNS16}. So, while being strong, it actually models the cases we are interested in.
Moreover, dictionary learning~\citep{AroraGMM15}, phase retrieval~\citep{ChenC15}, and matrix completion~\citep{SunL16}, all satisfy the one-point convexity locally~\citep{Zhu18b}, and in turn they all satisfy the PL condition locally.

%% file: 3_CosExp/motivation.tex
\section{Convergence and Adaptivity of Cosine and Exponential Step Sizes}
\label{sec:theorem}
Here, we present the guarantees of the exponential step size and the cosine step size and their adaptivity property.

\subsection{Noise and Step Sizes}
\label{sec:step}
For the stochastic optimization of smooth functions, the noise plays a crucial role in setting the optimal step sizes: \emph{To achieve the best performance, we need two completely different step size decay schemes in the noisy and noiseless case}. In particular, if the PL condition holds, in the noise-free case a constant step size is used to get a linear rate (i.e., exponential convergence), while in the noisy case the best rate $O(1/T)$ is given by time-varying step sizes $O(1/(\mu t))$~\citep{KarimiNS16}. Similarly, without the PL condition, we still need a constant step size in the noise-free case for the optimal rate whereas a $O(1/\sqrt{t})$ step size is required in the noisy case~\citep{GhadimiL13}.
Using a constant step size in noisy cases is of course possible, but the best guarantee we know is converging towards a neighborhood of the critical point or the optimum, instead of the exact convergence let alone the adaptivity to the noise, as shown in Theorem 2.1 of~\citep{GhadimiL13} and Theorem 4 of~\citep{KarimiNS16}.
Moreover, if the noise decreases over the course of the optimization, we should change the step size as well. Unfortunately, noise levels are rarely known or measured. On the other hand, an optimization algorithm \emph{adaptive to noise} would always get the best performance without changing its hyperparameters.

\emph{This means that for each noise level, namely mini-batch size in the finite-sum scenario, we need to tune a different step size decay to obtain the best performance.} This process is notoriously tedious and time-consuming.

Another choice is the stagewise step decay.
For example, \citet{GeKKN19} propose to start from a constant step size and cut it by a fixed factor every $O(\ln T)$ steps, decaying roughly to $O\left(1/T\right)$ after $T$ iterations.
However, in practice, deciding when to cut the step size becomes a series of hyperparameters to tune, making this strategy difficult to use in real-world applications. 

We show that the above problems can be solved by using the exponential step sizes
In the following, we will show that exponential and cosine step sizes achieve exactly this adaptation to noise. It is worth reminding the reader that \emph{any} polynomial decay of the step size does not give us this adaptation. 
So, let's gain some intuition on why this should happen with these two step sizes.
In the early stage of the optimization process, we can expect that the disturbance due to the noise is relatively small compared to how far we are from the optimal solution. Accordingly, at this phase, a near-constant step size should be used. More precisely, the proofs shows that to achieve a linear rate we need $\sum_{t=1}^T \eta_t = \Omega(T)$ or even $\sum_{t=1}^T \eta_t = \Omega(T/ \ln T)$. This is exactly what happens with \eqref{eq: step_size} and \eqref{eq:cosine_step}. On the other hand, when the iterate is close to the optimal solution, we have to decrease the step size to fight with the effects of the noise. In this stage, the exponential step size goes to 0 as $O \left(1/T \right)$, which is the optimal step size used in the noisy case. 
Meanwhile, the last $i$th cosine step size is $\eta_{T-i} = \frac{\eta_0}{2}(1- \cos\frac{i \pi }{T})= \eta_0 \sin^2 \frac{i\pi}{2T}$, which amounts $O (1/T^2)$ when $i$ is much smaller than $T$.

Hence, the analysis shows that \eqref{eq: step_size} and \eqref{eq:cosine_step} are surprisingly similar, smoothly varying from the near-constant behavior at the start and decreasing with a similar pattern towards the end, and both will be adaptive to the noise level.
Particularly, the exponential step size is emulating the transition between the \emph{optimal} constant one at the beginning and \emph{optimal} decreasing one towards the end in a smooth continuous way.
Next, we formalize these intuitions in convergence rates.

%% file: 3_CosExp/thm.tex
\subsection{Convergence Guarantees}
\label{ssec:guarantee}
In this section, we will prove the convergence guarantees for these two step sizes.  We will first consider the case where the function is smooth and satisfies the PL condition. Before we introduce the convergence rates for these two,  let's take a look at what are the known rates under the same condition.

\citet{KarimiNS16} proved that SGD with an appropriate step size will give a $O(1/T)$ convergence for Lipschitz and PL functions. However, it is easy to see that the Lipschitz assumption can be substituted by the smoothness one and obtain a rate that depends on the variance of the noise. Even if this is a straightforward result, we could not find it anywhere so we report here our proof.

\begin{theorem}
\label{thm:pl_smooth_poly_step}
Assume $f$ is $L$-smooth and satisfies $\mu$-PL condition. Set the step sizes as $\eta_t = \min \left(\frac{1}{L(1+a)}, \frac{2t+1}{\mu (t+1)^2}\right)$. Then, SGD guarantees
\begin{align*}
f(\bx_{T+1}) - f^{\star}
\leq \frac{L^2 (1+a)b}{2\mu^3 T^2}+ \frac{2L}{\mu^2 T}b
+ (f(\bx_1) - f^{\star})\frac{L^2(1+a)^2}{\mu^2 T^2} \left(1-\frac{\mu}{L(1+a)}\right)^{\frac{L(1+a)}{\mu}}~.
\end{align*}
\end{theorem}
\begin{proof}
For simplicity, denote $\E f(\bx_t) - f^{\star}$ by $\Delta_t$. With the same analysis as in Theorem~\ref{thm: pl_smooth_cst_noise}, we have
\[
	\Delta_{t+1} \leq \left(1- \mu \eta_t\right) \Delta_t + \frac{L}{2} \eta_t^2 b~.
	\]
	Denote by $t^{\star} = \min \left\{t: \frac{t^2}{2t+1} \leq \frac{L(1+a) - \mu}{\mu}\right\}$. When $t \leq t^{\star}$, $\eta_t = \frac{1}{L(1+a)}$ and we obtain
	\[
	\Delta_{t+1} \leq \left(1-\frac{\mu}{L(1+a)}\right) \Delta_t + \frac{b}{2L(1+a)^2}~.
	\]
	Thus, by Lemma \ref{lemma: ratio_bound}, we get
	\begin{align*}
		\Delta_{t^{\star}}
		& \leq \left(1-\frac{\mu}{L(1+a)}\right)^{t^{\star}-1} \Delta_1
		+ \frac{b}{2L(1+a)^2}\sum_{i=0}^{t^{\star}}\left(1-\frac{\mu}{L(1+a)}\right)^{t^{\star}-i} \\
		& \leq \left(1-\frac{\mu}{L(1+a)}\right)^{t^{\star}} \Delta_1 + \frac{b}{2\mu(1+a)}~.
	\end{align*}
	Instead, when $t \geq t^{\star}$, $\eta_t = \frac{2t+1}{\mu (t+1)^2}$, we have
	\[
	\Delta_{t+1} \leq \frac{t^2}{(t+1)^2} \Delta_t + \frac{L (2t+1)^2}{2 \mu^2 (t+1)^4 }b~.
	\]
	Multiplying both sides by $(t+1)^2$ and denoting by $\delta_t = t^2 \Delta_t$, we get
	\[
	\delta_{t+1} \leq \delta_t + \frac{L(2t+1)^2}{2\mu^2(t+1)^2}b
	\leq \delta_t + \frac{2L}{\mu^2}b~.
	\]
	Summing over $t$ from $t^{\star}$ to $T$, we have
	\[
	\delta_{T+1} \leq \delta_{t^{\star}\\
	} + \frac{2L(T-t^{\star})}{\mu^2} b~.
	\]
	Then, we finally get
	\begin{align*}
		\Delta_{T+1}
		& \leq \frac{t^{\star2}}{T^2} \left(1-\frac{\mu}{L(1+a)}\right)^{t^{\star}}\Delta_1
		+ \frac{t^{\star2}b}{2\mu(1+a)T^2}+ \frac{2L(T-t^{\star})}{\mu^2 T^2}b \\
		& \leq \frac{L^2(1+a)^2}{\mu^2 T^2} \left(1-\frac{\mu}{L(1+a)}\right)^{\frac{L(1+a)}{\mu}}\Delta_1
		+ \frac{L^2 (1+a)b}{2\mu^3 T^2}+ \frac{2L}{\mu^2 T}b~. \qedhere
	\end{align*}
\end{proof}

Now, we prove the convergence rates for these two stepsizes. 
\begin{theorem}[SGD with exponential step size]
\label{thm: pl_smooth_cst_noise}
Assume $f$ is $L$-smooth and satisfies $\mu$-PL condition. Suppose that the variance of the noise on stochastic gradients satisfies Assumption~\ref{assump:linear_variance}.  For a given $T \geq \max\{3, \beta \}$ and $\eta_0 = (L(1+a))^{-1}$, with step size \eqref{eq: step_size}, SGD guarantees 
\begin{align*}
& \E f(\bx_{T+1}) - f^{\star} \leq \frac{5LC(\beta)}{e^2 \mu^2 } \frac{\ln^2 \frac{T}{\beta}}{T} b+ C(\beta) \exp\left(-\frac{0.69\mu }{L+a} \left(\frac{T}{\ln \frac{ T}{\beta}}\right)\right)\cdot (f(\bx_1) - f^{\star}), 
\end{align*}
where $C(\beta)\triangleq \exp \left((2\mu\beta)/(L (1+a)\ln T/\beta)\right)$.
\end{theorem}

\textbf{Choice of $\beta$} 
Note that if $\beta = L(1+a)/\mu$, we get 
\begin{align*}
\E f(\bx_{T+1}) - f^{\star} \leq O\left(\exp\left(-\frac{\mu }{L+a} \left(\frac{T}{\ln \frac{\mu T}{L}}\right)\right)+ \frac{b \ln^2 \frac{\mu T}{L}}{\mu^2  T} \right)~.
\end{align*}
In words, this means that we are basically free to choose $\beta$, but will pay an exponential factor in the mismatch between $\beta$ and $\frac{L}{\mu}$, which is basically the condition number for PL functions. This has to be expected because it also happens in the easier case of stochastic optimization of strongly convex functions~\citep{MoulinesB11}.
\begin{theorem}[SGD with cosine step size]
\label{thm:PL_cosine}
Assume $f$ is $L$-smooth and satisfies $\mu$-PL condition. Suppose that the variance of the noise on stochastic gradients satisfies Assumption~\ref{assump:linear_variance}.  For a given $T$ and $\eta_0 = (L(1+a))^{-1}$, with step size \eqref{eq:cosine_step}, SGD guarantees 
\begin{align*}
\E f(\bx_{t+1}) - f^{\star} 
& \leq \exp \left(- \frac{\mu (T-1)}{2L(1+a)}\right) (f(x_1) - f^{\star})  \\
& \quad + \frac{ \pi^4 b}{32 (1+a)T^4} \left( \left(\frac{8T^2}{\mu}\right)^{4/3} + \left(\frac{6T^2}{\mu}\right)^{\frac{5}{3}}\right) ~.
\end{align*}
\end{theorem}

The original cosine stepsize was proposed with a restarting strategy,
yet it has been commonly used without restarting and achieves good results~\citep[e.g.,][]{LoshchilovH17,Gastaldi17,ZophVSL18,HeZZZXL19,CubukZMVL19,LiuSY18,ZhaoJK20,YouLHX20,ChenKNH20,GrillSATRBDABGGPKMV20}. Indeed, the previous theorem has confirmed that the cosine stepsize alone is well worth studying theoretically. Yet for completeness, 
we also cover the analysis in a restart scheme for SGD with cosine stepsize in the PL condition as following. We obtain the same convergence rate $\mu $ and $T$ as that in the case of no restarts under the PL condition. 
\begin{algorithm}
\caption{SGD with Cosine Stepsize and Restarts}
\label{alg:restart}
\begin{algorithmic}
\STATE \textbf{Input:} Initial Step size $\eta_0$, time increase factor $r$, initial point $\bx_1$. 
\FOR{$i = 0, \dots, l$}
\STATE   Let $T_i = T_0 \, r^i $
\FOR{$t = 0, \dots, T_i -1$}
\STATE  Run SGD with cosine stepsize  $\frac{\eta_0}{2} \left(1+ \cos \frac{t \pi}{T_i}\right)$
\ENDFOR
\ENDFOR
\end{algorithmic}
\end{algorithm}
\begin{theorem}[SGD with cosine step size and restart]
\label{thm:PL_cosine_restart}
Under the same assumptions in Theorem~\ref{thm:PL_cosine}. For a given $T_0 $, $r >1$ , $T_i = T_0 \, r^i $, $T \triangleq \sum_{i=0}^{l} T_i$, and $\eta_0 = (L(1+a))^{-1}$, Algorithm~\ref{alg:restart} guarantees (where $\tilde{O}$ hides the $log$ terms) 
\begin{align*}
\E f(\bx_T) - f^{\star} 
& \leq 
\tilde{O} \left(\exp \left(- \frac{\mu (T - l - 1)}{2L(1+a)} \right) 
+ b\left( \frac{1}{\mu^{4/3} T^{4/3}} + \frac{1}{\mu^{5/3} T^{2/3}}\right) \right)~.
\end{align*}
\end{theorem}

\paragraph{Adaptivity to noise}
From the above theorems, we can see that both the exponential step size and the cosine step size have a provable advantage over polynomial ones: \emph{adaptivity to the noise}. Indeed, when $b=0$, namely there is only noise relative to the distance from the optimum, they both guarantee a linear rate. Meanwhile, if there is noise, using the \emph{same step size without any tuning}, the exponential step size recovers the rate of $O\left(1/(\mu^2 T)\right)$ while the cosine step size achieves the rate of $O(1/(\mu^{\frac{5}{3}}T^{\frac{2}{3}}))$ (up to poly-logarithmic terms). In contrast, polynomial step sizes would require two different settings---decaying vs constant---in the noisy vs no-noise situation~\citep{KarimiNS16}.
It is worth stressing that the rate in Theorem~\ref{thm: pl_smooth_cst_noise} is one of the first results in the literature on stochastic optimization of smooth PL functions \citep{KhaledR20}.
\paragraph{Optimality of the bounds}
As far as we know, it is unknown if the rate we obtain for the optimization of non-convex smooth functions under the PL condition is optimal or not. However, up to poly-logarithmic terms, Theorem~\ref{thm: pl_smooth_cst_noise} matches at the same time the best-known rates for the noisy and deterministic cases~\citep{KarimiNS16}. We would remind the reader that this rate is not comparable with the one for strongly convex functions which is $O(1/(\mu T))$.
Meanwhile, cosine step size achieves a rate slightly worse in $T$ (but better in $\mu$) under the same assumptions.

\paragraph{Convergence without the PL condition}
The PL condition tells us that all stationary points are optimal points~\citep{KarimiNS16}, which is not always true for the parameter space in deep learning~\citep{JinGNKJ17}. However, this condition might still hold locally, for a considerable area around the local minimum. Indeed, as we said, this is exactly what was proven for deep neural networks~\citep{Allen-ZhuLS19}.
The previous theorems tell us that once we reach the area where the geometry of the objective function satisfies the PL condition, we can get to the optimal point with an almost linear rate, depending on the noise. Nevertheless, we still have to be able to reach that region.
Hence, in the following, we discuss the case where the PL condition is not satisfied and show for both step sizes that they are still able to move to a critical point at the optimal speed.
\begin{theorem}
\label{thm:no_PL_no_noise}
Assume $f$ is $L$-smooth and the variance of the noise on stochastic gradients satisfies Assumption~\ref{assump:linear_variance} and $c > 1$. SGD with step sizes \eqref{eq: step_size} with $\eta_0 = (c L(1+a))^{-1}$ guarantees
\begin{align*}
\E  \| \nabla f(\tilde{\bx}_T) \|^2
\leq \frac{3 L c (a+1)\ln \frac{T}{\beta}}{T- \beta} \cdot \left(f(\bx_1) - f^{\star}\right) + \frac{b T}{c(a+1)(T-\beta)}, 
\end{align*}
where $\tilde{\bx}_T$ is a random iterate drawn from $\bx_1, \dots, \bx_T$ with $\Pr[\tilde{\bx}_T=\bx_t]= \frac{\eta_t}{\sum_{i=1}^T \eta_i }$.
\end{theorem}
\begin{theorem}
\label{thm:no_PL_cosine}
Assume $f$ is $L$-smooth and the variance of the noise on stochastic gradients satisfies Assumption~\ref{assump:linear_variance} and $c > 1$. SGD with step sizes \eqref{eq:cosine_step} with $\eta_0 = (c L(1+a))^{-1}$ guarantees
\begin{align*}
\E  \| \nabla f(\tilde{\bx}_T) \|^2
\leq  \frac{4 L c (a+1)}{T- 1} \cdot \left(f(\bx_1) - f^{\star}\right)+ \frac{21 b T}{4 \pi^4 c L (a+1)(T-1)}, 
\end{align*}
where $\tilde{\bx}_T$ is a random iterate drawn from $\bx_1, \dots, \bx_T$ with $\Pr[\tilde{\bx}_T=\bx_t]= \frac{\eta_t}{\sum_{i=1}^T \eta_i }$.
\end{theorem}
If $b\neq 0$ in Assumption~\ref{assump:linear_variance}, setting $c \propto \sqrt{T}$ and $\beta=O(1)$ would give the $\tilde{O}(1/\sqrt{T})$ rate\footnote{The $\tilde{O}$ notations hides poly-logarithmic terms.} and $O(1/\sqrt{T})$ for the exponential and cosine step size respectively. Note that the optimal rate in this setting is $O(1/\sqrt{T})$. On the other hand, if $b=0$, setting $c=O(1)$ and $\beta=O(1)$ yields a $\tilde{O}(1/T)$ rate and $O(1/T)$ for the exponential and cosine step size respectively. It is worth noting that the condition $b=0$ holds in many practical scenarios~\citep{VaswaniBS19}.
Note that both guarantees are optimal up to poly-logarithmic terms~\citep{ArjevaniCDFSW19}.

In the following, we present the proofs of these theorems.  The proofs also show the mathematical similarities between these two step sizes.

We first introduce some technical lemmas.
\begin{lemma}
\label{lemma:start}
Assume $f$ is $L$-smooth and satisfies $\mu$-PL condition,  and $\eta_t \leq \frac{1}{L(1+a)}$.  SGD guarantees 
\begin{equation}
\label{eq:thm2_eq1}
\begin{split}
\E f(\bx_{t+1})- \E  f(\bx_t)
\leq -  \frac{\eta_t}{2} \E \| \nabla f(\bx_t) \|^2 + \frac{L \eta_t^2 b}{2}~. 
\end{split}
\end{equation}
\end{lemma}
\begin{proof}[Proof of Lemma~\ref{lemma:start}]
By the property of smooth functions, we have
\begin{equation}
\label{eq:smooth_one_step}
f(\bx_{t+1}) \leq f(\bx_t) - \langle \nabla f(\bx_t), \eta_t\bg_t \rangle + \frac{L}{2} \eta_t^2 \| \bg_t \|^2~.
\end{equation}
Taking expectation on both sides, we get
\begin{align*}
\E f(\bx_{t+1})- \E f(\bx_t)
& \leq - \left(\eta_t - \frac{L(a+1)}{2} \eta_t^2 \right)\E \| \nabla f(\bx_t) \|^2 + \frac{L}{2}\eta_t^2 b\\ 
& \leq - \frac{1}{2}\eta_t \E \| \nabla f(\bx_t) \|^2 + \frac{L}{2}\eta_t^2 b,
\end{align*}
where in the last inequality we used the fact that $\eta_t \leq \frac{1}{L(1+a)}$.
\end{proof}

\begin{lemma}
\label{lemma: ratio_bound}
Assume $X_k, A_k, B_k \geq 0, k = 1 ,...$, and $X_{k+1} \leq A_k X_k + B_k$, then we have 
\[
X_{k+1} \leq \prod_{i=1}^k A_i X_1 + \sum_{i=1}^{k} \prod_{j=i+1}^k A_j B_i~. 
\]
\end{lemma}
\begin{proof}[Proof of Lemma~\ref{lemma: ratio_bound}]
When $k=1$, $X_2 \leq A_1 X_1 + B_1$ satisfies. By induction, assume $X_{k} \leq \prod_{i=1}^{k-1} A_i X_1 + \sum_{i=1}^{k-1}\prod_{j=i+1}^{k-1} A_j B_i$, and we have
\begin{align*}
X_{k+1}
&\leq A_k \left( \prod_{i=1}^{k-1} A_i X_1 + \sum_{i=1}^{k-1} \prod_{j=i+1}^{k-1} A_j B_i \right) + B_k 
= \prod_{i=1}^k A_i X_1 + \sum_{i=1}^{k-1} \prod_{j=i+1}^{k} A_j B_i + A_k B_k \\
&= \prod_{i=1}^k A_i X_1 + \sum_{i=1}^k \prod_{j=i+1}^k A_j B_i~. \qedhere
\end{align*}
\end{proof}
\begin{lemma}
\label{lemma:sum_cosine}
For $\forall T \geq 1$, we have $\sum_{t=1}^{T} \cos\frac{t \pi}{T} = -1$.
\end{lemma}
\begin{proof}[Proof of Lemma~\ref{lemma:sum_cosine}]
If $T$ is odd, we have
\begin{align*}
\sum_{t=1}^{T} \cos\frac{t \pi}{T}
= \cos \frac{T\pi }{T} + \sum_{t=1}^{(T-1)/2} \cos \frac{t \pi}{T} + \cos \frac{(T-t)\pi}{T} 
= \cos \pi = -1,
\end{align*}
where in the second inequality we used the fact that $\cos (\pi - x) = - \cos (x)$ for any $x$.
If $T$ is even, we have
\[
\sum_{t=1}^{T} \cos\frac{t \pi}{T}
= \cos \frac{T\pi }{T} + \cos \frac{T\pi }{2T} + \sum_{t=1}^{T/2 - 1} \cos \frac{t \pi}{T} + \cos \frac{(T-t)\pi}{T} 
= \cos \pi = -1~. \qedhere
\]
\end{proof}
\begin{lemma}
\label{lemma: ineq_constant}
For $T \geq 3$, $\alpha \geq 0.69$ and $\frac{ \alpha^{T+1}}{(1-\alpha)} \leq \frac{2\beta}{\ln \frac{T}{\beta}}$. 
\end{lemma}
\begin{proof}[Proof of Lemma~\ref{lemma: ineq_constant}]
We have
\begin{align*}
\frac{\alpha^{T+1}}{(1-\alpha)}
= \frac{\alpha \beta }{T (1-\alpha)}
= \frac{\beta }{T\left(1 - \exp\left(-\frac{1}{T} \ln \frac{T}{\beta}\right)\right)}
\leq \frac{2\beta }{\ln \frac{T}{\beta}},
\end{align*}
where in the last inequality we used $\exp(-x) \leq 1- \frac{x}{2}$ for $0 < x < \frac{1}{e}$ and the fact that $\frac{1}{T} \ln\left(\frac{T}{\beta}\right) \leq \frac{\ln T}{T} \leq \frac{1}{e}$.
\end{proof}
\begin{lemma}
\label{lemma: ineq_alpha}
$
1-x \leq \ln \left(\frac{1}{x}\right), \forall x > 0. 
$
\end{lemma}
\begin{proof}[Proof of Lemma~\ref{lemma: ineq_alpha}]
It is enough to prove that $f(x) := x - 1- \ln x \geq 0$. Observe that $f'(x)$ is increasing and $f'(1) = 0$, hence, we have $f(x) \geq f(1) = 0$.
\end{proof}
\begin{lemma}
\label{lemma: integral_bound}
Let $a,b\geq0$. Then 
\[
\sum_{t=0}^T \exp(-b t) t^a \leq 2\exp(-a)\left(\frac{a}{b}\right)^a+ \frac{\Gamma(a+1)}{b^{a+1}}~.
\]
\end{lemma}
\begin{proof}[Proof of Lemma~\ref{lemma: integral_bound}]
Note that $f(t)=\exp(-b t) t^a$ is increasing for $t\in [0,a/b]$ and decreasing for $t\geq a/b$. Hence, we have
\begin{align*}
\sum_{t=0}^T \exp(-b t) t^a
&\leq \sum_{t=0}^{\lfloor a/b \rfloor-1} \exp(-b t) t^a + \exp(-b \lfloor a/b \rfloor) \lfloor a/b \rfloor^a + \exp(-b \lceil a/b \rceil) \lceil a/b \rceil^a \\
& \quad +\sum_{\lceil a/b \rceil+1}^T \exp(-b t) t^a \\
&\leq 2\exp(-a)(a/b)^a+\int_{0}^{\lfloor a/b \rfloor} \exp(-b t) t^a dt + \int_{\lceil a/b \rceil}^T \exp(-b t) t^a dt\\
&\leq 2\exp(-a)(a/b)^a+\int_{0}^{T} \exp(-b t) t^a dt \\
&\leq 2\exp(-a)(a/b)^a+\int_{0}^{\infty} \exp(-b t) t^a dt \\
&=2\exp(-a)(a/b)^a+ \frac{1}{b^{a+1}} \Gamma(a+1)~. \qedhere
\end{align*}
\end{proof}

\begin{proof}[Proof of Theorem \ref{thm:no_PL_no_noise} and Theorem \ref{thm:no_PL_cosine}]
We observe that for exponential step sizes,
\begin{align*}
\sum_{t=1}^T \eta^2_t
\leq \frac{\alpha^2}{L^2 c^2 (a+1)^2(1- \alpha^2)}~.
\end{align*}
and for cosine step sizes,
\begin{align*}
\sum_{t=1}^{T} \eta_t^2
& = \frac{\eta_0^2}{4} \sum_{t=1}^{T} \left(1+ \cos\frac{t\pi}{T}\right)^2 = \frac{\eta_0^2}{4} \sum_{t=0}^{T-1} \left(1- \cos\frac{t\pi}{T}\right)^2 
= \eta_0^2 \sum_{t=1}^{T} \sin^4\frac{t\pi}{2T}\\
&  \leq \eta_0^2 \sum_{t=1}^{T} \frac{t^4\pi^4}{16T^4}
\leq \frac{21 \eta_0^2 T}{8 \pi^4}~.
\end{align*}
Summing \eqref{eq:thm2_eq1} over $t=1,\dots,T$ and dividing both sides by $\sum_{t=1}^T \eta_t$, we get the stated bound.
\end{proof}

We can now prove both Theorem~\ref{thm: pl_smooth_cst_noise} and Theorem~\ref{thm:PL_cosine}.
\begin{proof}[Proof of Theorem~\ref{thm: pl_smooth_cst_noise} and Theorem~\ref{thm:PL_cosine}.]
Denote $\E f(\bx_t) - f^{\star}$ by $\Delta_t$. From Lemma~\ref*{lemma:start} and the PL condition, we get
\begin{equation}
\Delta_{t+1} 
\leq  (1 - \mu \eta_t ) \Delta_t + \frac{L}{2} \eta_t^2 b^2~. 
\end{equation} 
By Lemma~\ref{lemma: ratio_bound} and $1- x \leq \exp (-x)$, we have 
\begin{align}
& \Delta_{T+1} 
\leq \prod_{t=1}^{T} (1- \mu \eta_t) \Delta_1 +  \frac{L}{2} \sum_{t=1}^{T} \prod_{i=t+1}^{T} (1- \mu \eta_i) \eta_t^2 b\\
&  \leq \exp \left(- \mu \sum_{t=1}^{T} \eta_t \right) \Delta_1  +  \frac{Lb}{2} \sum_{t=1}^{T} \exp \left( - \mu \sum_{i=t+1}^{T} \eta_i \right) \eta_t^2. 
\end{align}
We then show that both the exponential step size and the cosine step size satisfy $\sum_{t=1}^{T} \eta_t = \Omega (T)$, which guarantees a linear rate in the noiseless case.

For the cosine step size \eqref{eq:cosine_step}, we observe that 
\begin{align*}
\sum_{t=1}^{T} \eta_t
= \frac{\eta_0 T}{2} + \frac{\eta_0}{2} \sum_{t=1}^{T} \cos \frac{t\pi}{T} = \frac{\eta_0 (T-1)}{2},
\end{align*}
where in the last equality we used Lemma~\ref{lemma:sum_cosine}. 

Also, for the exponential step size \eqref{eq: step_size}, we can show
\begin{align*}
\sum_{t=1}^{T} \eta_t = \eta_0 \frac{\alpha - \alpha^{T+1}}{1- \alpha}
& \geq  \frac{\eta_0 \alpha}{1- \alpha} - \frac{2\eta_0\beta }{\ln \frac{T}{\beta}}\\
& \geq  T \cdot  \frac{0.69\eta_0}{\ln \frac{T}{\beta}} -  \frac{2\eta_0\beta }{\ln \frac{T}{\beta}}, 
\end{align*}
where we used Lemma~\ref{lemma: ineq_constant} in the first inequality and Lemma~\ref{lemma: ineq_alpha} in the second inequality. 

Next, we upper bound $\sum_{t=1}^{T} \exp \left( - \mu \sum_{i=t+1}^{T} \eta_i \right) \eta_t^2$ for these two kinds of step sizes respectively.

For the exponential step size, by Lemma~\ref{lemma: ineq_constant}, we obtain
\begin{align*}
& \sum_{t=1}^{T} \exp \left( - \mu \sum_{i=t+1}^{T} \eta_i \right) \eta_t^2 \\
& = \eta_0^2 \sum_{t=1}^{T} \exp \left( - \mu \eta_0 \frac{\alpha^{t+1} - \alpha^{T+1}}{1- \alpha}\right) \alpha^{2t}\\
& \leq \eta_0^2 C(\beta) \sum_{t=1}^{T} \exp \left( - \frac{ \mu \eta_0\alpha^{t+1}}{1- \alpha}\right) \alpha^{2t}\\
& \leq \eta_0^2 C(\beta)  \sum_{t=1}^{T} \left(\frac{e}{2} \frac{\mu \alpha^{t+1}}{L(1+a)(1-\alpha)}\right)^{-2} \alpha^{2t}\\
& \leq \frac{4L^2(1+a)^2}{e^2\mu^2} \sum_{t=1}^T\frac{1}{\alpha^2} \ln^2 \left(\frac{1}{\alpha}\right)
\leq  \frac{10 L^2(1+a)^2\ln^2 \frac{T}{\beta}}{e^2 \mu^2 T}, 
\end{align*}
where in the second inequality we used $\exp(-x) \leq \left(\frac{\gamma}{e x}\right)^\gamma, \forall x >0, \gamma>0$. 

For the cosine step size, using the fact that $\sin x \geq \frac{2}{\pi}x$ for $ 0 \leq x \leq \frac{\pi}{2}$, we can lower bound $\sum_{i=t+1}^{T} \eta_i $ by
\begin{align*}
\sum_{i=t+1}^{T} \eta_i 
&= \frac{\eta_0 }{2}\sum_{i=t+1}^{T}  \left(1+ \cos \frac{i \pi}{T}\right) \\
& = \frac{\eta_0 }{2}\sum_{i=0}^{T-t-1}  \sin^2 \frac{i \pi}{2T} 
\geq \frac{\eta_0}{2T^2}\sum_{i=0}^{T-t-1}  i^2 \\
&\geq \frac{\eta_0(T-t-1)^3}{6T^2}~. 
\end{align*}
Then, we proceed 
\begin{align*}
& \sum_{t=1}^{T} \exp \left(- \mu \sum_{i=t+1}^{T} \eta_i\right) \eta_t^2 \\
& \leq \frac{\eta_0^2}{4} \sum_{t=1}^{T} \left(1+ \cos \frac{t\pi}{T}\right)^2 \exp \left(- \frac{\mu \eta_0(T-t-1)^3}{6T^2}\right)\\
& = \frac{ \eta_0^2}{4} \sum_{t=1}^{T-1} \left(1-  \cos \frac{t\pi}{T}\right)^2 \exp \left(- \frac{\eta_0 \mu (t-1)^3}{6T^2}\right) \\
& =\eta_0^2\sum_{t=1}^{T-1} \sin^4 \frac{t\pi}{2T} \exp \left(- \frac{\eta_0 \mu (t-1)^3}{6T^2}\right) \\
& \leq \frac{\eta_0^2 \pi^4}{16T^4}\sum_{t=0}^{T-1}t^4  \exp \left(- \frac{\eta_0 \mu t^3}{6T^2}\right)\\
& \leq \frac{\eta_0 \pi^4}{16 T^4} \left(2 \exp\left(-\frac{4}{3}\right) \left(\frac{8T^2}{\mu}\right)^{4/3} + \left(\frac{6T^2}{\mu}\right)^{\frac{5}{3}}\right), 
\end{align*}
where in the third line we used $\cos(\pi-x) = - \cos(x)$, in the forth line we used $1- \cos(2x) = 2\sin^2(x)$, and in the last inequality we applied Lemma~\ref{lemma: integral_bound}. 

Putting things together, we get the stated bounds. 
\end{proof}

\begin{proof}[Proof of Theorem~\ref{thm:PL_cosine_restart}] Denote by $S_i =\sum_{j=0}^{i}T_j$ and $S_{-1} = 1$. Given Theorem~\ref{thm:PL_cosine},  it is immediate to have $\forall i = 0, \dots, l$:
\begin{align*}
\E f(\bx_{S_i}) - f^{\star} 
& \leq \frac{ \pi^4 b}{32 (1+a)T_i^4} \left( \left(\frac{8T_i^2}{\mu}\right)^{4/3} + \left(\frac{6T_i^2}{\mu}\right)^{\frac{5}{3}}\right) \\
& \quad + \exp \left(- \frac{\mu (T_i-1)}{2L(1+a)}\right) (f(x_{S_{i-1}})- f^{\star})\\
& \leq C_1 \left(\frac{1}{\mu^{4/3} T_i^{4/3}} + \frac{1}{\mu^{5/3} T_i^{2/3}}\right) + \exp \left(-  C_2 \mu (T_i-1)\right) (f(x_{S_{i-1}}) - f^{\star})~.
\end{align*}

Repeatedly using the above inequality, we get 
\begin{align*}
\E f(\bx_{S_l}) - f^{\star}  
& \leq  C_1 \sum_{i=0}^{l}\prod_{j= i+1}^{l} \exp(- C_2 \mu (T_j - 1))   \left(\frac{1}{\mu^{4/3} T_i^{4/3}} 
+ \frac{1}{\mu^{5/3} T_i^{2/3}}\right) \\
& \quad + \exp \left(- C_2\mu  (S_l-l-1)\right) (f(x_1) - f^{\star})~.
\end{align*}

In the case of $r= 1$, $T_i = T_0$, we have for any $i$ that
\begin{align*}
& \sum_{i=0}^{l}\prod_{j= i+1}^{l} \exp(- C_2 \mu (T_j - 1)) \left(\frac{1}{\mu^{4/3} T_i^{4/3}} + \frac{1}{\mu^{5/3} T_i^{2/3}}\right)\\
& = \left(\frac{1}{\mu^{4/3} T_0^{4/3}} + \frac{1}{\mu^{5/3} T_0^{2/3}}\right) \sum_{i=0}^{l}\exp \left(- C_2 \mu (T_0 -1) (l-i)\right)   \\
& = \left(\frac{1}{\mu^{4/3} T_0^{4/3}} + \frac{1}{\mu^{5/3} T_0^{2/3}}\right)  \frac{1- \exp \left(-C_2 \mu(T_0 - 1) (l+1)\right)}{1- \exp\left(-C_2 \mu(T_0-1)\right)}~.
\end{align*}

In the case of $r > 1$, denote by $A_i = \prod_{j= i+1}^{l} \exp(- C_2 \mu (T_j - 1))   \frac{1}{\mu^pT_i^{q}} , p,q>0$.  For any $i = 0, ... , l$, $\frac{A_i}{A_{i+1}} = \exp(-C_2 \mu (T_{i+1} -1)) r^q$ is decreasing over $i$. Denote by $i^{\star} = \min \{ i: \frac{A_i}{A_{i+1}} \leq 1\}$.
We have 
\begin{align*}
\sum_{i=0}^{l}A_i 
& \leq A_0 \cdot  i^{\star} + (l - i^{\star} + 1) \cdot A_l \leq  (l + 1) \cdot (A_0 + A_l) \\
& =\frac{ l + 1  }{\mu^p}\cdot \left(\frac{1}{T_l^q} + \frac{1}{T_0^q}\exp(- C_2\mu (T-T_0-l))\right), \quad p, q >0~.
\end{align*}
Note that  $T_l = T \cdot \frac{r^l (r-1)}{r^{l+1}-1} \geq \frac{r-1}{r} T$ and $l = O \left(\ln T \right)$. Then, the stated bound follows. 
\end{proof}

%% file: 3_CosExp/disc.tex
\section{Conclusion}
\label{sec:disc}

We have analyzed theoretically the exponential and cosine step size, two successful step size decay schedules for the stochastic optimization of non-convex functions. We have shown that, up to poly-logarithmic terms, both step sizes guarantee convergence with the best-known rates for smooth non-convex functions. Moreover, in the case of functions satisfying the PL condition, we have also proved that they are both adaptive to the level of noise. 

%% file: 4_Momentum/momentum.tex
\chapter{Last Iterate of Momentum Methods}
\label{chapter:momentum}
\thispagestyle{myheadings}

\input{4_Momentum/intro}

\input{4_Momentum/rel}
\input{4_Momentum/setup}
\input{4_Momentum/low}
\input{4_Momentum/sgdm_FTRL}

\section{Conclusion}

We have presented an analysis of the convergence of the last iterate of SGDM in the convex setting.
We prove for the first time through a lower bound the suboptimal convergence rate for the last iterate
of SGDM with constant momentum after T iterations. Moreover, we study a class of FTRL-based
SGDM algorithms with increasing momentum and shrinking updates, of which the last iterate has
optimal convergence rate without projections onto bounded domain nor knowledge of T. 

%% file: 4_Momentum/intro.tex
\section{Introduction}

In this chapter, we consider SGDM with the following updates 
\begin{equation}
\label{eq:sgdm}
\bx_{t+1} = \bx_t - \eta_t \bm_t, \quad  \bm_t = \beta_t \bm_{t-1} + (1- \beta_t) \bg_t,
\end{equation}
where  $0 \leq \beta_t\leq 1$. 

Often an average is taken in the analysis of momentum, while in real world applications, most of the time only the last iterate is taken. Hence, we are interested in the performance of the last iterate of SGDM.  In particular, we study the convergence of the last iterate of SGDM for unconstrained optimization of convex functions. Unfortunately, our first result is a negative one: We show that the last iterate of SGDM can have a suboptimal convergence rate for \emph{any constant momentum setting}.

Motivated by the above result, we analyze yet another variant of SGDM.
We start from the very recent observation~\citep{Defazio20} that SGDM can be seen as a primal averaging procedure~\citep{NesterovS15,TaoPWT18,Cutkosky19} applied to the iterates of Online Mirror Descent (OMD)~\citep{NemirovskyY83,Warmuth97}. Based on this fact, we analyze SGDM algorithms based on the Follow-the-Regularized-Leader (FTRL) framework\footnote{FTRL is known in the offline optimization literature as Dual Averaging (DA)~\citep{Nesterov09}, but in reality DA is a special case of FTRL when the functions are linearized.}~\citep{Shalev-Shwartz07,AbernethyHR08} and the primal averaging. The use of FTRL instead of OMD removes the necessity of projections onto bounded domains, while the primal averaging acts as a momentum term and guarantees the optimal convergence of the last iterate. The resulting algorithm has an \emph{increasing momentum} and \emph{shrinking updates} that precisely allow to avoid our lower bound.


\begin{table}[t]
\caption{Last iterate convergence of momentum methods in convex setting}
\centering
\resizebox{\textwidth}{!}{
\begin{tabular}{c|c |c| c| c| c}
\hline\hline
Algorithm & Assumption& \makecell[c]{Bounded \\ Domain }  &  \makecell[c]{Requires \\T } &  Rate &  Reference \\ [0.5ex] 
\hline
Adaptive-HB& Assumption~\ref{assump: bounded_l2} &  Yes & No & $O(\frac{1}{\sqrt{T}})$ & \cite{TaoLWT21} \\
\hline  
\multirow{2}*{SHB-IMA} & \multirow{2}*{Smooth + 
Assumption~\ref{assump: bounded_variance}}& \multirow{2}*{No} & Yes & $O(\frac{1}{\sqrt{T}})$ &\multirow{2}*{\cite{SebbouhGD21}}  \\
\cline{4-5}
~ & ~ & ~& No & $O(\frac{\ln T}{\sqrt{T}})$ & ~\\
\hline
AC-SA & \makecell[c]{Lipschitz/Smooth + \\
Assumption~\ref{assump: bounded_l2}} & No  &Yes & $O(\frac{1}{\sqrt{T}})$ & \cite{GhadimiL12} \\
\hline 
\multirow{2}*{FTRL-SGDM}& Assumption~\ref{assump: bounded_expectation_G} & No& No  & $O(\frac{1}{\sqrt{T}})$  &  Corollary~\ref{cor:poly_constant_step} \\[1ex]
\cline{2-6}
~ & Smooth + Assumption~\ref{assump: bounded_variance}, ~\ref{assump: bounded_l2} & No& No  & $O(\frac{\ln T}{T} + \frac{\sigma}{\sqrt{T}})$  & Corollary~\ref{cor:ada_smooth} \\[1ex]
\hline
\end{tabular}}
\label{table:comparison}
\end{table}

%% file: 4_Momentum/rel.tex
\section{Related Work}
\label{sec:momentum_rel}


\paragraph{Lower bound}
\citet{HarveyLPR19} prove the tight convergence bound $O(\ln T/\sqrt{T})$ of the last iterate of SGD for convex and Lipschitz functions. \cite{KidambiNJK18} provide a lower bound for the Heavy Ball method for least square regression problems. To the best of our knowledge, there is no lower bound for the last iterate of SGDM in the general convex setting. 

\paragraph{Last iterate convergence of SGDM}
\citet{NesterovS15} introduces a quasi-monotone subgradient method, which uses double averaging based on Dual Averaging, to achieve the optimal convergence of the last iterate for the convex and Lipschitz functions. However, they just considered the batch case. This approach was then rediscovered and extended by \citet{Cutkosky19}.
Our FTRL-based SGDM is a generalization of the approach in \citet{NesterovS15} with generic regularizers and in the stochastic setting.
\citet{TaoPWT18} extends \citet{NesterovS15}'s method to Mirror Descent, calling it stochastic primal averaging. They recover the same bound for convex functions, again with a bounded domain assumption. 
They also get the optimal bound for strongly convex functions and analyze them in the stochastic and regularized setting.
\citet{Defazio20} points out that the sequence generated by the stochastic primal averaging \citep{TaoPWT18} can be identical to that of stochastic gradient descent with momentum for specific choices of the hyper-parameters. Accordingly, they give a Lyapunov analysis in the nonconvex and smooth case. 
Based on this work, \citet{JelassiD20} introduce ``Modernized dual averaging method'', which is actually equal to the one by \citet{NesterovS15}. 
They also give a similar Lyapunov analysis as in \citet{Defazio20} with specific choices of hyper-parameters 
in the non-convex and smooth optimization setting.
Recently, \citet{TaoLWT21} propose the very same algorithm as in \citet{TaoPWT18} and analyze it as a modified Polyak's Heavy-ball method (already pointed out by \citet{Defazio20}). They give an analysis in the convex cases and extend it to an adaptive version, obtaining in both cases an optimal convergence of the last iterate. However, they all assume the use of projections onto bounded domains.


\paragraph{Last iterate convergence rate $O(\frac{1}{\sqrt{T}})$}
\citet{GhadimiL12} present the last iterate of AC-SA~\citep{NemirovskiJLS09, Lan12} for convex functions in the unconstrained setting, that  in Euclidean case reduces to SGD with an increasing Nesterov momentum, showing that it can achieve a convergence rate $O(\frac{1}{\sqrt{T}})$ if the number of iterations $T$ is known in advance.
\citet{SebbouhGD21} analyze Stochastic Heavy Ball-Iterave Moving Average method (SHB-IMA), which is equal to Stochastic Heavy Ball method (SHB) with a specific choice of hyper-parameters. They prove a convergence rate for the last iterate of of $O (\frac{1}{\sqrt{T}})$ if $T$ is given in advance, and is $O(\frac{\log T}{\sqrt{T}})$ if $T$ is unknown.
\citet{JainNN21} conjecture that under the assumption that the stochastic gradients are bounded``for any-time algorithm (i.e., without apriori knowledge of $T$ ) expected error rate of $\frac{D G \ln T}{\sqrt{T}}$ is information theoretically optimal'', where $D$ is the diameter of the bounded domain. This was already disproved by the results in \citet{TaoLWT21}, but here we disprove it even in the more challenging unconstrained setting.

%% file: 4_Momentum/setup.tex
\section{Assumptions}
We will use one or more of the following assumptions on the stochastic gradients $\bg_t$. 

\begin{assumptionA}
\label{assump: bounded_variance}
Bouned Variance:  $\E_t \| \bg_t - \nabla f(\bx_t ) \|^2 \leq \sigma^2, \quad \sigma >0$.
\end{assumptionA}

\begin{assumptionA}
\label{assump: bounded_expectation_G}
Bounded in expectation: $\E_t \| \bg_t \|^2 \leq G^2, \quad G >0$.
\end{assumptionA}

\begin{assumptionA}
\label{assump: bounded_l2}
$\ell_2$ bounded: $\| \bg_t \| \leq G, \quad G >0$.
\end{assumptionA}

\begin{assumptionA}
\label{assump: bounded_l_inf}
$\ell_{\infty}$ bounded: $\| \bg_t \|_{\infty} \leq G_{\infty}, \quad G_{\infty} >0$.
\end{assumptionA}

%% file: 4_Momentum/low.tex
\section{Lower bound for SGDM}
\label{sec:lower}


In this section, we show the surprising result that for SGD with any constant momentum, there exists a function for which the lower bound of the last iterate is $\Omega \left( \log T/\sqrt{T}\right)$. Our proof extends the one in \citet{HarveyLPR19} to SGD with momentum.

We consider SGDM  with constant momentum factor $\beta$ in \eqref{eq:sgdm},
where $\bg_t \in \partial f(\bx_t)$ and a polynomial stepsize $\eta_t = c \cdot  t^{-\alpha}, 0 \leq \alpha \leq \frac{1}{2}$.  


For any fixed $\beta$ and $\alpha$ and $L >0$, we introduce the following function. 
Define $f$: $\mathcal{X} \to \R $ and $\bh_i \in \R^T$ for $i \in [T+1]$ by 
\begin{equation}
\label{eq:def_of_f}
f(\bx)  = \max_{i \in [T+1]} \bh_i^T \bx,
\quad
h_{i,j} = 
\begin{cases}
a_j,  &  1 \leq j < i  \\
- b_j,  & i = j < T\\
0, & i < j \leq T
\end{cases}
\end{equation}
where $b_j = \frac{Lj^{\alpha}}{2T^{\alpha}}$ and $a_j = \frac{L(1-\beta)}{8(T-j+1)}$. 
We have that $\partial f(\bx_t)$ is the convex hull of $\bh_i: i \in \mathcal{I} (\bx)$ where $\mathcal{I}(\bx) = \{i: \bh_i ^T \bx = f(\bx)\}$.
Note that $f$ is $L$-Lipschitz over $\R^T$ since
\begin{align}
\| \bh_i \|^2 
\leq \sum_{i=1}^{T} a_i^2 + b_T^2
\leq \frac{L^2(1-\beta)^2}{64} \sum_{i=1}^{T}\frac{1}{i^2} + \frac{L^2}{4}
\leq L^2 ~. 
\end{align}
\begin{claim}
\label{clm:f_min_0}
For $f$ defined in \eqref{eq:def_of_f}, it satisfies that $\inf_{\bx \in \R^T} f(\bx) = 0.$
\end{claim}
\begin{proof}
First, since $f(0) = 0$, we have that $\inf_{\bx \in \R^T} \ f(\bx) \leq 0$. 
\\
We continue to prove this claim by contradiction. Assume that there exists $\bx^{\star} = [x_1^{\star}, x_2^{\star},\dots, x_T^{\star}]$ such that \[f(\bx^{\star}) < 0~.\] 
By the definition of $f$, it satisfies that 
\begin{equation}
\label{eq:assumption_min_f_negative}
\bh_i^T\bx^{\star} < 0,\quad \forall i \in [T+1]~.
\end{equation}
In particular, $\bh_1^T \bx^{\star} = -b_1 x^{\star}_1 < 0$. Since that $b_1$ is positive, we know that $x_1^{\star} >0$. Also, $\bh_2^T \bx^{\star} = a_1 x^{\star}_1 - b_2  x^{\star}_2 < 0$. Due to the positiveness of $a_1, x^{\star}_1$, and $b_2$, $x^{\star}_2$ has to be positive. Similarly, we have that for any $x^{\star}_j$, $j \in [T]$, $x^{\star}_j > 0$.
\\
Then, we have 
\[\bh_{T+1}^T \bx^{\star} = \sum_{j=1}^T a_j x^{\star}_j > 0~.\] However, this is contradict with \eqref{eq:assumption_min_f_negative}. 
\\
Thus, we conclude that $\inf_{\bx \in \R^T} f(\bx) = 0.$
\end{proof}
\begin{theorem}[Lower bound of SGDM]
\label{thm:lower_bound}
Fix a polynomial stepsize sequence $\eta_t =c \cdot t^{-\alpha}$, where $0 \leq \alpha \leq \frac{1}{2}$, a momentum factors $\beta \in [0,1)$, a Lipschitz constant $L > 0$ and a number of iterations $T$. Then, there exists a sequence $\bz_t$ generated by SGDM with stepsizes $\eta_t$ and momentum factor $\beta$ on the function $f$ in \eqref{eq:def_of_f}, where the $T$-th iterate satisfies
\[
f(\bz_T) - f^{\star} 
\geq 
\frac{L^2(1-\beta)^2c \ln T}{4T^{\alpha}} ~.
\]
\end{theorem}

We stress that $\ln T$ cannot be cancelled by any setting of $\beta$. Indeed, the above lower bound can be instantiated by any $\beta$ and any $T$. Hence, for a given $\beta$, there exists $T$ large enough such that $\ln T$ is constant-times bigger than $\frac{1}{(1-\beta)^2}$.

When $\beta = 1$, the algorithm is basically staying at the initial point.  We can choose arbitrary positive number $C>0$ and let $z_1 = C$, then
\[
f(\bz_T) - f^{\star} \geq C , \quad C>0~.
\]

We will use the following lemma in the proof. 
\begin{lemma}
	\label{lemma:poly_bound}
	For any $1 \leq j \leq t \leq T$ and $0< \alpha \leq \frac{1}{2}$, we have $\frac{1}{T-j+1} \sum_{k=j+1}^{t} \frac{1}{j^{\alpha}} \leq  \frac{2}{T^{\alpha}}$. 
\end{lemma}
\begin{proof}
	First, we observe that 
	\begin{align*}
		\sum_{k=j+1}^{t} \frac{1}{k^{\alpha}} 
		&\leq \int_{j}^{t}  \frac{1}{x^{\alpha}} dx 
		= \frac{t^{1-\alpha} - j^{1-\alpha}}{1-\alpha} 
		= \frac{1}{1-\alpha}\frac{t^{2-2\alpha} - j^{2-2\alpha}}{t^{1-\alpha} + j^{1-\alpha}} \\
		&\leq \frac{1}{1-\alpha}\frac{(2-2\alpha)t^{1-2\alpha}  (t-j)}{t^{1-\alpha} + j^{1-\alpha}} 
		\leq \frac{2(t-j)}{t^{\alpha}},
	\end{align*}
	where in the second inequality we used the convexity of $f(x) = x^{2-2\alpha}, 0 < \alpha \leq \frac{1}{2}$. 
	
	Then, we claim $\frac{1}{T-j+1}\frac{t-j}{t^{\alpha}} \leq \frac{1}{T^{\alpha}}$. 
	
	Let $g(x) = \frac{x- j}{x^{\alpha}}$. The derivative 
	$g'(x) = \frac{1-\alpha + \frac{j}{\alpha x}}{x^{\alpha}}$ is positive for all $x> 0 $ and $j \geq 0$. So 
	it satisfies that $\frac{t-j }{t^{\alpha}} \leq \frac{T-j}{T^{\alpha}}$, which implies the claim. 
\end{proof}
\begin{proof}[Proof of Theorem~\ref{thm:lower_bound}]
Define a sequence $\bz_t$ for $t \in [T+1]$ as follows: $\bz = 0$, where $s$ is a positive number decided later, and 
\begin{equation}
\label{eq:def_z_t}
\bz_{t+1}  = \bz_t - (1-\beta) \eta_t \sum_{i=1}^{t} \beta^{t-i} \bh_i~. 
\end{equation}
We will show that $\bz_t$ are exactly the updates of SGDM and $f(\bz_{T+1}) \geq  \Omega \left(\frac{\ln T}{T^{\alpha}} \right)$.
We will use the following two lemmas.
\begin{lemma}
\label{lemma:bound_z_t}
Let  $b_j = \frac{L j^{\alpha}}{2T^{\alpha}}$, $a_j = \frac{L (1 - \beta)}{8 (T-j+1)}$, and $\eta_j = c \cdot j^{-\alpha}$. $\bz_t$ is defined as in \eqref{eq:def_z_t}. Then, for $1 \leq t < j $, $z_{t,j} = 0$,  and for $t > j $, $z_{t,j} \geq \frac{L (1-\beta) c}{4T^{\alpha} }$. 
\end{lemma}
\begin{proof}[Proof of Lemma~\ref{lemma:bound_z_t}]
We first prove by induction that when $1 \leq t \leq j$, $z_{t,j} = 0$. 
First, $\bz_1=0$
Also, suppose it holds for $t$. 
Then, in the case of $t+1$, for any $j \geq t+1$, 
\[
z_{t+1, j} = z_{t,j} - (1-\beta)  \eta_t \sum_{i=1}^{t} \beta^{t-i} h_{i, j} = 0 - 0 = 0, 
\]
which implies $t \leq j$, $z_{t,j} = 0$ holds. 
Next, we claim that $\bz_t$ satisfies 
\begin{equation}
\begin{aligned}
\label{eq:lower_bound_z_t}
z_{t,j} 
& \geq z_{j,j} + (1-\beta) b_j \eta_j - a_j\sum_{k=j+1}^{t-1} \eta_k, \quad 1\leq j < t \leq T~. 
\end{aligned}
\end{equation}
We prove \eqref{eq:lower_bound_z_t} by induction. 
For any $t$, $z_{t,t-1}$ satisfies \eqref{eq:lower_bound_z_t} since 
\begin{align*}
z_{t+1, t} 
= z_{t,t} - (1-\beta)  \eta_t \sum_{i=1}^{t} \beta^{t-i} h_{i,t}
= - (1-\beta)  \eta_t h_{t,t} =  (1-\beta)\eta_t b_t~. 
\end{align*}
Then, suppose \eqref{eq:lower_bound_z_t} holds for any $j < t$. We show that it holds for any $j < t+1$.  We already proved for $j = t$. For $j < t$,
\vspace{-0.2cm}
\begin{align}
z_{t+1, j} 
&= z_{t,j} - (1-\beta)  \eta_t \sum_{i=1}^{t} \beta^{t-i} h_{i,j}
= z_{t,j} - (1-\beta)  \eta_t \sum_{i=j}^{t} \beta^{t-i} h_{i,j} \nonumber\\
& = z_{t,j} + (1-\beta)  \eta_t \beta^{t-j} b_j - (1-\beta)  \eta_t \sum_{i=j+1}^{t} \beta^{t-i} h_{i,j}\nonumber\\
& \geq z_{t,j}  -   (1-\beta) \eta_t \sum_{i=j+1}^{t} \beta^{t-i} a_j\nonumber \\
&  \geq (1-\beta) b_j \eta_j - a_j\sum_{k=j+1}^{t-1} \eta_k -  (1-\beta)a_j \eta_t\sum_{i=j+1}^{t} \beta^{t-i}  \nonumber\\
&  \geq  (1-\beta) b_j \eta_j - a_j\sum_{k=j+1}^{t} \eta_k,  \label{eq:z_t}
\end{align}
where in the second inequality we used the induction hypothesis.

Using that $b_j = \frac{L j^{\alpha}}{2T^{\alpha}}$, $a_j = \frac{L (1 - \beta)}{8 (T-j+1)}$ 
and $\eta_j = \frac{c}{j^{\alpha}}$, we have
\vspace{-0.2cm}
\begin{align}
\eqref{eq:z_t}
=  \frac{L(1-\beta) c}{2T^{\alpha}} - \frac{L(1-\beta) c}{8(T-j+1)} \sum_{k=j+1}^t \frac{1}{k^{\alpha}}~.
\label{eq:lower_alpha}
\end{align}
By Lemma~\ref{lemma:poly_bound} in the Appendix, we have that for $0 < \alpha \leq \frac{1}{2}$, 
\[
\eqref{eq:lower_alpha}
\geq \frac{L(1-\beta) c}{2T^{\alpha}} -  \frac{L(1-\beta) c}{4T^{\alpha}}
\geq \frac{L(1-\beta) c}{4T^{\alpha}}, 
\]
and for $\alpha= 0$, 
\[
\eqref{eq:lower_alpha} \geq \frac{L (1-\beta) c}{2} -  \frac{L (1-\beta) (t- j -1)c}{8(T-j+1)}\geq \frac{L (1-\beta) c}{4}~.
\] 
Thus, we have $z_{t,j} \geq \frac{L (1-\beta) c}{4T^{\alpha} } \geq \frac{L (1-\beta) c}{4T^{\alpha} }$. 
\end{proof}
\begin{lemma}
\label{lemma:remove_max}
 $f(\bz_t) = \bh_t^T \bz_t$ for any $t \in [T+1]$. The subgradient oracle for $f$ at $\bz_t$ returns $\bh_t$.
\end{lemma}  
\begin{proof}[Proof of Lemma~\ref{lemma:remove_max}]

We claim that $\bh_t^T \bz_t = \bh_i^T\bz_t$ for all $i > t \geq 1$ and $\bh_t^T \bz_t > \bh_i^T \bz_t$ for all $1\leq i < t$.  

When $i>t\geq 2$, $\bz_t$ is supported on the first $t-1$ coordinates, while $\bh_t$ and $\bh_i $ agree on the first $t-1$ coordinates. 
	
In the case of $1 \leq i<t$, by the definition of $\bz_t$ and $\bh_t$, we have 
\begin{equation*}
\label{eq:z_t_times_h_t_minus_h_i}
\bz_t^T (\bh_t - \bh_i) 
= \sum_{j=1}^{t-1} z_{t,j} (h_{t,j} - h_{i,j}) 
= \sum_{j=i}^{t-1} z_{t,j} (h_{t,j} - h_{i,j}) 
= z_{t,i} (a_i + b_i) + \sum_{j=i+1}^{t-1} z_{t,j} a_j > 0, 
\end{equation*}
where in the last inequality we used the fact that $a_i$, $b_i$  and $z_{t,i}$ are at least non-negative.

Thus, we have proved $f(\bz_t) = \bh_t^T \bz_t$ by the definition. Moreover, $\mathcal{I} (\bz_t) = \{ i: \bh_i^T \bz_t  = f(\bz_t)\} = \{t, \dots, T+1\}$. So the subgradient evaluated at $\bz_t$ is $\bh_t$. 
\end{proof}

Now, we first get a lower bound and an upper bound of $\bz_t$ using Lemma~\ref{lemma:bound_z_t}. 
Then, by Lemma~\ref{lemma:remove_max}, we have shown that $\bz_t$ are exactly the updates of SGDM.

Thus, for $\beta \in [0,1)$, we have 
\vspace{-0.2cm}
\begin{align*}
f(\bz_{T+1}) 
& = \bh_{T+1}^T \bz_{T+1} 
= \sum_{j=1}^{T} h_{T+1,j} z_{T+1,j} \\
& \geq \frac{L^2(1-\beta)^2c}{4T^{\alpha}}\sum_{j=1}^{T} \frac{1}{T-j+1} 
\geq \frac{L^2(1-\beta)^2c \ln T}{4T^{\alpha}}~.
\qedhere
\end{align*}
\end{proof}

%% file: 4_Momentum/sgdm_FTRL.tex
\section{FTRL-based  SGDM}
\label{sec:ftrl_m}

The lower bound for the last iterate in the previous Section motivates us to study a different variant of SGDM.
In particular, we aim to find a way to remove the $\ln T$ term from the convergence rate.

\citet{Defazio20} points out that the stochastic primal averaging method~\citep{TaoPWT18} (which is also an instance of Algorithm~1 in  \citet{Cutkosky19} with OMD):
\begin{align*}
\bz_{t+1} =\bz_t - \gamma_t \bg_t, 
\quad 
\bx_{t+1} = s_t \bx_t +(1- s_t ) \bz_t
\end{align*}
could be one-to-one mapped to the momentum method
\begin{align*}
\bm_{t+1} =  \beta_t \bm_t + \bg_t, 
\quad 
\bx_{t+1}= \bx_t - \alpha_t \bm_t 
\end{align*}
by setting $\gamma_{t+1} = \frac{\gamma_t - \alpha_t}{\beta_{t+1}}$.  
While this is true, the convergence rate depends on the convergence rate of OMD with time-varying stepsizes, that in turn requires to assume that $\| \bx_t  - \bx^{\star} \|^2 \leq D^2 $. This is possible only by using a projection onto a bounded domain in each step.

\begin{algorithm}[t]
\caption{FTRL-based SGDM}
\label{alg:sgdm_0}
\begin{algorithmic}[1]
\STATE \textbf{Input:} A sequence $\alpha_1, ..., \alpha_T$,  with $\alpha_1 > 0$. Non-increasing sequence $\bgamma_1, \dots,  \bgamma_{T-1}$. $\bm_0 = 0$. $\bx_1 \in \R^d$.
\FOR{$t = 1, \dots, T$}
\STATE  Get $\bg_t$ at $\bx_t$ such that $\E_t \left[\bg_t \right]= \nabla f(\bx_t )$
\STATE $\beta_t = \frac{\sum_{i=1}^{t-1} \alpha_i }{\sum_{i=1}^{t} \alpha_i } $ (Define $\sum_{i=1}^{0} \alpha_i = 0$)
\STATE $\bm_t = \beta_t \bm_{t-1} +(1- \beta_t) \bg_t$
\STATE $\bta_t = \frac{\alpha_{t+1}  \sum_{i=1}^{t} \alpha_i }{\sum_{i=1}^{t+1} \alpha_i} \bgamma_t  $
\STATE $\bx_{t+1} = \frac{\sum_{i=1}^{t} \alpha_i }{\sum_{i=1}^{t+1} \alpha_i} \bx_t + \frac{\alpha_{t+1}}{\sum_{i=1}^{t+1} \alpha_{i}} \bx_1 - \bta_t \bm_t $
\ENDFOR
\end{algorithmic}
\end{algorithm}

Thus, to go beyond bounded domains, we propose to study a new variant of SGDM which has the following form (details in Algorithm~\ref{alg:sgdm_0}), 
\begin{align*}
\bm_{t+1} =  \beta_t \bm_t + (1- \beta_t)\bg_t, 
\quad
\bx_{t+1} = s_t \bx_t - \alpha_t \bm_t ~.
\end{align*}
This variant comes naturally when using the primal averaging scheme with FTRL (Algorithm~\ref{alg:ftrl}) rather than to OMD. Hence, we just denote it by FTRL-based SGDM.
\begin{algorithm}[h]
	\caption{Follow-the-Regularized-Leader on Linearized Losses}
	\label{alg:ftrl}
	\begin{algorithmic}[1]
		\STATE \textbf{Input:} Regularizers $\psi_1, \dots, \psi_T: \R^d \to(-\infty, \infty ]$. 
		\FOR{$t = 1, \dots, T$}
		\STATE $\bw_t \in \argmin_{\bw \in \R^d} \ \psi_t (\bw) + \sum_{i=1}^{t-1} \langle \bg_i , \bw \rangle$
		\STATE Receive $\ell_t:  \R^d \to(-\infty, \infty ]$ and pay $\ell_t (\bw_t)$
		\STATE Set $\bg_t \in \partial \ell_t (\bw_t)$
		\ENDFOR
	\end{algorithmic}
\end{algorithm}

\subsection{Convergence Rates for FTRL-based SGDM}

We first present a very general theorem for FTRL-based SGDM.
\begin{theorem}
\label{thm:sgdm_0}
Assume $f$ is convex and $L$-smooth. Algorithm~\ref{alg:sgdm_0} guarantees
\begin{align*}
& \E\left[f(\bx_T ) - f(\bx^{\star})\right] 
\leq\frac{1}{\sum_{t=1}^{T} \alpha_t} \E \left[\left\| \frac{\bx_1 - \bx^{\star} }{ \sqrt{\bgamma_{T-1}}}  \right\|^2+ \sum_{t=1}^{T}\langle \bgamma_{t-1}, \alpha_t^2 \bg_t^2 \rangle\right]~. 
\end{align*}
\end{theorem}

The above theorem is very general and it gives rise to a number of different variations of the FTRL-based SGDM. In particular, we can instantiate it with the following choices.

First, we consider the most used polynomial stepsize $\frac{c}{\sqrt{t}}$ for convex and Lipschitz function, and the constant stepsize $\frac{c}{\sqrt{T}}$ if $T$ is given in advance. 
\begin{cor}
\label{cor:poly_constant_step}
Assume $\E \| \bg_t \|^2 \leq G^2, G>0$ and set $\alpha_t = 1$ for all $t$. Algorithm~\ref{alg:sgdm_0} with either $\bgamma_{t-1}= \frac{c}{G\sqrt{t}} \cdot \mathbf{1}$ or $\bgamma_{t-1} = \frac{c}{G\sqrt{T}}\cdot \mathbf{1}$ guarantees 
\[
\E\left[f(\bx_T) - f(\bx^{\star})\right] \leq \frac{\|\bx_1 - \bx^{\star} \|^2 G }{c\sqrt{T}} + \frac{2cG}{\sqrt{T}}~. 
\]
\end{cor}

The above corollary tells that both of these two stepsizes give the optimal bound $O(\frac{1}{\sqrt{T}})$ for the last iterate. Next, we will show that if we use an adaptive\footnote{Even if widely used in the literature, it is a misnomer to call these stepsize ``adaptive'': an algorithm can be adaptive to some unknown quantities (if proved so), not the stepsizes.} stepsize,  Algorithm~\ref{alg:sgdm_0} gives a data-dependent convergence rate for the last iterate. We first consider a global version of the AdaGrad stepsize as in \citet{StreeterM10,WardWB19}. 
\begin{cor}
\label{cor:adaptive_norm}
Assume $\| \bg_t \|\leq G, G>0$ and take  $\bgamma_t = \frac{\alpha \cdot \mathbf{1}}{\sqrt{G^2 + \sum_{i=1}^{t} \alpha_i^2 \| \bg_i \|^2}}, 1 \leq t \leq T$ and $\alpha_t = 1$. Then, Algorithm~\ref{alg:sgdm_0} guarantees
\begin{align*}
& \E\left[f(\bx_T) - f(\bx^{\star})\right] 
\leq \frac{2}{T} \left(\frac{\|\bx_1 - \bx^{\star} \|^2}{\alpha} + \alpha\right) \sqrt{\sum_{t=1}^{T-1} \| \bg_t \|^2 + G^2  }~.
\end{align*}
\end{cor}
\begin{sloppypar}
We also state a result for the coordinate-wise AdaGrad stepsizes~\citep{McMahanS10, DuchiHS11}.
\end{sloppypar}
\begin{cor}
\label{cor:adaptive_coord}
Assume $\| \bg_t \|_{\infty} \leq G_{\infty}, G_{\infty}>0$and set $\bgamma_t = \frac{\alpha }{\sqrt{G_{\infty}^2 + \sum_{i=1}^{t} \alpha_i^2  \bg_i^2}}, 1 \leq t \leq T$ and $\alpha_t = 1$. Then, Algorithm~\ref{alg:sgdm_0} guarantees
\begin{align*}
\E\left[f(\bx_T) - f(\bx^{\star})\right] 
\leq \frac{2}{T}  \left(\frac{\|\bx_1 - \bx^{\star} \|_1^2}{\alpha} + \alpha\right) \sum_{j=1}^{d}\sqrt{\sum_{t=1}^{T-1} \bg_{t,j}^2 + G^2_{\infty}} ~.
\end{align*}
\end{cor}
They show that the convergence bounds are adaptive to the stochastic gradients. 
When the stochastic gradients are small or sparse, the rate could be much faster than $O(\frac{1}{\sqrt{T}})$. 
Moreover, the above results give very simple ways to obtain optimal convergence for the last iterate of first-order stochastic methods, that was still unclear if it could be obtained as discussed in \citet{JainNN21}.

Also, we now show that if in addition $f$ is smooth, the last iterate of FTRL-based momentum with the global adaptive stepsize of Corollary~\ref{cor:adaptive_norm} gives adaptive rates of convergence that interpolate between $O(\frac{1}{\sqrt{T}})$ and $O(\frac{\ln T}{T})$.
\begin{cor}
\label{cor:ada_smooth}
Assume $f$ is $L$-smooth. Then, under the same assumption and parameter setting of Corollary~\ref{cor:adaptive_norm},  Algorithm~\ref{alg:sgdm_0} guarantees
\begin{align*}
\E &\left[f(\bx_T) - f^{\star}\right] 
\leq \frac{2C}{T} \sqrt{16L^2C^2 \ln^2 T + 8LG C \ln T + G^2} + \frac{2\sqrt{2}C \sigma}{\sqrt{T}}~.
\end{align*}
where $C  \triangleq \left( \frac{\| \bx_1 - \bx^{\star} \|^2 }{\alpha} + \alpha\right)$. 
\end{cor}
Observe that when $\sigma = 0$, namely when there is no noise on the gradients, the rate of $O(\frac{\ln T }{T})$ is obtained. 
As far as we know, the above theorems are the first convergence guarantees for the last iterate of adaptive methods in 
unconstrained convex optimization.

First, we state some technical lemmas for the proofs. 

The following lemma is a well-known result for FTRL~\citep[see, e.g.,][]{Orabona19}.
\begin{lemma}
	\label{lemma:ftrl}
	Let $\ell_t$ a sequence of convex loss functions. Set the sequence of regularizers as $\psi_t({\bx})= \left\| \frac{\bx_1 - \bu}{\sqrt{\bgamma_{t-1}}} \right\|^2$, where $\bgamma_{t+1} \leq \bgamma_t, \ t=1, \dots, T$. Then, FTRL (Algorithm~\ref{alg:ftrl}) guarantees
	\begin{align*}
		\sum_{t=1}^{T} \ell_t (\bx_t ) - \ell_t(\bu)
		\leq  \left\| \frac{\bu - \bx_1 }{ \sqrt{\bgamma_{T-1}}} \right \|^2 + \frac{1}{2} \sum_{t=1}^{T}  \langle \bgamma_{t-1},  \bg_t^2\rangle~.
	\end{align*}
\end{lemma}

\begin{lemma}{\citep[Lemma 14]{GaillardSV14}}
	\label{lemma:sum_intergral_bounds_extra_term}
	Let $a_0 > 0$ and $a_1, \dots, a_m \in [0,A]$ be real numbers and let $f:(0,+\infty)\rightarrow [0, +\infty)$ nonincreasing function. Then 
	\[
	\sum_{i=1}^{m} a_i f(a_0 + \dots  + a_{i-1})
	\leq  \int_{a_0}^{\sum_{i=0}^m a_i} f(u) du + Af(a_0) ~.
	\]
\end{lemma}
\begin{proof}
	Denote by $s_t=\sum_{i=0}^{t} a_i$.
	\begin{align*}
		\sum_{i=1}^{m} a_i f(s_{i-1}) 
		& =  \sum_{i=1}^{m} a_i f(s_i)  + \sum_{i=1}^{m} a_i (f(s_{i-1}) - f(s_i))\\
		& \leq  \sum_{i=1}^{m} a_i f(s_i)  + A\sum_{i=1}^{m} (f(s_{i-1}) - f(s_i))\\
		& \leq  \sum_{i=1}^{m} \int_{s_{i-1}}^{s_i} f(x) d x + A\sum_{i=1}^{m} (f(s_{i-1}) - f(s_i))\\
		& \leq \int_{a_0}^{\sum_{i=0}^m a_i} f(u) du + Af(a_0),
	\end{align*}
	where the first inequality holds because $f(x_{i-1}) \geq f(s_i)$ and $a_i \leq A$, while the second inequality uses the fact that $f$ is nonincreasing together with $s_i - s_{i-1} = a_i$. 
\end{proof}

We can now present the proofs of the Corollaries 2-4.

\begin{proof}[Proof of Corollary~\ref{cor:adaptive_norm} and Corollary~\ref{cor:adaptive_coord}]
	By Lemma~\ref{lemma:sum_intergral_bounds_extra_term}, for adaptive stepsize \\$\bgamma_t =  \frac{\alpha \cdot \mathbf{1}}{\sqrt{\epsilon  + \sum_{i=1}^{t} \alpha_i^2 \| \bg_i \|^2}}$, we have 
	\begin{align*}
		\sum_{t=1}^{T} \bgamma_{t-1} \| \bg_t \|^2
		= \sum_{t=1}^{T} \frac{\alpha\| \bg_t \|^2}{\sqrt{\epsilon  + \sum_{i=1}^{t-1} \| \bg_i \|^2}} 
		\leq 2\alpha \sqrt{\sum_{t=1}^{T} \| \bg_t \|^2 } 
		+  \frac{\alpha G^2}{\sqrt{\epsilon}}~. 
	\end{align*}
	Similarly for $ \bgamma_t = \frac{\alpha }{\sqrt{\epsilon  + \sum_{i=1}^{t} \alpha_i^2  \bg_i^2}}$, we have
	\begin{align*}
		\sum_{t=1}^{T} & \langle \bgamma_{t-1},  \bg_t^2 \rangle 
		= \sum_{j=1}^{d}\sum_{t=1}^{T} \frac{ \alpha \bg_{t,j}^2}{\sqrt{\epsilon  + \sum_{i=1}^{t-1} \bg_{i,j}^2}} 
		\leq 2\alpha \sum_{j=1}^{d} \sqrt{\sum_{t=1}^{T} \bg_{t,j}^2}
		+ \frac{\alpha dG_{\infty}^2}{\sqrt{\epsilon}}~. 
		\qedhere 
	\end{align*}

\end{proof}

\begin{proof}[Proof of Corollary~\ref{cor:ada_smooth}]
	By Corollary~\ref{cor:adaptive_norm}, we have 
	\begin{align}
		& \E\left[f(\bx_T) \right] - f^{\star}
		\leq \frac{1}{T} \left[\left(\frac{\|\bx_1 - \bx^{\star} \|^2}{\alpha} + 2\alpha\right) \sqrt{\epsilon +  \E \sum_{t=1}^{T} \| \bg_t \|^2 } 
		+ \frac{\alpha G^2}{\sqrt{\epsilon}}\right]~.\label{eq:adaptive_bound}
	\end{align}
	
	From the unbiasedness of the gradients, we have
	\begin{align*}
		&  \E \sum_{t=1}^{T} \| \bg_t \|^2 
		\leq \E \sum_{t=1}^{T} \| \nabla f(\bx_t) \|^2 + T \sigma^2,
	\end{align*}
	and 
	\begin{align*}
		\E \sum_{t=1}^{T} \| \nabla f(\bx_t) \|^2
		& \leq 2L \sum_{t=1}^{T} \E \left[f(\bx_t ) \right] - f^{\star}\\
		& \leq 2L \left(\frac{\|\bx_1 - \bx^{\star} \|^2}{\alpha} + 2\alpha\right) \sum_{t=1}^{T}\frac{ \sqrt{ \epsilon + \E \sum_{i=1}^{t} \| \bg_i\|^2 }}{t} \\
		& \leq 2L \left(\frac{\|\bx_1 - \bx^{\star} \|^2}{\alpha} + 2\alpha\right)\cdot \left(\sqrt{ \E \sum_{t=1}^{T} \| \bg_t\|^2+ \epsilon} + \frac{\alpha G^2}{\sqrt{\epsilon}}\right) \ln T,
	\end{align*}
	where in the second inequality we used Lemma~\ref{lemma:smooth} and Holder's and Jensen's inequalities in the third inequality.
	
	Solve for $\E \sum_{t=1}^{T} \| \bg_t \|^2$ to have
	\begin{align*}
		& \E \sum_{t=1}^{T} \| \bg_t \|^2 \\
		& \leq 4L^2\left(\frac{\|\bx_1 - \bx^{\star} \|^2}{\alpha} + 2\alpha\right)^2 \ln^2 T 
		+ 4L\sqrt{\epsilon} \left(\frac{\|\bx_1 - \bx^{\star} \|^2}{\alpha} + 2\alpha\right) \ln T + \frac{2\alpha G^2}{\sqrt{\epsilon}} + 2T \sigma^2~. 
	\end{align*}	
	Putting it back to \eqref{eq:adaptive_bound}, we get the stated bound. 
\end{proof}

\subsection{Convergence Rate in Interpolation Regime}
\label{sec:interpolation}

Now we assume that $F(\bx)=\E_{\xi} [f(\bx,\xi)]$ and that the stochastic gradient is calculated drawing one function in each time step and calculating its gradient: $\bg_t = \nabla f(\bx_t,\xi_t)$. In this scenario, it makes sense to consider the \emph{interpolation} condition~\citep{NeedellSW15,MaBB18}
\begin{equation}
\label{eq:interpolation}
\bx^\star \in \argmin_{\bx} \ F(\bx) \Rightarrow \bx^\star \in \argmin_{\bx} \ f(\bx,\xi), \ \forall \xi~.
\end{equation}
This condition says that the problem is ``easy'', in the sense that all the functions in the expectation share the same minimizer. This case morally corresponds to the case in which there is no noise on the stochastic gradients. However, this condition seems weaker because it says that only in the optimum the gradient is exact and noisy everywhere else.
We will also assume that each function $f(\bx,\xi)$ is $L$-smooth in the first argument.

\begin{theorem}
\label{thm:interpolation}
Assume $f$ is $L$-smooth and $\| \bg_t \| \leq G, G>0$.  Then, under the interpolation assumption in~\eqref{eq:interpolation}, Algorithm~\ref{alg:sgdm_0} with $\bgamma_t  = \frac{\alpha \cdot \mathbf{1}}{\sqrt{G^2 + \sum_{i=1}^{t} \alpha_i^2 \| \bg_i \|^2}}$ guarantees 
\begin{align*}
\E F(\bx_T) - F(\bx^{\star}) 
\leq \frac{2C}{T} \sqrt{ 16 L^2 C^2 \ln^2 T + 8L C \ln T + G^2  }, 
\end{align*}
where $C \triangleq \left(\frac{\|\bx_1 - \bx^{\star} \|^2}{\alpha} + \alpha\right)$. 
\end{theorem}

To the best of our knowledge, this is the first rate for the last iterate of momentum methods in the interpolation setting.

\subsection{Proofs}

Before presenting the proofs of our convergence rates, we revisit the Online-to-Batch algorithm (Algorithm~\ref{alg:online_to_batch}) by \citet{Cutkosky19}, which introduce a modification to any online learning algorithm to obtain a guarantee on the last iterate in the stochastic convex setting. 
\begin{lemma}{\citep[Theorem~1]{Cutkosky19}}
\label{lemma:ashok}
Assume $\E \| \bg_t - \nabla f(\bx_t) \| \leq \sigma^2, \sigma > 0$.
Then, for all $\bx^{\star} \in D$, Algorithm~\ref{alg:online_to_batch} guarantees
\begin{equation}
\label{eq:conversion}
\E [f(\bx_T) ]  - f^{\star}
\leq \E \left[\frac{R_T (\bx^{\star})}{\sum_{t=1}^{T} \alpha_t}\right]~. 
\end{equation}
\end{lemma}

\begin{algorithm}[t]
\caption{Anytime Online-to-Batch~\citep{Cutkosky19}}
\label{alg:online_to_batch}
\begin{algorithmic}[1]
\STATE \textbf{Input:} Online learning algorithm $\mathcal{A}$ with convex domain D, $\alpha_1, ..., \alpha_T$, with $\alpha_1 > 0$.
\STATE Get Initial point $\bw_1$ from $\mathcal{A}$
\FOR{$t = 1, \dots, T$}
\STATE $\bx_t = \frac{\sum_{i=1}^{t} \alpha_i \bw_i }{\sum_{i=1}^{t} \alpha_i }$
\STATE Play $\bx_t$, receive subgradient $\bg_t$
\STATE Send $\ell_t (\bx) = \langle \alpha_t \bg_t, \bx \rangle$ to $\mathcal{A}$ as the $t$th loss
\STATE Get $\bw_{t+1}$ from $\mathcal{A}$
\ENDFOR
\end{algorithmic}
\end{algorithm}

\begin{algorithm}[t]
\caption{Anytime Online-to-Batch with FTRL}
\label{alg:online_to_batch_ftrl}
\begin{algorithmic}[1]
\STATE \textbf{Input:} $\alpha_1, ..., \alpha_T$, with $\alpha_t > 0$. $0 < \gamma_{t+1} \leq \gamma_t$. 
\STATE Initialize  $\bw_1$
\FOR{$t = 1, \dots, T$}
\STATE $\bx_t = \frac{\sum_{i=1}^{t} \alpha_i \bw_i }{\sum_{i=1}^{t} \alpha_i }$
\STATE Play $\bx_t$, receive subgradient $\bg_t$
\STATE $\bw_{t+1} = \bw_1 - \bgamma_t \sum_{i=1}^{t} \alpha_i \bg_i$
\ENDFOR
\end{algorithmic}
\end{algorithm}

Set $\psi_t({\bx})=\| \frac{\bx_1 - \bx}{\sqrt{\bgamma_{t-1}}} \|^2, 1\leq t\leq T$ as the regularizers of FTRL, where $\bgamma_{t+1} \leq \bgamma_t$ and $\bgamma_0 > 0$. Then, we write FTRL with loss $\ell_t(\bw) = \langle \alpha_t \bg_t, \bw \rangle$ as
\[
\bw_t
\in \argmin_{\bw \in \R^d} \ \psi_t (\bw) + \sum_{i=1}^{t-1} \langle \alpha_i \bg_i , \bw \rangle
= \bw_1 - \bgamma_{t-1}\sum_{i=1}^{t-1} \alpha_i \bg_i~. 
\]

We then plug FTRL into Algorithm~\ref{alg:online_to_batch} and it gives Algorithm~\ref{alg:online_to_batch_ftrl}.
Hence, using the well-known regret upper bound of FTRL (Lemma~\ref{lemma:ftrl}), we get the following Lemma.
\begin{lemma}
\label{lemma: online_to_batch_ftrl}
Under the same setting with Lemma~\ref{lemma:ashok}, Algorithm~\ref{alg:online_to_batch_ftrl} guarantees
\begin{align*}
\E&\left[f(\bx_T) \right]  - f^{\star}
\leq \frac{1}{\sum_{t=1}^{T} \alpha_t} \E \left[\left\| \frac{\bu - \bx_1 }{ \sqrt{\bgamma_{T-1}}}  \right\|^2+ \sum_{t=1}^{T}\langle \bgamma_{t-1}, \alpha_t^2 \bg_t^2 \rangle\right]~. 
\end{align*}
\end{lemma}

Now we prove the connection between the FTRL-based SGDM and Algorithm~\ref{alg:online_to_batch_ftrl}. 
\begin{proof}[Proof of Theorem~\ref{thm:sgdm_0}]
We prove that the updates of $\bx_t$ in Algorithm~\ref{alg:sgdm_0} can be one-to-one mapped to the updates of $\bx_t$ Algorithm~\ref{alg:online_to_batch_ftrl} when $\bw_1 = \bx_1$. 

The update of $\bx_t $ in Algorithm~\ref{alg:online_to_batch_ftrl} can be written as following: 
\begin{align*}
\bx_{t+1}
= \frac{\sum_{i=1}^{t}\alpha_i}{\sum_{i=1}^{t+1}\alpha_i} \bx_t + \frac{\alpha_{t+1}}{\sum_{i=1}^{t+1} \alpha_i} \bw_{t+1}
= \frac{\sum_{i=1}^{t}\alpha_i}{\sum_{i=1}^{t+1}\alpha_i} \bx_t + \frac{\alpha_{t+1}}{\sum_{i=1}^{t+1} \alpha_i} \left(\bw_1 -\bgamma_t \sum_{i=1}^{t} \alpha_i \bg_i \right)~.
\end{align*}
It is enough to prove that for any $t$, $\bta_t \bm_t = \frac{\alpha_{t+1}}{\sum_{i=1}^{t+1} \alpha_i} \left(\bgamma_t \sum_{i=1}^{t} \alpha_i \bg_i \right)$.
We claim it is true and prove it by induction. 

When $t=1$, it holds that $\bta_1 \bm_1= \frac{\alpha_2 \alpha_1}{\alpha_1 + \alpha_2} \bgamma_1 \bg_1$. 
Suppose it holds for $t= k-1, k \geq 2$. Then in the case of $t = k $, we have 
\begin{align*}
& \bta_k \bm_k \\
& =\left(\frac{\sum_{i=1}^{k-1} \alpha_i }{\sum_{i=1}^{k} \alpha_i} \bm_{k-1} + \frac{\alpha_k }{\sum_{i=1}^{k} \alpha_i } \bg_k \right) \cdot  \frac{\alpha_{k+1}  \sum_{i=1}^{k} \alpha_i }{\sum_{i=1}^{k+1} \alpha_i} \bgamma_k  \\
& = \left(\frac{\sum_{i=1}^{k-1} \alpha_i }{\sum_{i=1}^{k} \alpha_i} \left( \frac{1}{\eta_{k-1}}\frac{\alpha_{k}}{\sum_{i=1}^{k} \alpha_i} \bgamma_{k-1} \sum_{i=1}^{k-1} \alpha_i \bg_i \right) + \frac{\alpha_k }{\sum_{i=1}^{k} \alpha_i } \bg_k \right)\cdot \frac{\alpha_{k+1}  \sum_{i=1}^{k} \alpha_i }{\sum_{i=1}^{k+1} \alpha_i} \bgamma_k \\
& = \frac{\alpha_{k+1}  \sum_{i=1}^{k} \alpha_i }{\sum_{i=1}^{k+1} \alpha_i} \bgamma_k  \cdot 
\left(\frac{\sum_{i=1}^{k-1} \alpha_i }{\sum_{i=1}^{k} \alpha_i} \left(  \frac{ \sum_{i=1}^{k-1} \alpha_i \bg_i}{\sum_{i=1}^{k-1} \alpha_i }\right) + \frac{\alpha_k }{\sum_{i=1}^{k} \alpha_i } \bg_k \right)\\
 & = \frac{\alpha_{k+1} }{\sum_{i=1}^{k+1} \alpha_i} \bgamma_k \sum_{i=1}^{k} \alpha_i  \bg_i~. 
\end{align*}
where in the first equation we used the definitions of $\eta_k$ and $\bm_k$ and in the second equality we used the induction step.  So we proved the above claim. Thus, we can directly use Lemma~\ref{lemma: online_to_batch_ftrl}. 
\end{proof}

Here we show the proof of Theorem~\ref{thm:interpolation}.
\begin{proof}[Proof of Theorem~\ref{thm:interpolation}]
By Theorem~\ref{thm:sgdm_0}, we have 
\begin{align}
\E\left[F(\bx_T) \right] - F(\bx^{\star})
\leq \frac{2}{T}\left(\frac{\|\bx_1 - \bx^{\star} \|^2}{\alpha} + 2\alpha\right) \sqrt{ \E \sum_{t=1}^{T} \| \nabla f(\bx_t , \xi_t )\|^2 + \epsilon} + \frac{\alpha G^2}{\sqrt{\epsilon}}\label{eq:interpolation_bound}~. 
\end{align}
Under the interpolation condition and $L$-smoothness of the functions $f$, it satisfies that 
\begin{align*}
\E \sum_{t=1}^{T} \|\nabla f(\bx_t,\xi_t)\|^2 
& \leq  2L \E \left[ \sum_{t=1}^{T} \left( f(\bx_t , \xi_t) - f(\bx^{\star}, \xi_t ) \right)\right]\\
& \leq  2L  \sum_{t=1}^{T}\E  \left[ F(\bx_t )  \right]- F(\bx^{\star} )~.
\end{align*}
Use \eqref{eq:interpolation_bound} on each $t$ to get  
\begin{align*}
& \sum_{t=1}^{T}\E  \left[ F(\bx_t )  \right]- F(\bx^{\star}) \\
& \leq  \sum_{t=1}^{T}\frac{1}{t} \left[\left(\frac{\|\bx_1 - \bx^{\star} \|^2}{\alpha} + 2\alpha\right) \sqrt{ \E \sum_{i=1}^{t} \| \nabla f(\bx_i , \xi_i )\|^2 + \epsilon  } + \frac{\alpha G^2}{\sqrt{\epsilon}}\right]\\
& \leq \left(\frac{\|\bx_1 - \bx^{\star} \|^2}{\alpha} + 2\alpha\right)  \cdot \left(\sqrt{ \E \sum_{t=1}^{T} \| \nabla f(\bx_t, \xi_t )\|^2 + \epsilon} + \frac{\alpha G^2}{\sqrt{\epsilon}} \right)\ln T~.
\end{align*}
Then, we solve for $\E \sum_{t=1}^{T} \|\nabla f(\bx_t,\xi_t)\|^2 $ and get 
\begin{align*}
& \E \sum_{t=1}^{T} \|\nabla f(\bx_t,\xi_t)\|^2 \\
&  \leq 4L^2\left(\frac{\|\bx_1 - \bx^{\star} \|^2}{\alpha} + 2\alpha\right)^2 \ln^2 T 
 + 4L \sqrt{\epsilon}\left(\frac{\|\bx_1 - \bx^{\star} \|^2}{\alpha} + 2\alpha\right) \ln T + \frac{2\alpha G^2}{\sqrt{\epsilon}} ~.
\end{align*}
Using this expression in \eqref{eq:interpolation_bound}, we have the stated bound. 
\end{proof}

%% file: 5_Conclusion/conclusion.tex
\chapter{Conclusions}
In this dissertation, we studied the convergence of a series of heuristic variants of SGD, including the strategies for choosing stepsizes: Delayed AdaGrad, exponential and cosine stepsize, and the use of momentum. We moved along the way of closing the gap between the theory and practice by presenting formal guarantees to these heuristic optimization methods, providing possible explanations for good empirical performance from a theoretical perspective, and proposing improved algorithms when the theoretical results are suboptimal.

For future work, on the high probability analysis of  SGD with adaptive stepsizes, an interesting direction is to extend the current analysis of  Delayed AdaGrad with momentum to Adam, the popularly employed algorithm in machine learning applications. The updates of Adam are composed of a weighted sum of past gradients, each one multiplied by the current learning rate. The dependency between the past and the future makes the analysis challenging, as discussed in Section~\ref{sec:lemma}.

Moreover, most of the time we focus on smooth functions, yet some machine learning objective functions are non-smooth, such as for ReLU neural network. \citet{ZhangHSJ19, ZhangJFW20} analyzed clipped SGD on a class of non-smooth functions. An open problem is to study SGD with adaptive stepsizes and momentum on such non-smooth functions.